\documentclass{article}

\usepackage{arxiv}

\usepackage[utf8]{inputenc} 
\usepackage[T1]{fontenc}    
\usepackage{hyperref}       
\usepackage{url}            
\usepackage{booktabs}       
\usepackage{amsfonts}       
\usepackage{nicefrac}       
\usepackage{microtype}      
\usepackage{lipsum}		
\usepackage{graphicx}
\usepackage{natbib}
\usepackage{doi}
\usepackage[ruled]{algorithm2e}
\usepackage{amsthm}
\usepackage[dvipsnames]{xcolor}

\def\E{{\mathbb E}}
\def\Var{\mathrm{Var}}

\def\Cov{{\mathrm Cov}}
\def\Corr{{\mathrm Corr}}
\def\Diag{{\mathrm Diag}}
\def\Bias{{\mathrm Bias}}

\def\P{\mathbb P}
\def\Tr{{\mathrm Tr}}
\def\dd{\mathrm d}

\usepackage{mathtools}
\newcommand{\defeq}{\vcentcolon=}

\usepackage{verbatim} 
\usepackage{enumitem}

\newtheorem{theorem}{Theorem}
\newtheorem{lemma}{Lemma} 
\newtheorem{proposition}{Proposition}
\newtheorem{corollary}{Corollary}
\newtheorem{remark}{Remark}
\newtheorem{definition}[theorem]{Definition}

\title{On the Convergence of Experience Replay in Policy Optimization: Characterizing Bias, Variance, and Finite-Time Convergence}

\date{} 					

\author{ 
Hua Zheng\thanks{The majority of this work was done when Hua Zheng was at Northeastern University} \\
	Meta\\
	  Menlo Park, CA 94025, USA \\
	\texttt{huazhn@meta.com} \\
	\And
    Wei Xie\thanks{Corresponding author} \\
	Department of Mechanical and Industrial Engineering\\
	Northeastern University\\
	Boston, MA 02115, USA \\
	\texttt{w.xie@northeastern.edu} \\
	\And
	M. Ben Feng \\
	Department of Statistics and Actuarial Science\\
	University of Waterloo\\
	Waterloo, ON, Canada \\
	\texttt{ben.feng@uwaterloo.ca} \\
}

\renewcommand{\shorttitle}{On the convergence of Experience Replay}

\begin{document}
\maketitle

\begin{abstract}%
  Experience replay is a core ingredient of modern deep reinforcement learning, yet its benefits in policy optimization are poorly understood beyond empirical heuristics. This paper develops a novel theoretical framework for experience replay in modern policy gradient methods, 
  where two sources of dependence fundamentally complicate analysis: Markovian correlations along trajectories and policy drift across optimization iterations. We introduce a new proof technique based on auxiliary Markov chains and lag-based decoupling that makes these dependencies tractable. Within this framework, we derive finite-time bias bounds for policy-gradient estimators under replay, identifying how bias scales with the cumulative policy update, the mixing time of the underlying dynamics, and the age of buffered data, thereby formalizing the practitioner's rule of avoiding overly stale replay. We further provide a correlation-aware variance decomposition showing how sample dependence governs gradient variance from replay and when replay is beneficial. Building on these characterizations, we establish the finite-time convergence guarantees for experience-replay-based policy optimization, explicitly quantifying how buffer size, sample correlation, and mixing jointly determine the convergence rate and revealing an inherent bias-variance trade-off: larger buffers can reduce variance by averaging less correlated samples but can increase bias as data become stale. These results offer a principled guide for buffer sizing and replay schedules, bridging prior empirical findings with quantitative theory.
\end{abstract}


\section{Introduction} \label{sec: introduction}
Experience replay (ER) has become a critical component of modern deep reinforcement learning, enabling agents to learn efficiently by replaying past experiences. Despite its widespread adoption, the theoretical understanding of why and when experience replay helps remains surprisingly incomplete. 
A comprehensive empirical investigation comes from the value-based setting: \cite{fedus2020revisiting} systematically study how replay capacity and replay ratio (gradient updates per environment step) drive performance in Q-learning methods.  This work is part of nearly a decade of empirical exploration into ER work, beginning with its introduction
as a key ingredient in deep Q network (DQN) \citep{mnih2015human}, continuing through prioritized variants \citep{Schaul2016PrioritizedER}, hindsight modifications \citep{andrychowicz2017hindsight}, and large-scale ablation studies \citep{hessel2018rainbow} that confirm ER's importance but provide limited mechanistic understanding. In the policy-optimization setting, the empirical works are fragmented: practitioners have explored various off-policy correction schemes \citep{wang2017sample,degris2012off} and buffer management strategies \citep{zhang2017deeper,openai2019rubiks}. Yet these remain guided by experimental trial-and-error rather than theoretical principles, leaving a fundamental question unanswered: what does experience replay actually do to policy optimization?

This paper addresses this gap by developing a novel theoretical tool for analyzing ER in policy optimization. We focus on the step-based setting employed by modern policy gradient methods such as PPO and TRPO, where gradients are constructed from streams of individual state–action pairs rather than full episodes. This regime presents two critical challenges for theoretical analysis: (i) Markovian dependence among successive samples within trajectories, and (ii) policy drift across optimization iterations. Standard off-policy corrections—such as per-step likelihood ratios (importance weights)—are computationally convenient but introduce bias because replayed samples are no longer independent draws from the policy's stationary distribution. Our analysis makes this bias explicit and quantifies how it accumulates with policy updates and the staleness of buffered data.

\begin{figure}[h]
\vspace{-0.in}
	\centering
	\includegraphics[width=1\textwidth]{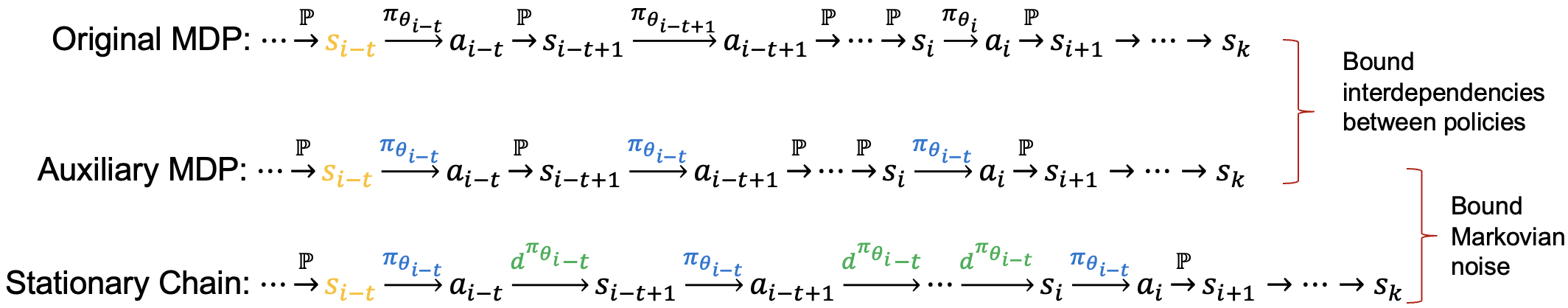}
	\vspace{-0.1in}
 \caption{\textbf{Illustration of our proof technique.}
To replay $(\pmb{s}_i,\pmb{a}_i)$ at iteration $k$, we condition on a $t$-lagged pair $(\pmb{s}_{i-t},\pmb\theta_{i-t})$ ($0\le t<i$) and compare three processes: (top) the original trajectory under drifting policies, (middle) an auxiliary MDP that replaces the policy sequence by the fixed policy $\pi_{\pmb\theta_{i-t}}$ to bound policy interdependence, and (bottom) a stationary chain, sampled from the fixed stationary distribution $d^{\pi_{\pmb\theta_{i-t}}}$, to bound the Markovian noise.
The bias/convergence bounds follow by decomposing the error quantity into the total-variations, i.e. $\Vert \mathbb{P}(\text{Original MDP}|\pmb{s}_i,\pmb\theta_i) - \mathbb{P}(\text{Auxilary MDP}|\pmb{s}_i,\pmb\theta_i) \Vert_{TV}$ and $\Vert \mathbb{P}(\text{Auxilary MDP}|\pmb{s}_i,\pmb\theta_i) - \mathbb{P}(\text{Stationary MDP}|\pmb{s}_i,\pmb\theta_i) \Vert_{TV}$.
}
	\label{fig: convergence comparison}
  \vspace{-0.in}
\end{figure}

Finite-sample bias bounds are derived for the likelihood-ratio (LR) and clipped likelihood-ratio (CLR) policy-gradient estimators under experience replay. These bounds reveal three key drivers of bias: (a) cumulative policy movement (governed by the learning-rate schedule), (b) environment dynamics (characterized by mixing rates), and (c) the age and volume of replayed trajectories. Our bounds tighten when the MDP mixes rapidly and policy updates remain small, and they imply asymptotic unbiasedness when buffer growth is appropriately controlled. These results provide the first formal justification for a widely used practitioner heuristic: "don't replay data that's too old."

Beyond bias, we present a variance decomposition that exposes how the covariance structure among gradient estimates from different historical policies governs the variance of replayed estimators. This decomposition clarifies two phenomena: replay reduces variance when cross-policy gradient correlations are positive but less than one, and it explains why practitioners benefit from diverse yet reasonably fresh samples. The interplay between bias and variance forms the foundation for understanding when experience replay accelerates learning and when it hinders progress.

Building on our bias and variance characterizations, we establish the first finite-time convergence guarantees for experience-replay-based policy optimization. Our analysis reveals how buffer size, sample correlation, and the mixing rate of the underlying Markov chain jointly govern convergence rate, and it provides a theoretical lens through which to interpret the empirical observations identified by \cite{fedus2020revisiting}. Specifically, our framework exposes a fundamental bias–variance trade-off: larger buffers reduce gradient variance by averaging over less correlated samples, but increase bias when those samples become stale relative to the evolving policy. 
This yields an explanation for when scaling replay capacity accelerates learning and when it hinders progress, complementing prior empirical findings from the value-based domain with principled theoretical guidance.

We make four main contributions:
\begin{itemize}
    \item Bias-variance analysis: Finite-sample bounds for LR/CLR policy-gradient estimators under ER that decompose the effects of policy staleness, buffer growth, learning rates, and mixing.
    \item Variance decomposition: A correlation-aware characterization of gradient variance under replay that explains when and why replay helps.
\item Finite-time convergence analysis: The first finite-time convergence guarantees for policy optimization with experience replay, yielding insights on buffer sizing and replay schedules.
\item Novel proof technique: Introduce a new theoretical tool based on auxiliary Markov chains and lag-based decoupling that makes experience replay analytically tractable in policy optimization, providing a principled explanation for empirical ER behaviors.

\end{itemize}
Together, these results elevate experience replay from an empirical practice to a tool with quantitative, theoretically grounded prescriptions for the policy optimization methods. The remainder of the paper is organized as follows. Section~\ref{sec:policyGradientRL} introduces the problem setting and presents the generic experience replay (ER) algorithm. Section \ref{sec: theory of experience replay} develops the variance and bias analysis for ER. Section \ref{sec: convergence analysis} presents our main finite-time convergence results.

\section{Problem Description}
\label{sec:policyGradientRL}


\noindent \textbf{Notation}: We use $\rho(x,y)=\mbox{Cov}\left(x,y\right)/\left(\sqrt{\Var(x)}\sqrt{\Var(y)}\right)$ denote the correlation coefficient between random variable $x$ and $y$. Let $\Diag(x_1,\ldots,x_n)$ denote the $n\times n$ matrix with diagonal entries $x_1,\ldots,x_n$, and zero off-diagonal elements. 
For any $(d\times d)$ matrix $A$, the trace used in the paper is defined as $\Tr({A})=\sum_{i=1}^d A_{i,i}$. We consider $L_2$ norm for any $d$-dimensional vector $\pmb{x}\in \mathbb{R}^d$, i.e., $\Vert\pmb{x}\Vert
\defeq
\Vert\pmb{x}\Vert_2 
=
\sqrt{\sum_{i=1}^d x_i^2}$. Let $P$ and $Q$ be two probability measures, and $p$ and $q$ denote their density functions. 
The total variation (TV) distance between two probability measures $P$ and $Q$ on the sample space $\mathcal{X}$
is defined
as $\Vert P-Q\Vert_{TV}= 1/2\int_{\mathcal{X}}|P(dx)- Q(dx)|=1/2\int_{\mathcal{X}}|p(x)- q(x)|dx$. In what follows, 
$\lceil x \rceil$ denotes the smallest integer greater than or equal to $x$, $\lfloor x \rfloor$ denotes the largest integer less than or equal to $x$, 
and $[N]$ denotes the set of integers $\{1,2,\ldots,N\}$.

\subsection{Markov Decision Process}
\label{subsec:mdpFormulation}

Consider an infinite-horizon discounted MDP specified by $(\mathcal{S},\mathcal{A}, r, \P, \gamma, \pmb{s}_1)$, where $\mathcal{S}$ and $\mathcal{A}$ denote the state and action space.
At any time $t$, the agent observes the state $\pmb{s}_t \in \mathcal{S}$, takes an action $\pmb{a}_t \in \mathcal{A}$,
and receives 
a reward $r(\pmb{s}_t,\pmb{a}_t) \in \mathbb{R}$. 
In this study, we consider stochastic policy $\pi_{\pmb{\theta}}: \mathcal{S}\mapsto\mathcal{A}$, defined as a mapping from state space to the action and it is parameterized by $\pmb{\theta} \in \mathbb{R}^d$, i.e., $\pmb{a}_t\sim \pi_{\pmb\theta}(\pmb{a}|\pmb{s}_t)$.
The state transition is specified by a probability model $\P$, i.e., $\P(\pmb{s}_{t+1}\in\cdot|\pmb{s}_t,\pmb{a}_t)$ with probability density function {$\P(\pmb{s}_t=\dd\pmb{s}|\pmb{s}_t,\pmb{a}_t)=p(\pmb{s}_t=\pmb{s}|\pmb{s}_t,\pmb{a}_t)$}.
 We use $\P(\pmb{s}_t\in\cdot|\pmb{s}_1,\pmb\theta)$ to represent the $t$-step state transition measure under a fixed behavior policy $\pi_{\pmb\theta}$. By an abuse of notation, \textit{we use $\P(\pmb{s}_t\in\cdot|\pmb{s}_1)$ to represent the $t$-step state transition measure under the evolving behavior policies $\pi_{\pmb\theta_1},\pi_{\pmb\theta_2},\ldots,\pi_{\pmb\theta_{t-1}}$}.


Let $\gamma\in(0,1)$ denote the discount factor. 
Our goal is to find the optimal policy, denoted by $\pi_{\pmb{\theta}^*}$, maximizing the expected discounted rewards, i.e.,
\begin{align} \label{eq: objective infinite horizon}
 \max_{\pmb{\theta}\in\Theta} ~  J(\pmb\theta)
 = \E \left[ \left. 
 \sum_{t=1}^\infty \gamma^{t-1}r(\pmb{s}_{t},\pmb{a}_{t})
 \right|\pi_{\pmb\theta}, \pmb{s}_1 \right] 
 = \E_{\pmb{s}\sim d^{\pi_{\pmb\theta}}(\pmb{s}),\pmb{a}\sim \pi_{\pmb\theta}(\pmb{a}|\pmb{s})}[r(\pmb{s},\pmb{a})],
\end{align}
where $\Theta$ represents the policy parameter space and $d^{{\pi_{\pmb\theta}}}(\pmb{s})$ is the state stationary distribution induced by the policy ${\pi_{\pmb\theta}}$, defined as
$d^{\pi_{\pmb\theta}}(\pmb{s})=(1-\gamma)
\sum_{t=1}^\infty\gamma^{t-1}p(\pmb{s}_{t}=\pmb{s}|\pmb{s}_1;{\pi_{\pmb\theta}}).
$
We denote 
the state-occupancy
measure of state-action pair by 
$d^{\pi_{\pmb\theta}}(\pmb{s},\pmb{a})=\pi_{\pmb\theta}(\pmb{a}|\pmb{s})d^{\pi_{\pmb\theta}}(\pmb{s})$.

Similarly, given policy $\pi_{\pmb{\theta}}$, we define the state-value and action-value functions as follows, 
\begin{align}
 V^{\pi_{\pmb\theta}}(\pmb{s})= \E\left[\left.\sum_{t=1}^{\infty} \gamma^{t-1}r(\pmb{s}_{t},\pmb{a}_{t})\right| \pmb{s}_1=\pmb{s};{\pi_{\pmb\theta}}\right],  Q^{\pi_{\pmb\theta}}(\pmb{s},\pmb{a}) = \E\left[\left.\sum_{t=1}^{\infty} \gamma^{t-1}r(\pmb{s}_{t},\pmb{a}_{t})\right| \pmb{s}_1=\pmb{s}, \pmb{a}_1=\pmb{a};\pi_{\pmb\theta}\right].
 \nonumber 
 \nonumber 
\end{align}


\subsection{Policy Gradient}
\label{subsec:generalPolicyGradient}

\begin{sloppypar}

Stochastic policy gradient ascent is a popular method to solve the RL optimization problem~\eqref{eq: objective infinite horizon}. At each $k$-th iteration, we iteratively update the policy parameters by
\begin{equation} 
\label{eq: policy gradient update}
 \pmb\theta_{k+1} \leftarrow \pmb\theta_k + \eta_k \widehat{\nabla} J(\pmb\theta_k),
\end{equation}
where $\eta_k$ is learning rate or step size and $\widehat{\nabla} J(\pmb{\theta}_k)$ is an estimator of policy gradient $\nabla  J(\pmb{\theta}_k)$.
For convenience, $\nabla$ denotes the gradient with respect to $\pmb{\theta}$ unless specified otherwise.

Under some regular conditions
used to change the order of gradient and expectation
\citep{sutton1999policy,williams1992simple},
the generic policy gradient can be written as,
\begin{equation}
 \nabla J(\pmb{\theta})= \E_{(\pmb{s},\pmb{a})\sim d^{\pi_{\pmb\theta}}(\cdot,\cdot)}[g(\pmb{s},\pmb{a})|\pmb{\theta}]=\E_{(\pmb{s},\pmb{a})\sim d^{\pi_{\pmb\theta}}(\cdot,\cdot)}\left[A^{\pi_{\pmb\theta}}(\pmb{s},\pmb{a}) \nabla \log \pi_{\pmb{\theta}}(\pmb{a}|\pmb{s})\right]
 \label{eq: policy gradient}
\end{equation} 
 where $A^{\pi_{\pmb\theta}}(\pmb{s},\pmb{a}) \defeq Q^{\pi_{\pmb\theta}}(\pmb{s},\pmb{a})-V^{\pi_{\pmb\theta}}(\pmb{s})$ is called \textit{advantage} and it intuitively measures the extra reward that the agent can obtain by taking a particular action $\pmb{a}$ at state $\pmb{s}$. Also, 
 the scenario-based policy gradient estimate
 in (\ref{eq: policy gradient}) is
 \begin{equation}
g\left(\pmb{s},\pmb{a}|\pmb\theta_k\right) \defeq
 A^{\pi_{\pmb\theta_k}}(\pmb{s},\pmb{a}) \nabla \log \pi_{\pmb{\theta}_k}(\pmb{a}|\pmb{s}).
 \label{eq.scenariobasedGradient}
 \end{equation}
The classical \textit{policy gradient (PG) estimator} in the $k$-th iteration is given by 
\begin{equation}
 \widehat{\nabla} J^{PG}_k \defeq
 \widehat{\nabla} J^{PG}(\pmb{\theta}_k) =\frac{1}{n} \sum_{j=1}^n g\left(\pmb{s}^{(k,j)},\pmb{a}^{(k,j)}|\pmb\theta_k\right),
 \label{eq.PG-estimator}
\end{equation}
where $\left({\pmb{s}}^{(k,j)},{\pmb{a}}^{(k,j)}\right)$ represents the $j$-th sample collected at the $k$-th iteration. 

\subsection{Off-policy Policy Gradient with Experience Replay}
When using ER, we estimate the policy gradient $\nabla J(\pmb\theta_k)$ at iteration $k$ using historical samples from old policies $\pi_{\pmb\theta_i}$ ($i<k$) with the importance weight:

\begin{equation}
 \widehat{\nabla} J^{LR}_{i,k}=\frac{1}{n}\sum^{n}_{j=1}f_{i,k}\left(\pmb{s}^{(i,j)},\pmb{a}^{(i,j)}\right)
 g\left(\pmb{s}^{(i,j)},\pmb{a}^{(i,j)}|\pmb\theta_k\right)
 ~ \mbox{with} ~ f_{i,k}(\pmb{s},\pmb{a}) = \frac{\pi_{\pmb\theta_k}(\pmb{a}|\pmb{s})}{\pi_{\pmb\theta_i}(\pmb{a}|\pmb{s})}.
 \label{eq: individual likelihood ratio estimator}
\end{equation}
The likelihood ratio $f_{i,k}(\pmb{s},\pmb{a})$ weights the historical samples to account for 
the mismatch between the behavior and target policies specified by parameters $\pmb\theta_i$ and $\pmb\theta_k$. 
Through utilizing all samples in the replay buffer of policies $\mathcal{U}_k$, we average the individual LR estimators to obtain the likelihood ratio (LR) policy gradient estimator,
\begin{equation}
\label{eq.LR-gradient}
\widehat{\nabla} J^{LR}_{k}=\frac{1}{|\mathcal{U}_k|}\sum_{{\pmb\theta_i}\in \mathcal{U}_k}\widehat{\nabla} J^{LR}_{i,k} = 
 \frac{1}{|\mathcal{U}_k|n}
 \sum_{{\pmb\theta_i}\in \mathcal{U}_k}
 \sum^{n}_{j=1}
 \frac{\pi_{\pmb\theta_k}(\pmb{a}^{(i,j)}|\pmb{s}^{(i,j)})}
 {\pi_{\pmb\theta_i}(\pmb{a}^{(i,j)}|\pmb{s}^{(i,j)})}
 g\left(\pmb{s}^{(i,j)},\pmb{a}^{(i,j)}|\pmb\theta_k\right).
\end{equation}

The LR estimators in \eqref{eq: individual likelihood ratio estimator}-\eqref{eq.LR-gradient} suffer from potentially large or infinite variance as the likelihood ratio $f_{i,k}(\cdot,\cdot)$ can be large or unbounded \citep{veach1995optimally}. A common solution is weight clipping~\citep{ionides2008truncated}, which truncates the LR by
$$f^{clip}_{i,k}(\cdot,\cdot)\defeq \min\left(f_{i,k}(\cdot,\cdot), U_f\right)$$ where $U_f>1$ is a clipping threshold. This yields the following estimators
\begin{align}
\widehat{\nabla} J^{CLR}_{k}&=\frac{1}{|\mathcal{U}_k|}\sum_{{\pmb\theta_i}\in \mathcal{U}_k} \widehat{\nabla} J^{CLR}_{i,k},
\label{eq.CLR-gradient} \\
\mbox{  with  } \widehat{\nabla} J^{CLR}_{i,k} &= 
 \frac{1}{n}
 \sum^{n}_{j=1}f^{clip}_{i,k}(\pmb{s}^{(i,j)},\pmb{a}^{(i,j)})
g\left(\pmb{s}^{(i,j)},\pmb{a}^{(i,j)}|\pmb\theta_k\right).\label{eq: individual clipped likelihood ratio estimator}
\end{align}
It is straightforward to show
$f^{clip}_{i,k}(\pmb{s}, \pmb{a}) \leq U_f$. 
We refer to $\widehat{\nabla} J^{CLR}_{i,k}$ in Eq.~\eqref{eq: individual clipped likelihood ratio estimator} and $\widehat{\nabla} J^{CLR}_{k}$ in Eq.~ \eqref{eq.CLR-gradient} as the individual and average clipped likelihood ratio (CLR) policy gradient estimators, respectively.

For notational simplicity, we present experience replay using a replay buffer of behavior policies $\mathcal{U}_k$ which indexes both the policies and the data collected under them. This is solely an analytical convenience to simplify the notation. In practice, we do not require storing policy snapshots; instead, the replay buffer $\mathcal{D}_k$ stores the collected transitions along with the associated (log-)likelihood terms—and, when needed, value estimates and advantages—required for replay updates.
 
\subsection{Regularity Conditions for Generic Policy Gradient Estimation}
\label{subsec:assumptions}

We outline the assumptions for policy gradient regularity and MDPs used throughout this paper.
\begin{enumerate}[label=\textbf{A.\arabic*}]
\item Suppose the reward and policy functions satisfy the following regularity conditions.
\begin{enumerate}
    \item[(i)] The absolute value of the reward $r(\pmb{s}, \pmb{a})$ is bounded uniformly, i.e., there exists a constant, say $U_r>0$ such that  $|r(\pmb{s},\pmb{a})|\leq U_r$ for any $(\pmb{s},\pmb{a})\in \mathcal{S}\times\mathcal{A}$.
    \item[(ii)] The score function is Lipschitz continuous and has bounded norm; and the policy $\pi_{\pmb\theta}$ is differentiable and Lipschitz continuous with respect to $\pmb\theta$
i.e., for any $(\pmb{s},\pmb{a})\in \mathcal{S}\times\mathcal{A}$, there exist positive constant bounds, denoted by $L_\Theta, U_{\Theta}$, and $U_{\pi}$, such that
\begin{align}
    &\Vert \nabla\log\pi_{\pmb\theta_1}(\pmb{a}|\pmb{s})  -\nabla\log\pi_{\pmb\theta_2}(\pmb{a}|\pmb{s}) 
    \Vert
    \leq L_\Theta\Vert  \pmb{\theta}_1-\pmb{\theta}_2\Vert \text{ for any $\pmb\theta_1$ and $\pmb\theta_2$}; \label{eq: assumption item 1}\\
    &\Vert \nabla\log\pi_{\pmb\theta}(\pmb{a}|\pmb{s})\Vert \leq U_\Theta \text{ for any $\pmb\theta$};\label{eq: assumption item 2}
    \\
    &\Vert\pi_{\pmb\theta_1}(\cdot|\pmb{s})-\pi_{\pmb\theta_2}(\cdot|\pmb{s}) \Vert_{TV} \leq {U_\pi}\Vert\pmb\theta_1-\pmb\theta_2\Vert.
    \label{eq: assumption item 3}
\end{align}
\end{enumerate}
\label{assumption 2}

\item (Uniform Ergodicity)
For a fixed $\pmb\theta$, denote $d^{\pi_{\pmb\theta}}(\cdot)$ as the stationary distribution of 
an infinite-horizon MDP generated by the rule, i.e., $\pmb{a}_t\sim \pi_{\pmb\theta}(\cdot|\pmb{s}_t)$ and $\pmb{s}_{t+1}\sim p(\cdot|\pmb{s}_t,\pmb{a}_t)$. 
There exists a decreasing function $\varphi(t)>0$ such that:
$$\left\Vert \P(\pmb{s}_t\in \cdot |\pmb{s}_1=\pmb{s})-d^{\pi_{\pmb\theta}}(\cdot)  \right\Vert_{TV}\leq \varphi(t), \forall t \geq 1, \forall \pmb{s} \in\mathcal{S},$$
where $\varphi(t)=\kappa_0 \kappa^t$ for some constants $\kappa_0 > 0$ and $\kappa\in(0,1)$.
 \label{assumption 3}
\end{enumerate}

The uniform bound assumption for reward function in \ref{assumption 2} (i) is commonly used in most RL studies \citep{wu2020finite,zhang2020global,kumar2019sample}.
The first two items in \ref{assumption 2} (ii) are standard assumptions on
the regularity of the MDP problem and the parameterized policy, which are the same conditions used in many recent studies \citep{wu2020finite,zhang2020global,kumar2019sample}. The third item in \ref{assumption 2} (ii) is adopted from 
\cite{xu2020improving} 
 which holds for any smooth policy with bounded action space and Gaussian policy; see justification in \cite{xu2020improving}. 
The uniform ergodicity assumption~\ref{assumption 3} is a commonly used assumption that implies long-term stability in the sense that the Markov chain has a stationary distribution to which it will converge, regardless of the initial state. This assumption has been adopted by \citet{xu2020improving,wu2020finite,bhandari2018finite,zou2019finite}, which holds for any time-homogeneous Markov chain with finite-state space or any uniformly ergodic Markov chain with general state space.
 
The uniform boundedness of the reward function $r(\pmb{s},\pmb{a})$ in Assumption~\ref{assumption 2}(i) also implies that the absolute value of the Q-function is upper bounded
by $\frac{U_r}{1-\gamma}$, since by definition
\begin{equation}\label{eq: bound of action value function}
|Q^{\pi_{\pmb\theta}}(\pmb{s},\pmb{a})| \leq \sum^\infty_{t=1}\gamma^{t-1} U_r=\frac{U_r}{1-\gamma}, \text{ for any $(\pmb{s},\pmb{a})\in \mathcal{S}\times\mathcal{A}$}.
\end{equation}
The same bound also applies for $V^{\pi_{\pmb\theta}}(\pmb{s})$ for any $\pmb\theta$ and $\pmb{s}\in\mathcal{S}$, i.e., 
\begin{equation}\label{eq: bound of state value function}
   |V^{\pi_{\pmb\theta}}(\pmb{s})| =  |\E\left[Q^{\pi_{\pmb\theta}}(\pmb{s},\pmb{a})\right]| \leq \E\left[|Q^{\pi_{\pmb\theta}}(\pmb{s},\pmb{a})|\right] \leq \sum^\infty_{t=1}\gamma^{t-1} U_r=\frac{U_r}{1-\gamma}.
\end{equation}
This applies to the objective $ J(\pmb\theta)$ because $| J(\pmb\theta)| = |\E[V^{\pi_{\pmb\theta}}(\pmb{s})]|\leq \frac{U_r}{1-\gamma}$. 
Thus, we have the bound of the objective as $U_ J \defeq\frac{U_r}{1-\gamma}$.

Under Assumption \ref{assumption 2}, we can establish the Lipschitz continuity of the policy gradient $\nabla  J(\pmb\theta)$ as shown in
Lemma~\ref{lemma: Lipschitz continuity}. 
 Lemma~\ref{lemma: Lipschitz continuity} is another essential condition to ensure the convergence of many gradient-based algorithms \citep{ niu2011hogwild,reddi2016stochastic,nemirovski2009robust,li2019convergence}. It implies that the policy gradient cannot change abruptly, i.e., $| J(\pmb{\theta}_2) -  J(\pmb{\theta}_1)-\langle\nabla  J(\pmb{\theta}_1), \pmb{\theta}_2-\pmb{\theta}_1\rangle| \leq L\lVert \pmb{\theta}_2-\pmb{\theta}_1\rVert^2$ (\cite{nesterov2003introductory}, Lemma 1.2.3).
\begin{lemma}[\cite{zhang2020global}, Lemma 3.2]\label{lemma: Lipschitz continuity}
 Under Assumption \ref{assumption 2}, the policy gradient $\nabla  J(\pmb\theta)$ is Lipschitz continuous, i.e., for any $\pmb\theta_1,\pmb\theta_2 \in\Theta$, 
 there exists a constant $L > 0$ s.t.
 $$\Vert\nabla  J(\pmb{\theta}_1) -\nabla  J(\pmb{\theta}_2)\Vert\leq L\lVert \pmb{\theta}_1-\pmb{\theta}_2\rVert.$$
\end{lemma}


Lemma~\ref{lemma: bounded of policy gradient} establishes the boundedness of policy gradient and its estimate; see the proof in \ref{appendix subsec: lipschitz continuity}.
 \begin{lemma}[Boundedness 
 of Stochastic Policy Gradients]
 \label{lemma: bounded of policy gradient} 
 For any $\pmb\theta$, the norm of the policy gradient $\nabla  J(\pmb\theta)$ and its scenario-based estimate $g(\pmb{s},\pmb{a}|\pmb\theta)$ is bounded, i.e., 
 $$\Vert \nabla  J(\pmb\theta)\Vert \leq M 
 \text{ and } 
  \Vert g\left(\pmb{s},\pmb{a}|\pmb\theta\right)\Vert \leq M,
  $$
 where $M=\frac{2U_r U_\Theta}{1-\gamma}$.
 \end{lemma}

\subsection{Generic Experience Replay}

Algorithm~\ref{algo: online} presents a generic experience-replay (ER) template for policy optimization.
At each iteration $k$, the agent first collects $n$ fresh samples $\mathcal{T}_k$ under the current policy $\pi_{\theta_k}$ and appends them to the replay buffer $\mathcal{D}_k$.
It then performs $K_{\mathrm{off}}$ updates using mini-batches sampled from $\mathcal{D}_k$.
The policy replay buffer $\mathcal{U}_k$ stores the most recent $B$ visited policies and is coupled with the data replay buffer $\mathcal{D}_k$. For notational convenience, in the following analysis, we use the policy replay buffer to index the ER process. However, in implementation, we do not store policy snapshots; the replay buffer $\mathcal{D}_k$ retains only the collected transitions together with the corresponding (log-)likelihood terms required for replay updates.

\begin{algorithm}[htb]
\SetAlgoLined
\caption{Generic Experience Replay for Policy Optimization (ER-PG)}
\label{algo: online}

\textbf{Input:} iterations $K$; buffer size $B$; batch size $n$; replay updates $K_{\mathrm{off}}$;
stepsizes $\{\eta_k\}_{k=1}^K$; initial replay buffer $\mathcal{D}_0$; initial replay buffer of policies $\mathcal{U}_0$.\\
\textbf{Initialize:} $\pmb{\theta}_1$; set $\pmb{\theta}_1^0\leftarrow \pmb{\theta}_1$; $\mathcal{U}_1\leftarrow \mathcal{U}_0\cup\{\pmb{\theta}_1\}$\;

\For{$k=1,2,\ldots,K$}{
  \textbf{(1) Collect:} Collect new samples
  $\mathcal{T}_k \defeq \{(s^{(k,j)},a^{(k,j)})\}_{j=1}^n$ by executing $\pi_{\pmb{\theta}_k}$ and
  $\mathcal{D}_{k}\leftarrow \mathcal{D}_{k-1}\cup \mathcal{T}_k$\;

  \textbf{(2) Offline updates:}
  \For{$h=0,1,\ldots,K_{\mathrm{off}}-1$}{
    (a) sample mini-batch from $\mathcal{D}_k$\;

    (b) compute gradient estimate $\widehat{\nabla}J_{k}^{h}$
    (e.g., LR/CLR in \eqref{eq.LR-gradient}--\eqref{eq.CLR-gradient})\;

    (c) $\pmb{\theta}_k^{h+1}\leftarrow \pmb{\theta}_k^{h} + \eta_k\,\widehat{\nabla}J_{k}^{h}$\;
  }

  \textbf{(3) Update buffer:}
  $\pmb{\theta}_{k+1}\leftarrow \pmb{\theta}_k^{K_{\mathrm{off}}}$; $\pmb{\theta}_{k+1}^0\leftarrow \pmb{\theta}_{k+1}$;
  $\mathcal{U}_{k+1}\leftarrow \mathcal{U}_{k}\cup\{\pmb{\theta}_{k+1}\}$;
  if $|\mathcal{U}_{k+1}|>B$, remove $\pmb{\theta}_{k+1-B}$ from $\mathcal{U}_{k+1}$\;
}
\end{algorithm}

\end{sloppypar}

\section{Variance-Bias Analysis in Experience Replay}
\label{sec: theory of experience replay}

While ER hyperparameters (e.g., policy age, buffer size, replay ratios) have been studied empirically, theoretical understanding remains limited. We address this gap by developing a framework that characterizes the bias-variance trade-off inherent in experience replay. Our analysis examines how sample dependence from three sources—Markovian noise, ER, and policy updates—impacts the bias (Section~\ref{subsec: experience replay bias}) and variance (Section~\ref{subsec: correlation increase gradient variance}) of LR/CLR policy gradient estimators.

\subsection{Replaying Old Samples with Dependence Cause Policy Gradient Bias}\label{subsec: experience replay bias}

 It has been empirically observed that given a fixed replay capacity, the performance of ER-based RL algorithms tends to improve as the age of the oldest policy decreases \citep{fedus2020revisiting}. The theoretical study presented in this section offers insights into this 
 observation.
\begin{definition}[\cite{fedus2020revisiting}]\label{def: age of policy}
The \textit{age} of a historical sample and associated policy
    in the ER 
    is defined to be the number of iterations passed
    since the sample 
    was generated.
\end{definition}


Despite their popularity, Policy Gradient (PG) methods have significant challenges in generating unbiased policy gradient estimators in the infinite-horizon MDP optimization. A critical condition for unbiased estimators, as outlined by the Policy Gradient Theorem \citep{sutton1999policy}, requires sampling state-action pairs $(\pmb{s},\pmb{a})$ from the stationary distribution of the Markov chain under the candidate policy of interest. However, as highlighted by \cite{zhang2020global}, in practice, samples are sequentially generated and the \textit{dependence} induced by Markovian noise consequently introduces bias into the policy gradient estimators. Furthermore, due to the mechanism of ER, the sample reuse will further increase the dependence between behavior policies, affecting the bias in policy gradient estimators.
To summarize, we present the upper bound on the bias of both LR and CLR policy gradient estimators in Theorem~\ref{theorem: bounded bias for LR graident}. A proof sketch is provided in Appendix~\ref{appendix sec: proof sketch} and the detailed proof is provided in Appendix~\ref{appendix sec: bias of policy gradient}. Before the presentation of the theorem, we introduce the mixing time
\begin{definition}[Mixing Time, \cite{wu2020finite}]
    Given the learning rate $\eta_k$, the mixing time of an ergodic MDP with a fixed policy $\pi_{\pmb\theta_k}$ is defined as 
$t_k^\star \defeq \min\{t \geq 0|\varphi(nt) \leq  \eta_{k}
\}$.
\end{definition}
At any $k$-th iteration, when the historical samples ${\pmb{x}}^{(i,j)}=\left({\pmb{s}}^{(i,j)},{\pmb{a}}^{(i,j)}\right)$ with $j=1,2,\ldots,n$ generated in the $i$-th iteration under $\pi_{\pmb\theta_i}$ are replayed, we use the lag term $t$ to control the impact of dependence induced by Markovian noise 
and derive an upper bound on the bias of the LR and CLR gradient estimators \eqref{eq.LR-gradient} and \eqref{eq.CLR-gradient} in Theorem~\ref{theorem: bounded bias for LR graident} that is applicable to
any positive integer $t<i$. However, it's important to note that the bound becomes tighter as the iteration $k$ increases and the learning rate $\eta_{k}$ decreases when we choose the mixing time defined below, i.e., $t=t_k^\star$.

\begin{theorem}[Gradient Biasedness]
\label{theorem: bounded bias for LR graident}
Suppose that Assumptions \ref{assumption 2} and \ref{assumption 3} hold. For LR/CLR policy gradient estimators in Eq.~\eqref{eq.LR-gradient} and \eqref{eq.CLR-gradient}, we consider the policy update rule (\ref{eq: policy gradient update}) with the learning rate $\eta_k$. For any integer $t$ such that $0\leq t < i$ with ${\pmb\theta_i}\in\mathcal{U}_k$, the bias of policy gradient estimators ($R\in$\{LR,CLR\}) can be bounded by 
\begin{equation*}
     \left\Vert\E\left[\widehat{\nabla} J^{R}_k\right] - \E[\nabla  J(\pmb\theta_k)]\right\Vert \leq \frac{Z_1^{R} \cdot t}{|\mathcal{U}_k|}
 \sum_{{\pmb\theta_i}\in \mathcal{U}_k}\sum^i_{\ell = i-t}\eta_\ell + 2M\varphi(nt) + \frac{Z_2^{R}}{|\mathcal{U}_k|}
 \sum_{{\pmb\theta_i}\in \mathcal{U}_k} \sum^k_{\ell=i-t}\eta_\ell, 
\end{equation*}
where $Z_1^{LR}=2nM^2 U_\pi$, $Z_1^{CLR}= 2nM^2 U_\pi U_f$, $Z_2^{LR}=M(2MC_d+2L_g+MU_\pi)$ and $Z_2^{CLR}=M(2MC_d+ (U_f+1)L_g+MU_\pi)$. Here, the Lipchitz constants $C_d$ and $L_g$ are defined in Lemma~\ref{lemma: lipchitz continuity of stationary distribution} and Lemma~\ref{lemma: lipchitz continuity of sample gradient estimate}, respectively in Appendix~\ref{appendix subsec: lipschitz continuity}.
\end{theorem}
Noting that 
$Z_1^{\mathrm{CLR}}-Z_1^{\mathrm{LR}}=2\,(U_f-1)\,nM^2U_\pi$ and 
$
Z_2^{\mathrm{CLR}}-Z_2^{\mathrm{LR}}=(U_f-1)\,L_g\,M$.
We see that clipping introduces an additional bias term that scales linearly with the clipping factor gap $(U_f-1)$. In particular, when $U_f=1$ (no clipping), the CLR bound reduces to the LR bound, whereas larger $U_f$ inflates the bias constants proportionally.

Then by applying Theorem~\ref{theorem: bounded bias for LR graident} and setting $i=k$ (without ER), we immediately obtain the upper bound for the bias of the policy gradient estimator 
(\ref{eq.PG-estimator}) in Corollary~\ref{corollary: bounded bias for vanila graident}. 

\begin{corollary}
\label{corollary: bounded bias for vanila graident}
Suppose that Assumptions \ref{assumption 2} and \ref{assumption 3} hold. 
For any integer $t$ such that $0\leq t < i$, the bias of policy gradient estimator (\ref{eq.PG-estimator}) can be bounded by 
\begin{equation*}
     \left\Vert\E\left[\widehat{\nabla} J^{PG}_k\right] - \E[\nabla  J(\pmb\theta_k)]\right\Vert \leq {Z_1^{LR} \cdot t}\sum^k_{\ell = k-t}\eta_\ell + 2M\varphi(nt) +{Z_2^{LR}}\sum^k_{\ell=k-t}\eta_\ell 
\end{equation*}
where $Z_1^{LR}$ and $Z_2^{LR}$ are specified in Theorem~\ref{theorem: bounded bias for LR graident}.
\end{corollary}

Consider a decreasing learning rate $\eta_k=\eta_1 k^{-r}$ where $\eta_1$ and $r\in (0,1)$ are positive constant. Notice that {$\sum^i_{i-t}\eta_\ell \leq t\eta_{i-t}$ and $\sum^{k}_{\ell=i-t}\eta_\ell\leq (k-i+t)\eta_{i-t}$} due to $k\geq i> t$. For $R\in\{LR,CLR\}$, the bound in Theorem~\ref{theorem: bounded bias for LR graident} can be expressed in $\mathcal{O}(\cdot)$ notation
\begin{equation}\label{eq: big O of bias}
\left\Vert\E\left[\widehat{\nabla} J^{R}_k\right] - \E[\nabla  J(\pmb\theta_k)]\right\Vert
\leq \mathcal{O}\left(\frac{1}{|\mathcal{U}_k|}
 \sum_{{\pmb\theta_i}\in \mathcal{U}_k}\frac{t^2}{(i-t)^r}\right)+\mathcal{O}\left(\frac{1}{|\mathcal{U}_k|}
 \sum_{{\pmb\theta_i}\in \mathcal{U}_k}\frac{k-i+t}{(i-t)^r}\right)+\mathcal{O}\left(\varphi(nt)\right).
\end{equation}
This upper bound shows that the bias of LR/CLR policy gradient estimator becomes larger when samples from older policies (i.e., smaller value of $i$) are selected since both terms $\frac{1}{(i-t)^r}$ and $\frac{k-i+t}{(i-t)^r}$ increase in $i$ for a fixed iteration $k$. Also, by observing the fact $k-i\leq B_k$ where the buffer size $B_k$ is a function of $k$, the bound \eqref{eq: big O of bias} can be expressed in the form of
\begin{align}
\left\Vert\E\left[\widehat{\nabla} J^{R}_k\right] - \E[\nabla  J(\pmb\theta_k)]\right\Vert  \leq \mathcal{O}\left(\frac{t^2}{(k-B_k-t)^r}\right)+\mathcal{O}\left(\frac{B_k+t}{(k-B_k-t)^r}\right)+\mathcal{O}\left(\varphi(nt)\right)\nonumber\\
= \mathcal{O}\left(\frac{\frac{r^2}{n^2}\log_{\kappa}^2 k}{(k-B_k+\frac{r}{n}\log_{\kappa} k)^r}\right)+\mathcal{O}\left(\frac{B_k-\frac{r}{n}\log_{\kappa} k}{(k-B_k+\frac{r}{n}\log_{\kappa} k)^r}\right)+\mathcal{O}\left(\frac{1}{k^r}\right) \label{eq: big O of bias with buffer size}
\end{align}
where the mixing rate of the process $\varphi(nt)$ becomes $\mathcal{O}(k^{-r})$ if the lag term $t$ is chosen to be the mixing time $t_k^*=\log_\kappa k^{-r/n}$.
Therefore, we have several theoretical conclusions:
\begin{itemize}
\item \textbf{Asymptotically Unbiased}. The bias converges in norm  $\Vert\E[\widehat{\nabla} J^{R}_k] - \E[\nabla  J(\pmb\theta_k)]\Vert \rightarrow 0$ as $k\rightarrow \infty$ if $B_k = \mathcal{O}(k^{r})$, indicating asymptotic unbiasedness of estimators \eqref{eq.LR-gradient} and \eqref{eq.CLR-gradient}. One special case satisfying this condition is that the buffer size $B_k$ is constant.

    \item \textbf{Replaying Old Samples Causes High Bias}. If the buffer size increases at a rate higher than $\mathcal{O}(k^{r})$, 
    the LR/CLR policy gradient estimators might have an unbounded bias. This bias tends to be minimized if no experience is replayed, i.e., $\mathcal{U}_k=\{\pmb\theta_k\}$. In such case, the bias is bounded by $\mathcal{O}\left({t^2}{(k-t)^{-r}}\right)$ by applying (\ref{eq: big O of bias}).
    \item \textbf{Mixing Condition}. The convergence rate of the bias is closely connected to the mixing rate of the environment. A higher mixing rate leads to a faster convergence. 
\end{itemize}

\subsection{Sample Correlation Increase Gradient Variance} \label{subsec: correlation increase gradient variance}

At any $k$-th iteration, given a replay buffer $|\mathcal{U}|_k$, the total variance of the LR/CLR policy gradient estimators \eqref{eq.LR-gradient} and \eqref{eq.CLR-gradient} is shown in 
Proposition~\ref{prop: gradient variance decomposition}; see the proof in Appendix~\ref{appendix sec: prop 1}.

 \begin{proposition}[Gradient Variance] 
 \label{prop: gradient variance decomposition}
Suppose that Assumptions \ref{assumption 2} and \ref{assumption 3} hold. At the $k$-th iteration with the target policy $\pi_{\pmb\theta_k}$, given the replay buffer $\mathcal{U}_k$, the total variance of LR/CLR policy gradient estimators ($R\in$\{LR,CLR\}) in Eq.~\eqref{eq.LR-gradient} and \eqref{eq.CLR-gradient} can be decomposed to
\begin{equation}\label{eq: variance of LR estimaor}
     \Tr\left(\Var
 \left[\widehat{\nabla} J^{R}_{k} \right]\right)=\frac{1}{|\mathcal{U}_k|^2}\sum_{\pmb\theta_i\in\mathcal{U}_k}\sum_{\pmb\theta_{i^\prime}\in\mathcal{U}_k} \left(\pmb\sigma^{R}_{i,k}\right)^\top \mathbf{P}_{i,i^\prime,k} \left(\pmb\sigma^{R}_{i^\prime,k}\right)
\end{equation}
where $\pmb\sigma^{R}_{i,k}=\left(\sqrt{\Var\left[\widehat{\nabla} J^{R,(1)}_{i,k} \right]}, \sqrt{\Var\left[\widehat{\nabla} J^{R,(2)}_{i,k} \right]},\ldots,\sqrt{\Var\left[\widehat{\nabla} J^{R,(d)}_{i,k} \right]}\right)^\top$ and \\ $\mathbf{P}_{i,i^\prime,k}=\Diag\left(\Corr^{(1)}_{i,i^\prime,k},\ldots,\Corr^{(d)}_{i,i^\prime,k}\right)$ with correlation coefficients of gradient estimate pairs $\Corr_{i,i^\prime,k}^{(\ell)}=\Corr\left(\widehat{\nabla} J^{R,(\ell)}_{i,k},\widehat{\nabla} J^{R,(\ell)}_{i^\prime,k}\right)$ for each $\ell$-th dimension of policy parameters with $\ell=1,2,\ldots,d$. By using the greatest element of $\mathbf{P}$, the total variance~\eqref{eq: variance of LR estimaor} is bounded by
\begin{equation}\label{eq: decomposed total variance}
     \Tr\left(\Var
 \left[\widehat{\nabla} J^{R}_{k} \right]\right)\leq \frac{1}{|\mathcal{U}_k|^2}\sum_{\pmb\theta_i\in\mathcal{U}_k}\sum_{\pmb\theta_{i^\prime}\in\mathcal{U}_k} \max_{\ell=1,2,\ldots,d}\left(\Corr^{(\ell)}_{i,i^\prime,k}\right)\left(\pmb\sigma^{R}_{i,k}\right)^\top\left(\pmb\sigma^{R}_{i^\prime,k}\right).
\end{equation}
 \end{proposition}
Proposition~\ref{prop: gradient variance decomposition} implies that replaying past experiences can reduce the variance of LR/CLR policy gradient estimators in  \eqref{eq.LR-gradient} and \eqref{eq.CLR-gradient}, provided the correlations between individual LR/CLR policy gradient estimators (i.e., $\widehat{\nabla} J^{R}_{i,k}$ and $\widehat{\nabla} J^{R}_{i^\prime,k}$ with $\pmb\theta_i$ and $\pmb\theta_{i^\prime}$ included in the replay buffer $\mathcal{U}_k$) are less than 1. 
Theorem~\ref{theorem: bounded bias for LR graident} and Proposition~\ref{prop: gradient variance decomposition} suggest a trade-off in experience replay: \textit{while historical sample reusing reduces the variance of LR/CLR policy gradient estimators, it concurrently increases their bias.} On one hand, replaying less dependent old experiences is desirable; on the other hand, it's necessary to discard too old behavior policies and associated historical samples to prevent bias escalation. {This observation underscores the critical role of selecting replay buffer and buffer size in achieving an optimal balance between variance reduction and bias control in policy optimization.} 

 \section{Finite-Time Convergence Analysis of Experience Replay}\label{sec: convergence analysis}

This section provides a finite-time convergence analysis for policy optimization with ER.
Our goal is to quantify how the convergence rate depends on (i) the replay buffer size $B_K$, (ii) the replay-induced dependence, and (iii) the mixing rate of the underlying Markov chain. The proof of Theorem~\ref{convergence theorem} is provided in Appendix~\ref{appendix: sec convergence} together with supportive lemmas.



\begin{sloppypar}
\begin{theorem}[Convergence]\label{convergence theorem}
Suppose Assumptions \ref{assumption 2} and \ref{assumption 3} hold. Let $\eta_k=\eta_1 k^{-r}$ denote the learning rate used in the $k$-th iteration with two constants $\eta_1\in(0, \frac{1}{4L}]$ and $r\in (0,1)$. Define the {L2 importance-weight norm} as \begin{equation}
w_{i,k}=\sqrt{\E\left[f_{i,k}\left(\pmb{s}^{(i,j)},\pmb{a}^{(i,j)}\right)^2 \right]} \text{ with } f_{i,k}(\pmb{s},\pmb{a})=\frac{\pi_{\pmb\theta_k}(\pmb{a}|\pmb{s})}{\pi_{\pmb\theta_i}(\pmb{a}|\pmb{s})}.
\end{equation} By running Algorithm ~\ref{algo: online} with the replay buffer of size $B_K$, for both LR/CLR policy gradient estimators in Eq.~\eqref{eq.LR-gradient} and \eqref{eq.CLR-gradient} and any $t\leq K-B_K$, we have the rate of convergence
 \begin{align*}
    \frac{1}{K}\sum^K_{k=1}\E\left[\Vert\nabla  J(\pmb\theta_k)\Vert^2\right]  & \leq   \frac{8U_ J /\eta_1}{K^{1-r}} + \frac{4cLM^2}{K}\sum_{k=1}^K{\eta_k} \bar{\rho}_k+
    {4M^2}\varphi(nt) + \frac{2^{r+1}C_3}{(1-r)K^r}  \nonumber\\
    &\quad  + \frac{2^{r+1}C_2\eta_1 t^2}{(1-r)K^r} + \frac{2^{r+1}C_1\eta_1}{1-r}\frac{B_K+t}{K^r} + M^2\frac{B_K+t}{K}
 \end{align*}
where $C_1=\max\{C^\Gamma_1, C^\Gamma_2\}$, $C_2={2nM^3U_\pi  U_f}$, $C_3=\sup_{k\geq 1}\Vert \Bias_k\Vert$ with $\Bias_k=\E\left[\widehat{\nabla} J^{R}_{k}\right]-\E\left[\nabla J(\pmb\theta_k)\right]$ and $\bar{\rho}_{k}=\frac{1}{|\mathcal{U}_k|^2}\sum_{\pmb\theta_i\in\mathcal{U}_k}\sum_{\pmb\theta_{i^\prime}\in\mathcal{U}_k}\left|\max_{\ell=1,2,\ldots,d}\left(\Corr^{(\ell)}_{i,i^\prime,k}\right)\right|w_{i,k}w_{i^\prime,k}$. 
 Here $R\in\{LR,CLR\},$ $C_1^\Gamma$ and $C_2^\Gamma$ are defined in Lemma~\ref{appendix lemma: supportive lemma 6}.
Using $\mathcal{O}(\cdot)$ notation gives
 \begin{equation}\label{eq: convergence upper bound}
    \frac{1}{K}\sum^K_{k=1}\E\left[\Vert\nabla  J(\pmb\theta_k)\Vert^2\right] 
    \leq \mathcal{O}\left(\frac{1}{K^{1-r}}\right) + \mathcal{O}\left(\frac{\sum^K_{k=1}\eta_k \bar{\rho}_k}{K}\right) +
    \mathcal{O}\left(\varphi(nt)\right) + \mathcal{O}\left(\frac{t^2}{K^r}\right) + \mathcal{O}\left(\frac{B_K+t}{K^r}\right)  
 \end{equation}
\noindent where $n$ is the number of steps in each iteration. The notation
$\mathcal{O}(\cdot)$ hides constants $c$, $L$, $\eta_1$, $M$, $U_ J$, $n$, $U_f$, $U_\pi$, $\kappa_0$, $\kappa$ and $L_g$. The Lipchitz constant $L_g$ of the sample gradient is defined in Lemma~\ref{lemma: lipchitz continuity of sample gradient estimate}.
\end{theorem}.
\end{sloppypar}

Theorem~\ref{convergence theorem} guarantees the convergence of the policy optimization with ER for both LR or CLR policy gradient estimators. The rate of optimal convergence depends on three key factors: (1) sample covariance $\bar{\rho}_k$; (2) mixing rate $\varphi(nt)=\kappa_0 \kappa^{nt}$; and (3) buffer size $B_K$. In short, low covariance between replayed samples and a faster mixing rate of the environment (i.e., smaller $\kappa$) would improve the convergence. 

\begin{remark}[L2 importance-weight norm]
    For fixed policies $\pmb{\theta}_i$ and $\pmb{\theta}_k$, the L2 importance-weight norm $w_{i,k}$ can be expressed using the Rényi divergence of order $2$\footnote{The Rényi divergence of order $\alpha>0,\alpha\neq 1$, between distributions $P$ and $Q$, is defined as
\[
D_{\alpha}(P\|Q)
=
\frac{1}{\alpha - 1}
\log 
\mathbb{E}_{Q}\!\left[ \left(\frac{p(X)}{q(X)}\right)^{\alpha} \right].
\]
}:
\[
w_{i,k}
=
\exp\!\left\{ D_{2}\!\left(\pi_{\pmb{\theta}_i}\,\Vert\, \pi_{\pmb{\theta}_k}\right) \right\}.
\]

This representation highlights a key aspect of the covariance structure in experience replay: the variance of the policy gradient is related to the divergence between the sampling distribution and the target distribution. Larger distributional disparity leads to larger expected likelihood ratios, increased covariance, and consequently slower convergence. However, despite its simplicity, this formulation is not suitable for analysis. Conditioning on $\pmb\theta_k$—which is measurable with respect to filtration up to step 
$k$—while treating past samples from step 
$i$ as random is statistically inconsistent.
\end{remark}

\subsection{Clipping Helps Bound Covariance}
For CLR, the likelihood ratio is truncated and therefore, the L2 importance-weight norm is bounded: $$w_{i,k}=\sqrt{\E\left[f_{i,k}\left(\pmb{s}^{(i,j)},\pmb{a}^{(i,j)}\right)^2 \right]} \leq U_f.$$
It results in $\bar{\rho}_k\leq U_f^2$ given the fact of $|\Corr^{(\ell)}_{i,i^\prime,k}|\leq 1$. Since the learning rate $\eta_k=\eta_1 k^{-r}$ with $\eta_1\in(0, \frac{1}{4L}]$ and $r\in (0,1)$, the second term of Eq.~\eqref{eq: convergence upper bound} can be simplified to \begin{equation}\label{eq: big O of weighted rho}
    \frac{1}{K}\sum_{k=1}^K{\eta_k}\bar{\rho}_k\leq\frac{U_f^2}{K}\sum_{k=1}^K{\eta_k}\leq\frac{U_f^2}{1-r} K^{-r}=\mathcal{O}\left(\frac{1}{K^r}\right).
\end{equation}
 Consequently, Corollary~\ref{corollary: convergence rate} can be obtained immediately from Eq.~\eqref{eq: big O of weighted rho}
\begin{sloppypar}
\begin{corollary}\label{corollary: convergence rate}
    Suppose Assumptions \ref{assumption 2} and \ref{assumption 3} hold. Under the same configurations as Theorem~\ref{convergence theorem}, we have the rate of convergence
 \begin{equation*}
    \frac{1}{K}\sum^K_{k=1}\E\left[\Vert\nabla  J(\pmb\theta_k)\Vert^2\right] 
    = \mathcal{O}\left(\frac{1}{K^{1-r}}\right) + \mathcal{O}\left(\frac{1}{K^r}\right) + 
    \mathcal{O}\left(\varphi(nt)\right) + \mathcal{O}\left(\frac{t^2}{K^r}\right) + \mathcal{O}\left(\frac{B_K+t}{K^r}\right).  
 \end{equation*}
\end{corollary} 
\end{sloppypar}
Corollary~\ref{corollary: convergence rate} shows that clipping is crucial for controlling the variance of ER-based policy gradients and ensuring stable convergence.

\subsection{Simplified Rate via a Mixing-Time Choice of the Lag} \label{subsec: simplified rate}

Setting the lag term to the mixing time $t=t_K^\star$ yields Corollary~\ref{cor: convergence rate} with a simplified rate expression. Consistent with the conclusion in gradient bias and variance analysis, Corollary~\ref{cor: convergence rate} indicates that the convergence is guaranteed when the buffer size scales at a rate lower than $\mathcal{O}(K^{r})$ acknowledging that replaying old samples introduces extra bias and when sample covariance $\bar{\rho}_k$ is low.

\begin{sloppypar}
\begin{corollary}\label{cor: convergence rate} 
    Suppose Assumptions \ref{assumption 2} and \ref{assumption 3} hold. Under the same configurations as Theorem~\ref{convergence theorem}, by setting $t=t^\star_K=\log_{\kappa}K^{-r/n}$, we have the rate of convergence
 \begin{equation*}
    \frac{1}{K}\sum^K_{k=1}\E\left[\Vert\nabla  J(\pmb\theta_k)\Vert^2\right] 
    \leq \mathcal{O}\left(\frac{1}{K^{1-r}}\right)+ \mathcal{O}\left(\frac{\sum^K_{k=1}\eta_k \bar{\rho}_k}{K}\right) + \mathcal{O}\left(\frac{(t_K^\star)^2}{K^r}\right) + \mathcal{O}\left(\frac{B_K+t_K^\star}{K^r}\right),  
 \end{equation*}
\noindent where the notation
$\mathcal{O}(\cdot)$ hides constants $c$, $L$, $\eta_1$, $M$, $U_ J$, $n$, $U_f$, $U_\pi$, $\kappa_0$, $\kappa$ and $L_g$. The Lipchitz constant $L_g$ of the sample gradient is defined in Lemma~\ref{lemma: lipchitz continuity of sample gradient estimate}.
\end{corollary} 
\end{sloppypar}

\section{Discussion}
Theorem~\ref{convergence theorem} guarantees the convergence of policy optimization with ER for both LR and CLR
policy-gradient estimators. The bound reveals three mechanisms through which
replay affects convergence.

\textbf{(i) Replay dependence via $\bar{\rho}_k$.}
The quantity $\bar{\rho}_k$ aggregates (a) the pairwise correlation among replayed samples and (b) the second-moment
scale $w_{i,k}$ of the per-sample gradient contributions. Intuitively, ER improves sample efficiency only when it
increases the \emph{effective} number of informative samples. If the replayed data are highly correlated (e.g., many
near-duplicate transitions) or have heavy-tailed importance weights (large $w_{i,k}$), then reusing them does not
behave like drawing additional i.i.d.\ samples; instead, it amplifies variance and can slow down convergence through
the term $\frac{1}{K}\sum_{k=1}^K \eta_k\bar{\rho}_k$. 

\textbf{(ii) Markov mixing through $\varphi(nt)$.}
The mixing term $\varphi(nt)=\kappa_0\kappa^{nt}$ quantifies how quickly the Markov chain forgets its past over $nt$
steps. Faster mixing (smaller $\kappa$) reduces temporal dependence and makes replayed samples behave closer to i.i.d.
draws. When mixing is slow, dependence persists across time, so replaying nearby transitions can yield little new
information and can worsen the dependence term. In this sense, $\varphi(nt)$ can be viewed as an intrinsic property
of the environment that determines how aggressively one can replay data without incurring large dependence penalties.

\textbf{(iii) Buffer horizon and staleness via $B_K$ and $t$.}
The remaining $B_K$- and $t$-dependent terms quantify a \emph{staleness penalty}: as the buffer grows, replayed samples
tend to be generated under increasingly outdated policies, and the analysis requires a larger lag $t$ to decouple
the replayed data from the current iterate. This creates a tradeoff: increasing $B_K$ improves replay opportunities,
but also increases the bias captured by $(B_K+t)/K^r$. The theorem therefore formalizes a common empirical observation: overly large buffers
containing stale data can harm learning even if they increase the nominal sample count.


\section*{Acknowledgments}
We gratefully acknowledge the support from the National Science Foundation under Grant CAREER CMMI-2442970.

\bibliographystyle{unsrtnat}
\bibliography{references}  





\clearpage
\appendix
\section{Proof Sketch for Theorem~\ref{theorem: bounded bias for LR graident}} \label{appendix sec: proof sketch}
The following sketch describes the bias decomposition and connects each component with its physical meaning. 
\begin{proof} (\textbf{Sketch})
In the proof sketch, we focus on LR policy gradient estimator \eqref{eq.LR-gradient}; but the CLR estimator \eqref{eq.CLR-gradient} shares similar proof steps. The bias of the gradient estimator is,
\begin{align}
\left\Vert\E\left[\widehat{\nabla  J}^{LR}_k\right] -\E[\nabla  J(\pmb\theta_k)]\right\Vert &= \left\Vert \frac{1}{|\mathcal{U}_k|n}
 \sum_{{\pmb\theta_i}\in \mathcal{U}_k}
 \sum^{n}_{j=1}\E\left[
\frac{\pi_{\pmb\theta_k}\left(\pmb{a}^{(i,j)}|\pmb{a}^{(i,j)}\right)}
 {\pi_{\pmb\theta_i}\left(\pmb{a}^{(i,j)}|\pmb{a}^{(i,j)}\right)} g\left(\pmb{x}^{(i,j)}|\pmb\theta_k\right)-\nabla J(\pmb\theta_k)\right]\right\Vert \nonumber\\
 & \leq \frac{1}{|\mathcal{U}_k|n}
 \sum_{{\pmb\theta_i}\in \mathcal{U}_k}
 \sum^{n}_{j=1}\left\Vert \Delta_{\textbf{(i)}}(\pmb{x}^{(i,j)},\pmb\theta_i,\pmb\theta_k)\right\Vert\label{sketch eq1: theorem bounded bias for LR graident}
\end{align}
where $\Delta_{\textbf{(i)}}(\pmb{x}^{(i,j)},\pmb\theta_i,\pmb\theta_k) \defeq \E\left[
\frac{\pi_{\pmb\theta_k}\left(\pmb{a}^{(i,j)}|\pmb{a}^{(i,j)}\right)}
 {\pi_{\pmb\theta_i}\left(\pmb{a}^{(i,j)}|\pmb{a}^{(i,j)}\right)} g\left(\pmb{x}^{(i,j)}|\pmb\theta_k\right)-\nabla J(\pmb\theta_k)\right]$.
 To separately quantify the impact induced by policy update, sample reuse, and Markovian noise, $\Delta_{\textbf{(i)}}(\pmb{x}^{(i,j)},\pmb\theta_i,\pmb\theta_k)$ 
is decomposed below based on triangle inequality theorem, i.e.,
\begin{align}
\left\Vert\Delta_{\textbf{(i)}}(\pmb{x}^{(i,j)},\pmb\theta_i,\pmb\theta_k)\right\Vert&\leq\left\Vert\E\left[\widehat{\nabla  J}^{LR}\left({\pmb{x}}^{(i,j)}, \pmb\theta_i,\pmb\theta_{k}\right)-\widehat{\nabla  J}^{LR}\left({\pmb{x}}^{(i,j)}, \pmb\theta_i,\pmb\theta_{i-t}\right) \right] \right\Vert \label{sketch eq: bounded bias 1}\\
&\quad +\left\Vert\E\left[\widehat{\nabla  J}^{LR}\left({\pmb{x}}^{(i,j)}, \pmb\theta_i,\pmb\theta_{i-t}\right)-\widehat{\nabla  J}\left(\tilde{\pmb{x}}^{(i,j)}, \pmb\theta_i,\pmb\theta_{i-t}\right)\right] \right\Vert \label{sketch eq: bounded bias 2}\\
&\quad + \left\Vert \E\left[\widehat{\nabla  J}\left(\tilde{\pmb{x}}^{(i,j)}, \pmb\theta_i,\pmb\theta_{i-t}\right)-\widehat{\nabla  J}\left(\check{\pmb{x}}^{(i,j)}, \pmb\theta_i,\pmb\theta_{i-t}\right)\right] \right\Vert \label{sketch eq: bounded bias 3}\\
&\quad + \left\Vert \E\left[\widehat{\nabla  J}\left(\check{\pmb{x}}^{(i,j)}, \pmb\theta_i,\pmb\theta_{i-t}\right)-\widehat{\nabla  J}\left(\check{\pmb{x}}^{(i,j)}, \pmb\theta_i,\pmb\theta_{k}\right)\right]\right\Vert \label{sketch eq: bounded bias 4}\\
&\quad + \left\Vert\E\left[\widehat{\nabla  J}\left(\check{\pmb{x}}^{(i,j)}, \pmb\theta_i,\pmb\theta_{k}\right)-\nabla  J\left(\pmb\theta_k\right)\right] \right\Vert. \label{sketch eq: bounded bias 5}
\end{align}
The notations above are defined in Appendix~\ref{appendix subsec: notations}. In brief, $\tilde{\pmb{x}}^{(i,j)}=(\tilde{\pmb{s}}^{(i,j)},\tilde{\pmb{a}}^{(i,j)})$ represents a $j$-th state-action pair collected in iteration $i$. It 
is generated starting
from the state $\pmb{s}^{(i-t,j)}$ by following a \textit{fixed policy} $\pi_{\pmb\theta_{i-t}}$ with $0\leq t < i$ and the corresponding sequence from $\tilde{\pmb{x}}^{(i-t,j)}$ to $\tilde{\pmb{x}}^{(i,j)}$is called auxiliary Markov chain (AMC). In addition, $\check{\pmb{x}}^{(i,j)}=(\check{\pmb{s}}^{(i,j)},\check{\pmb{a}}^{(i,j)})\sim d^{\pi_{\pmb\theta_{i-t}}}(\cdot,\cdot)$ represents a state-action pair \textit{independently} generated from the stationary distribution with a fixed policy $\pi_{\pmb\theta_{i-t}}$ and the corresponding sequence from timestep $(i-t,j)$ to $(i,j)$ is called the stationary Markov chain (SMC). See their formal definitions in Appendix~\ref{appendix subsec: proof techniques}.


By Lipschitz conditions on policy and gradient updates, i.e., $\pi_{\pmb\theta_{k}}(\pmb{x}^{(i,j)})-\pi_{\pmb\theta_{i-t}}(\pmb{x}^{(i,j)})$
and $g\left(\pmb{x}^{(i,j)}|\pmb\theta_k\right)-g\left(\pmb{x}^{(i,j)}|\pmb\theta_{i-t}\right)$ over $\E \Vert\pmb\theta_k - \pmb\theta_{i-t}\Vert$, we can bound the terms \eqref{sketch eq: bounded bias 1} and \eqref{sketch eq: bounded bias 4} with $\mathcal{O}(\sum^k_{\ell=i-t}\eta_\ell)$ by applying Lemmas~\ref{appendix lemma: supportive lemma 1} and \ref{appendix lemma: supportive lemma 4}. 
Further, by applying Lemma~\ref{appendix lemma: supportive lemma 2}, the term \eqref{sketch eq: bounded bias 2} is bounded by the total variation 
$\left\Vert \P\left(\pmb{s}^{(i,j)}\in \cdot|\pmb{s}^{(i-t,j)},\pmb\theta_{i-t}\right)-\P\left(\tilde{\pmb{s}}^{(i,j)}\in \cdot|\pmb{s}^{(i-t,j)},\pmb\theta_{i-t}\right)\right\Vert_{TV}$ in terms of $\mathcal{O}(t\sum^i_{\ell = i-t}\eta_\ell)$. Basically, to account for the sequential policy updates from $\pmb\theta_{i-t}$ to $\pmb\theta_{i}$,
this is achieved by recursively bounding total variation
between $\P\left({\pmb{s}}^{(i,j)}|\pmb{s}^{(i-t,j)},\pmb\theta_{i-t}\right)$
and $\P\left(\tilde{\pmb{s}}^{(i,j)}|\pmb{s}^{(i-t,j)},\pmb\theta_{i-t}\right)$
by using the telescoping sum of the total variations from the timestep $(i-t,j)$ to the timestep $(i,j)$. Then by applying Lemma~\ref{appendix lemma: supportive lemma 3} to account for the dependence impact induced by Markovian noise, the term \eqref{sketch eq: bounded bias 3} is bounded by the total variation distance between the MC state transition probability $\P\left(\tilde{\pmb{s}}^{(i,j)}\in \cdot|\pmb{s}^{(i-t,j)}; \pmb\theta_{i-t}\right)$ and the independent stationary distribution $d^{\pi_{\pmb\theta_{i-t}}}(\cdot)$ following the uniform ergodicity assumption (\textbf{A.2}), which leads to $\mathcal{O}(\varphi(nt))$.
The last term \eqref{sketch eq: bounded bias 5} can be bounded by the total variation between two stationary distributions $d^{\pi_{\pmb\theta_k}}(\cdot)$ and $d^{\pi_{\pmb\theta_{i-1}}}(\cdot)$. By applying Lemma~\ref{appendix lemma: supportive lemma 5}, the difference is associated to policy update
$\E \Vert\pmb\theta_k - \pmb\theta_{i-t}\Vert$ and bounded by $\mathcal{O}(\sum^k_{\ell=i-t}\eta_\ell)$.

Therefore, we can bound the bias of the single-sample LR gradient estimator, 
\begin{equation*}
\left\Vert\Delta_{\textbf{(i)}}(\pmb{x}^{(i,j)},\pmb\theta_i,\pmb\theta_k)\right\Vert \leq 
2M\varphi(nt)+2nM^2 U_\pi t\sum^i_{\ell = i-t}\eta_\ell  + M(2MC_d+2L_g+M U_\pi)\sum^k_{\ell=i-t}\eta_\ell, 
\end{equation*}
from which the conclusion follows. The detailed proof is provided in Appendix~\ref{appendix sec: bias of policy gradient} and \ref{appendix sec: proof of technical lemmas}.
\end{proof}

\section{Proof of Preliminary Lemmas}\label{appendix sec: proof of preliminary lemmas}

\subsection{Key Proof Technique}\label{appendix subsec: proof techniques}
The key proof techniques in Theorems~\ref{theorem: bounded bias for LR graident} and \ref{convergence theorem} rely on employing the uniform ergodicity Assumption~(\textbf{A.2}). This technique was
first introduced by \cite{bhandari2018finite} to address the Markovian noise in policy evaluation. 
\cite{zou2019finite} extended its usage to the Q-learning setting and \cite{wu2020finite} further considered the situation with the policy parameter changing. In this work, we take two steps further: (1) reusing historical samples collected under old behavior policies, thereby causing a complex dependence structure; and (2) using a simple but biased LR-based gradient estimate (i.e., ratio of target and behavior policies), which introduces additional theoretical difficulty.

Suppose that a historical sample $\pmb{x}^{(i,j)}=\left(\pmb{s}^{(i,j)},\pmb{a}^{(i,j)}\right)$ is selected and replayed at iteration $k$ with $i<k$. Due to Markovian noise and policy updates, all samples are correlated and their dependencies can be analyzed by using the Auxiliary Markov Chain (AMC) and Stationary Markov Chain (SMC) constructed below.
To elucidate the effects of Markovian noise, policy update, and historical sample reuse, AMC and SMC will be utilized along with uniform ergodicity and triangle inequality to bound the dependency effect induced by Markovian noise and experience reply in the proof of Theorems~\ref{theorem: bounded bias for LR graident} and \ref{convergence theorem}.

\vspace{0.05in}
\noindent\textbf{Original Markov Chain}: For reference, we first present the original Markov chain with policy update, i.e.,
\begin{gather}
    \pmb{s}^{(i-t,1)}\stackrel{\pmb\theta_{i-t}}{\longrightarrow} \pmb{a}^{(i-t,1)}
    \stackrel{\P}{\longrightarrow} \cdots\stackrel{\P}{\longrightarrow}
    {\pmb{s}^{(i-t, j)}\stackrel{\pmb\theta_{i-t}}{\longrightarrow} \pmb{a}^{(i-t, j)}\stackrel{\P}{\longrightarrow}
    \cdots \stackrel{\P}{\longrightarrow} \pmb{s}^{(i-t,n
    )} \stackrel{\pmb\theta_{i-t}}{\longrightarrow} {\pmb{a}}^{(i-t,n)}\stackrel{\P}{\longrightarrow}} \nonumber\\
     {\cdots} \nonumber\\
     {\pmb{s}^{(i,1)}\stackrel{\pmb\theta_{i}}{\longrightarrow} \pmb{a}^{(i,1)}\stackrel{\P}{\longrightarrow} 
     \cdots \stackrel{\P}{\longrightarrow}
    \pmb{s}^{(i, j)}\stackrel{\pmb\theta_{i}}{\longrightarrow}} {\pmb{a}^{(i, j)}}\stackrel{\P}{\longrightarrow}
    \cdots \stackrel{\P}{\longrightarrow} \pmb{s}^{(i,n
    )}\stackrel{\pmb\theta_{i}}{\longrightarrow}{\pmb{a}}^{(i-t,n)}\stackrel{\P}{\longrightarrow} \nonumber\\
    \cdots \nonumber\\
    \pmb{s}^{(k,1)}\stackrel{\pmb\theta_{k}}{\longrightarrow} \pmb{a}^{(k,1)}\stackrel{\P}{\longrightarrow}  \cdots
    \pmb{s}^{(k, j+1)}\stackrel{\pmb\theta_{k}}{\longrightarrow} \pmb{a}^{(k, j+1)}\stackrel{\P}{\longrightarrow}
    \cdots \stackrel{\P}{\longrightarrow} \pmb{s}^{(k,n
    )} \stackrel{\pmb\theta_{i-t}}{\longrightarrow} \tilde{\pmb{a}}^{(k,n)}\stackrel{\P}{\longrightarrow}\cdots \label{eq: folded original MC}
\end{gather}
where $\P$ represents the state transition model on $\P(\pmb{s}^\prime\in\cdot|\pmb{s},\pmb{a})$ and $t$ is an integer with $0\leq t<i$.

\vspace{0.05in}
\noindent\textbf{Auxiliary Markov Chain (AMC)}:  The analysis of Markovian noise relies on the
auxiliary Markov chain introduced by \cite{zou2019finite} which constructs a state-action sequence by following a fixed policy $\pi_{\pmb\theta_{i-t}}$ with $0\leq t < i$,
\begin{gather}
    \pmb{s}^{(i-t,1)}\stackrel{\pmb\theta_{i-t}}{\longrightarrow} {\pmb{a}}^{(i-t,1)}\stackrel{\P}{\longrightarrow} 
     \cdots\stackrel{\P}{\longrightarrow}
    \textcolor{brown}{\tilde{\pmb{s}}^{(i-t, j)}}\textcolor{brown}{\stackrel{\pmb\theta_{i-t}}{\longrightarrow} \tilde{\pmb{a}}^{(i-t, j)}\stackrel{\P}{\longrightarrow}
    \cdots \stackrel{\P}{\longrightarrow} \tilde{\pmb{s}}^{(i-t,n
    )} \stackrel{\pmb\theta_{i-t}}{\longrightarrow} \tilde{\pmb{a}}^{i-t,n}\stackrel{\P}{\longrightarrow}} \nonumber\\
    \textcolor{brown}{\tilde{\pmb{s}}^{(i-t +1,1)}\stackrel{\pmb\theta_{i-t}}{\longrightarrow} \tilde{\pmb{a}}^{(i-t+1,1)}\stackrel{\P}{\longrightarrow} 
    \cdots \stackrel{\P}{\longrightarrow} \tilde{\pmb{s}}^{(i-t+1,n
    )} \stackrel{\pmb\theta_{i-t}}{\longrightarrow} \tilde{\pmb{a}}^{i-t,n}\stackrel{\P}{\longrightarrow}} \nonumber\\
     \textcolor{brown}{\cdots} \nonumber\\
     \textcolor{brown}{\tilde{\pmb{s}}^{(i,1)}\stackrel{\pmb\theta_{i-t}}{\longrightarrow} \tilde{\pmb{a}}^{(i,1)}\stackrel{\P}{\longrightarrow} 
     \cdots \stackrel{\P}{\longrightarrow}
    \tilde{\pmb{s}}^{(i, j)}\stackrel{\pmb\theta_{i-t}}{\longrightarrow}}\textcolor{brown}{\tilde{\pmb{a}}^{(i, j)}}\stackrel{\P}{\longrightarrow}
     \ldots  \label{eq: folded auxiliary MC}
\end{gather}
Since a fixed policy $\pi_{\pmb\theta_{i-t}}$ is used, the highlighted trajectory part associates with the difference comparing with original Markov chain with policy updates.

\vspace{0.05in}
\noindent\textbf{Stationary Markov Chain (SMC)}: The analysis of the interdependence of behavior policies relies on bounding the distance between their policy parameters. To achieve this, we construct a hypothetical sample path (highlighted trajectory part below)  wherein each sample in the path is independently drawn from the fixed stationary distribution $(\pmb{s},\pmb{a}) \sim d^{\pi_{\pmb\theta_{i-t}}}(\pmb{s},\pmb
a)$. It is worth noting that the chain we created is no longer a Markov chain but a sequence of independent samples.
\begin{gather}
{\pmb{s}}^{(i-t,1)}\stackrel{\pmb\theta_{i-t}}{\longrightarrow} {\pmb{a}}^{(i-t,1)}\stackrel{\P}{\longrightarrow} 
     \cdots\stackrel{\P}{\longrightarrow}
    \textcolor{brown}{\check{\pmb{s}}^{(i-t, j)}\stackrel{\pmb\theta_{i-t}}{\longrightarrow}\check{\pmb{a}}^{(i-t, j)}\stackrel{d^{\pi_{\pmb\theta_{i-t}}}}{\longrightarrow}
    \cdots\stackrel{d^{\pi_{\pmb\theta_{i-t}}}}{\longrightarrow} \check{\pmb{s}}^{(i-t, n)}\stackrel{\pmb\theta_{i-t}}{\longrightarrow} \check{\pmb{a}}^{(i-t, n)} \stackrel{d^{\pi_{\pmb\theta_{i-t}}}}{\longrightarrow}} \nonumber\\
\textcolor{brown}{\check{\pmb{s}}^{(i-t+1,1)}\stackrel{\pmb\theta_{i-t+1}}{\longrightarrow} \check{\pmb{a}}^{(i-t+1,1)}\stackrel{d^{\pi_{\pmb\theta_{i-t}}}}{\longrightarrow}
    \cdots \stackrel{d^{\pi_{\pmb\theta_{i-t}}}}{\longrightarrow}\check{\pmb{s}}^{(i-t+1, n)}\stackrel{\pmb\theta_{i-t}}{\longrightarrow} \check{\pmb{a}}^{(i-t+1, n)}\stackrel{d^{\pi_{\pmb\theta_{i-t}}}}{\longrightarrow}} \nonumber\\
     \textcolor{brown}{\cdots} \nonumber\\
     \textcolor{brown}{\check{\pmb{s}}^{(i,1)}\stackrel{\pmb\theta_{i-t}}{\longrightarrow} \check{\pmb{a}}^{(i,1)}  \stackrel{d^{\pi_{\pmb\theta_{i-t}}}}{\longrightarrow}
     \cdots \stackrel{d^{\pi_{\pmb\theta_{i-t}}}}{\longrightarrow}
    \check{\pmb{s}}^{(i, j)}\stackrel{\pmb\theta_{i-t}}{\longrightarrow}\check{\pmb{a}}^{(i, j)}}\stackrel{\P}{\longrightarrow}
    \cdots  \label{eq: stationary transitions}
\end{gather}

We can concatenate the nested chains \eqref{eq: folded original MC} and \eqref{eq: folded auxiliary MC} by reindexing the trajectory $t\leftarrow nt$ and  $\tau\leftarrow n (i-1)+j$ such that \begin{align*}
    (\pmb{s}_{\tau-t},\pmb{a}_{\tau-t})\leftarrow\left(\pmb{s}^{(i-t,j)}, \pmb{a}^{(i-t,j)}\right) \text{  and  } (\pmb{s}_\tau,\pmb{a}_\tau)\leftarrow\left(\pmb{s}^{(i,j)}, \pmb{a}^{(i,j)}\right).
\end{align*} 
As a result, the subsequence from $\left(\pmb{s}^{(i-t,j)}, \pmb{a}^{(i-t,j)}\right)$ to $\left(\pmb{s}^{(i,j)}, \pmb{a}^{(i,j)}\right)$  in Markov chains \eqref{eq: folded auxiliary MC} and \eqref{eq: folded original MC} can be concatenated into a single chain shown below.

\vspace{0.05in}
\noindent\textbf{(Concatenated) Auxilary Markov Chain:}
    \begin{small}
    \begin{equation}\label{eq: auxiliary MC}
    \tilde{\pmb{s}}_{\tau-t}\stackrel{\pmb\theta_{\tau-t}}{\longrightarrow} \tilde{\pmb{a}}_{\tau-t}\stackrel{\P}{\longrightarrow} \tilde{\pmb{s}}_{\tau-t+1}\stackrel{\pmb\theta_{\tau-t}}{\longrightarrow} \tilde{\pmb{a}}_{\tau-t+1}\stackrel{\P}{\longrightarrow} \tilde{\pmb{s}}_{\tau-t+2}\stackrel{\pmb\theta_{\tau-t}}{\longrightarrow} \tilde{\pmb{a}}_{\tau-t+2}\stackrel{\P}{\longrightarrow}\cdots \stackrel{\P}{\longrightarrow} \tilde{\pmb{s}}_{\tau}\stackrel{\pmb\theta_{\tau-t}}{\longrightarrow} \tilde{\pmb{a}}_{\tau}\stackrel{\P}{\longrightarrow}\tilde{\pmb{s}}_{\tau+1}.
\end{equation}
\end{small}

\noindent\textbf{(Concatenated) Original Markov Chain:}
\begin{small}
\begin{equation}\label{eq: original MC}
    \pmb{s}_{\tau-t}\stackrel{\pmb\theta_{\tau-t}}{\longrightarrow} \pmb{a}_{\tau-t}\stackrel{\P}{\longrightarrow} \pmb{s}_{\tau-t+1}\stackrel{\pmb\theta_{\tau-t+1}}{\longrightarrow} \pmb{a}_{\tau-t+1}\stackrel{\P}{\longrightarrow} \pmb{s}_{\tau-t+2}\stackrel{\pmb\theta_{\tau-t+2}}{\longrightarrow} \pmb{a}_{\tau-t+2}\stackrel{\P}{\longrightarrow}\cdots \stackrel{\P}{\longrightarrow} \pmb{s}_{\tau}\stackrel{\pmb\theta_{\tau}}{\longrightarrow} \pmb{a}_{\tau}\stackrel{\P}{\longrightarrow}\pmb{s}_{\tau+1}.
\end{equation}
\end{small}

Therefore, the existing results on the concatenated AMC (single chain MDP) can be directly transferred to our setting. We list those results as supportive lemmas in Section~\ref{appendix subsec: lipschitz continuity} 

\subsection{Notations for Main Theorems} \label{appendix subsec: notations}

We use $\tilde{\pmb{x}}^{(i,j)}=(\tilde{\pmb{s}}^{(i,j)},\tilde{\pmb{a}}^{(i,j)})$ to denote a state-action pair sample generated from the Auxiliary Markov Chain (AMC) that constructs a state-action sequence by following a fixed policy $\pi_{\pmb\theta_{i-t}}$ with $0\leq t < i$.
Let $\check{\pmb{x}}^{(i,j)}=(\check{\pmb{s}}^{(i,j)},\check{\pmb{a}}^{(i,j)})$ denote a sample generated from the stationary distribution $d^{\pi_{\pmb\theta_{i-t}}}(\pmb{s},\pmb{a})$.
We define some notations here for clarification:
\begin{align*}
    \pmb{x}_t &= (\pmb{s}_t, \pmb{a}_t)  \\
    \widehat{\nabla} J^{LR}\left(\pmb{x}^{(i,j)}, \pmb\theta_i,\pmb\theta_k\right)&= \frac{\pi_{\pmb\theta_k}\left(\pmb{a}^{(i,j)}|\pmb{s}^{(i,j)}\right)}{\pi_{\pmb\theta_i}\left(\pmb{a}^{(i,j)}|\pmb{s}^{(i,j)}\right)} g\left(\pmb{x}^{(i,j)}|\pmb\theta_k\right) \\
     \widehat{\nabla} J^{LR}\left(\pmb{x}^{(i,j)}, \pmb\theta_i,\pmb\theta_{i-t}\right)&= \frac{\pi_{\pmb\theta_{i-t}}\left(\pmb{a}^{(i,j)}|\pmb{s}^{(i,j)}\right)}{\pi_{\pmb\theta_i}\left(\pmb{a}^{(i,j)}|\pmb{s}^{(i,j)}\right)} g\left(\pmb{x}^{(i,j)}|\pmb\theta_{i-t}\right) \\
    \widehat{\nabla} J^{CLR}\left(\pmb{x}^{(i,j)}, \pmb\theta_i,\pmb\theta_k\right)&= \min\left(\frac{\pi_{\pmb\theta_k}\left(\pmb{a}^{(i,j)}|\pmb{s}^{(i,j)}\right)}{\pi_{\pmb\theta_i}\left(\pmb{a}^{(i,j)}|\pmb{s}^{(i,j)}\right)}, U_f \right)g\left(\pmb{x}^{(i,j)}|\pmb\theta_k\right) \\
     \widehat{\nabla} J^{CLR}\left(\pmb{x}^{(i,j)}, \pmb\theta_i,\pmb\theta_{i-t}\right)&= \min \left(\frac{\pi_{\pmb\theta_{i-t}}\left(\pmb{a}^{(i,j)}|\pmb{s}^{(i,j)}\right)}{\pi_{\pmb\theta_i}\left(\pmb{a}^{(i,j)}|\pmb{s}^{(i,j)}\right)}, U_f\right) g\left(\pmb{x}^{(i,j)}|\pmb\theta_{i-t}\right) \\
    \widehat{\nabla} J\left(\tilde{\pmb{x}}^{(i,j)}, \pmb\theta_i,\pmb\theta_k\right)&= g\left(\tilde{\pmb{x}}^{(i,j)}|\pmb\theta_k\right)\\
    \widehat{\nabla} J\left(\check{\pmb{x}}^{(i,j)}, \pmb\theta_i,\pmb\theta_k\right)&= g\left(\check{\pmb{x}}^{(i,j)}|\pmb\theta_k\right)\\
    \widehat{\nabla} J\left(\check{\pmb{x}}^{(i,j)}, \pmb\theta_i,\pmb\theta_k\right)&= g\left(\check{\pmb{x}}^{(i,j)}|\pmb\theta_k\right)\\
    \Gamma\left(\pmb{x}^{(i,j)},\pmb\theta_i, \pmb\theta_{k}\right) &= \left\langle \nabla J(\pmb\theta_k),\widehat{\nabla} J^{LR}\left(\pmb{x}^{(i,j)}, \pmb\theta_i,\pmb\theta_k\right) -\nabla J(\pmb\theta_k) \right\rangle \\
\Gamma\left(\pmb{x}^{(i,j)},\pmb\theta_i, \pmb\theta_{i-t}\right) &= \left\langle \nabla J(\pmb\theta_{i-t}),\widehat{\nabla} J^{LR}\left(\pmb{x}^{(i,j)}, \pmb\theta_i,\pmb\theta_{i-t}\right) -\nabla J(\pmb\theta_{i-t}) \right\rangle \\
\Gamma\left(\tilde{\pmb{x}}^{(i,j)},\pmb\theta_i, \pmb\theta_{i-t}\right) &= \left\langle \nabla J(\pmb\theta_{i-t}),\widehat{\nabla} J\left(\tilde{\pmb{x}}^{(i,j)}, \pmb\theta_i,\pmb\theta_{i-t}\right) -\nabla J(\pmb\theta_{i-t}) \right\rangle \\
\Gamma\left(\check{\pmb{x}}^{(i,j)},\pmb\theta_i, \pmb\theta_{i-t}\right) &= \left\langle \nabla J(\pmb\theta_{i-t}),\widehat{\nabla} J\left(\check{\pmb{x}}^{(i,j)}, \pmb\theta_i,\pmb\theta_{i-t}\right) -\nabla J(\pmb\theta_{i-t}) \right\rangle \\
\end{align*}

\subsection{Preliminary Supportive Lemmas} \label{appendix subsec: lipschitz continuity}
Lemma~\ref{lemma: bounded of policy gradient} establishes the boundedness of policy gradient and its stochastic estimate.
 There exists a constant $M>0$ such that the $L_2$ norm of the scenario-based policy gradient estimate is bounded, i.e., $\Vert g(\pmb{s},\pmb{a}|\pmb\theta)\Vert \leq M$.
 
\noindent\textbf{Lemma}~\ref{lemma: bounded of policy gradient} (\textbf{Boundedness
 of Stochastic Policy Gradients})
 \textit{For any $\pmb\theta$, the norm of the policy gradient $\nabla  J(\pmb\theta)$ and its scenario-based estimate $g(\pmb{s},\pmb{a}|\pmb\theta)$ is bounded, i.e., 
 $$\Vert \nabla  J(\pmb\theta)\Vert \leq M 
 \text{ and } 
  \Vert g\left(\pmb{s},\pmb{a}|\pmb\theta\right)\Vert \leq M,
  $$
 where $M=\frac{2U_r U_\Theta}{1-\gamma}$.}
 \begin{proof} From Eq.~\eqref{eq.scenariobasedGradient}, we have the scenario-based gradient estimate
 \begin{equation}
 \Vert g\left(\pmb{s},\pmb{a}|\pmb\theta\right) \Vert\leq |Q^{\pi_{\pmb\theta}}(\pmb{s},\pmb{a}) -V^{\pi_{\pmb\theta}}(\pmb{s})| \Vert\nabla \log \pi_{\pmb{\theta}}(\pmb{a}|\pmb{s})\Vert \leq \frac{2U_r U_\Theta}{1-\gamma}.
 \label{eq: boundedness of policy gradient}
 \end{equation}
Step~\eqref{eq: boundedness of policy gradient} follows by applying Assumption~\ref{assumption 2}(ii) and the boundedness of value functions in Eq.~\eqref{eq: bound of action value function} and \eqref{eq: bound of state value function}. Let $M=\frac{2U_r U_\Theta}{1-\gamma}$. The boundedness of $\nabla  J(\pmb\theta)$ follows the fact that
 $$\Vert\nabla J(\pmb{\theta})\Vert=\Vert\E_{(\pmb{s},\pmb{a})\sim d^{\pi_{\pmb\theta}}(\cdot,\cdot)}[g(\pmb{s},\pmb{a}|\pmb\theta)]\Vert\stackrel{(\star)}{\leq} \E_{(\pmb{s},\pmb{a})\sim d^{\pi_{\pmb\theta}}(\cdot,\cdot)}[\Vert g(\pmb{s},\pmb{a}|\pmb\theta)\Vert]\stackrel{(\star\star)}{\leq}\E[M]=M,$$
 where step~($\star$) follows by applying Jensen's inequality and $(\star\star)$ follows by applying \eqref{eq: boundedness of policy gradient}.
 \end{proof}
 
The proof of the following lemma is similar to the proof of Lemma B.2 in \cite{wu2020finite} with the difference that \cite{wu2020finite} considers the case with finite action space.

\begin{lemma}[Lipschitz Continuity of Stationary Distribution]\label{lemma: relations of total variation measures}
Given time indexes $t$ and $\tau$ such that $\tau\geq t\geq 0$, we consider the concatenated original Markov chain defined in \eqref{eq: original MC} and auxiliary Markov chain \eqref{eq: auxiliary MC}. Conditioning on $\pmb{s}_{\tau-t+1}$ and $\pmb\theta_{\tau-t}$, we have
\begin{align}
    \Vert & p(\pmb{s}_{\tau+1}|\pmb{s}_{\tau-t+1},\pmb\theta_{\tau-t}) -p(\tilde{\pmb{s}}_{\tau+1}|\pmb{s}_{\tau-t+1},\pmb\theta_{\tau-t})\Vert_{TV} \nonumber\\
    &\leq  \Vert p(\pmb{s}_{\tau},\pmb{a}_\tau|\pmb{s}_{\tau-t+1},\pmb\theta_{\tau-t}) -p(\tilde{\pmb{s}}_\tau,\tilde{\pmb{a}}_\tau|\pmb{s}_{\tau-t+1},\pmb\theta_{\tau-t})\Vert_{TV}\label{eq:lemma total_variation 1}\\
    \Vert & p(\pmb{s}_{\tau},\pmb{a}_\tau|\pmb{s}_{\tau-t+1},\pmb\theta_{\tau-t}) -p(\tilde{\pmb{s}}_{\tau},\tilde{\pmb{a}}_\tau|\pmb{s}_{\tau-t+1},\pmb\theta_{\tau-t})\Vert_{TV} \nonumber\\
    & \leq     \Vert p(\pmb{s}_\tau|\pmb{s}_{\tau-t+1},\pmb\theta_{\tau-t}) -p(\tilde{\pmb{s}}_\tau|\pmb{s}_{\tau-t+1},\pmb\theta_{\tau-t})\Vert_{TV}+  U_\pi\E[\Vert \pmb\theta_\tau-\pmb\theta_{\tau-t}\Vert]\label{eq:lemma total_variation 3}
\end{align}

\end{lemma}
\begin{proof}
The proof for \eqref{eq:lemma total_variation 1} and \eqref{eq:lemma total_variation 3} is identical to that of Lemma~B.2 in \cite{wu2020finite}. The only difference is that the proof in \cite{wu2020finite} handles a discrete action space and uses $\sum_{a\in\mathcal{A}}$, but our lemma considers both discrete and continuous action space, employing the integral notation $\int_{\mathcal{A}}$ to traverse this space.

\end{proof}

Lemmas~\ref{lemma: lipchitz continuity of stationary distribution} and \ref{lemma: lipchitz continuity of action value} establish the Lipschitz continuity of the stationary distribution and action value function. Similar supportive lemmas appear in many policy search algorithms \citep{xu2020improving,zou2019finite,wu2020finite}. Readers are referred to \citet[Lemma 3 and 4]{xu2020improving} for detailed proofs of these two lemmas.

\begin{lemma}[\cite{xu2020improving}, Lemma 3]\label{lemma: lipchitz continuity of stationary distribution} 
Consider the stationary distribution of the state-action pair $d^{\pi_{\pmb\theta}}(s,a)$.
For any $\pmb{\theta}_1,\pmb{\theta}_2\in\mathbb{R}^d$ it holds that 
\begin{equation*}
    \Vert d^{\pi_{\pmb\theta_1}}(\cdot,\cdot)-d^{\pi_{\pmb\theta_2}}(\cdot,\cdot)\Vert_{TV}\leq C_d \Vert\pmb\theta_1-\pmb\theta_2\Vert
\end{equation*}
where $C_d=U_\pi(1+\lceil\log_{\kappa}\kappa_0^{-1}\rceil+\frac{1}{1-\kappa})$.
\end{lemma}

\begin{lemma}[\cite{xu2020improving}, Lemma 4]\label{lemma: lipchitz continuity of action value} 
Suppose Assumptions \ref{assumption 2} and \ref{assumption 3} hold. For any $\pmb\theta_1,\pmb\theta_2\in \mathbb{R}^d$ and any state-action pair $(\pmb{s},\pmb{a})\in \mathcal{S}\times\mathcal{A}$, we have
$$|Q^{\pi_{\pmb\theta_1}}(\pmb{s},\pmb{a})-Q^{\pi_{\pmb\theta_2}}(\pmb{s},\pmb{a})|\leq L_Q \Vert \pmb\theta_1-\pmb\theta_2\Vert$$
where $L_Q=\frac{2 U_r C_d}{1-\gamma}$
\end{lemma}

\begin{lemma}[Lipchitz Continuity of Advantage Function]\label{lemma: lipchitz continuity of advantage}
Suppose Assumptions \ref{assumption 2} and \ref{assumption 3} hold. For any $\pmb\theta_1,\pmb\theta_2\in \mathbb{R}^d$ and any state-action pair $(\pmb{s},\pmb{a})\in \mathcal{S}\times\mathcal{A}$, we have
$$|A^{\pi_{\pmb\theta_1}}(\pmb{s},\pmb{a})-A^{\pi_{\pmb\theta_2}}(\pmb{s},\pmb{a})|\leq L_A \Vert \pmb\theta_1-\pmb\theta_2\Vert$$
where $L_A=2L_Q+U_ J U_\pi$, $L_Q=\frac{2 U_r C_d}{1-\gamma}$, and $U_ J = U_r/(1-\gamma)$.
\end{lemma}
\begin{proof}
To begin with, we show the Lipchitz continuity of state value function,
\begin{align}
    |V^{\pi_{\pmb\theta_1}}(\pmb{s}) - V^{\pi_{\pmb\theta_2}}(\pmb{s})|
    &=\left|\int_{\mathcal{A}}Q^{\pi_{\pmb\theta_1}}(\pmb{s},\pmb{a})\pi_{\pmb\theta_1}(\pmb{a}|\pmb{s})\dd\pmb{a} - \int_{\mathcal{A}} Q^{\pi_{\pmb\theta_2}}(\pmb{s},\pmb{a}) \pi_{\pmb\theta_2}(\pmb{a}|\pmb{s})\dd\pmb{a}\right| \nonumber\\
    &\leq\int_{\mathcal{A}} |Q^{\pi_{\pmb\theta_1}}(\pmb{s},\pmb{a})-Q^{\pi_{\pmb\theta_2}}(\pmb{s},\pmb{a})|\pi_{\pmb\theta_1}(\pmb{a}|\pmb{s})\dd\pmb{a} \nonumber\\
    & \quad + \int_{\mathcal{A}} \left|Q^{\pi_{\pmb\theta_2}}(\pmb{s},\pmb{a})\right|\left|\pi_{\pmb\theta_1}(\pmb{a}|\pmb{s}) -\pi_{\pmb\theta_2}(\pmb{a}|\pmb{s}) \right|\dd\pmb{a}. \nonumber
\end{align}
By applying Lemma~\ref{lemma: lipchitz continuity of action value}, it holds that $$\int_{\mathcal{A}} |Q^{\pi_{\pmb\theta_1}}(\pmb{s},\pmb{a})-Q^{\pi_{\pmb\theta_2}}(\pmb{s},\pmb{a})|\pi_{\pmb\theta_1}(\pmb{a}|\pmb{s})\dd\pmb{a}\leq L_Q\Vert\pmb\theta_1-\pmb\theta_2\Vert \int_{\mathcal{A}} \pi_{\pmb\theta_1}(\pmb{a}|\pmb{s})\dd\pmb{a}=L_Q\Vert\pmb\theta_1-\pmb\theta_2\Vert.$$ 
By applying Eq.~\eqref{eq: bound of action value function} and Assumption~\ref{assumption 2}, it holds that
$$\int_{\mathcal{A}} \left|Q^{\pi_{\pmb\theta_2}}(\pmb{s},\pmb{a})\right|\left|\pi_{\pmb\theta_1}(\pmb{a}|\pmb{s}) -\pi_{\pmb\theta_2}(\pmb{a}|\pmb{s}) \right|\dd\pmb{a}\leq U_ J U_\pi \Vert\pmb\theta_1-\pmb\theta_2\Vert.$$
Therefore, we have $|V^{\pi_{\pmb\theta_1}}(\pmb{s}) - V^{\pi_{\pmb\theta_2}}(\pmb{s})|\leq (L_Q +U_ J U_\pi)\Vert\pmb\theta_1-\pmb\theta_2\Vert$. Now we prove the Lipchitz continuity of advantage function. By triangle inequality, it holds that
\begin{align}
    |A^{\pi_{\pmb\theta_1}}(\pmb{s},\pmb{a}) - A^{\pi_{\pmb\theta_2}}(\pmb{s},\pmb{a})| &= |Q^{\pi_{\pmb\theta_1}}(\pmb{s},\pmb{a})-V^{\pi_{\pmb\theta_1}}(\pmb{s}) - Q^{\pi_{\pmb\theta_2}}(\pmb{s},\pmb{a}) + V^{\pi_{\pmb\theta_2}}(\pmb{s})| \nonumber\\
    &\leq |Q^{\pi_{\pmb\theta_1}}(\pmb{s},\pmb{a}) - Q^{\pi_{\pmb\theta_2}}(\pmb{s},\pmb{a})| + |V^{\pi_{\pmb\theta_1}}(\pmb{s}) - V^{\pi_{\pmb\theta_2}}(\pmb{s})| \nonumber\\
    &\leq (2L_Q+U_ J U_\pi) \Vert \pmb\theta_1-\pmb\theta_2\Vert,
\end{align}
which completes the proof.
\end{proof}

\begin{lemma}[Lipchitz Continuity of Sample Gradient]\label{lemma: lipchitz continuity of sample gradient estimate}
Suppose Assumptions \ref{assumption 2} and \ref{assumption 3} hold. For any $\pmb\theta_1,\pmb\theta_2\in \mathbb{R}^d$ and any state-action pair $(\pmb{s},\pmb{a})\in \mathcal{S}\times\mathcal{A}$, we have
$$\Vert g(\pmb{s},\pmb{a}|\pmb\theta_1)-g(\pmb{s},\pmb{a}|\pmb\theta_2)\Vert \leq L_g \Vert \pmb\theta_1-\pmb\theta_2\Vert$$
where $L_g=U_\Theta L_A+2U_ J L_\Theta$.
\end{lemma}
\begin{proof} By applying Eq.~(\ref{eq.scenariobasedGradient}) and Minkowski's inequality, we have
\begin{align}
    \Vert g(\pmb{s},\pmb{a}|\pmb\theta_1)-g(\pmb{s},\pmb{a}|\pmb\theta_2)\Vert &= \Vert A^{\pi_{\pmb\theta_1}}(\pmb{s},\pmb{a}) \nabla\log\pi_{\pmb\theta_1}(\pmb{a}|\pmb{s}) -A^{\pi_{\pmb\theta_2}}(\pmb{s},\pmb{a}) \nabla\log\pi_{\pmb\theta_2}(\pmb{a}|\pmb{s}) \Vert \nonumber\\
 &= \Vert A^{\pi_{\pmb\theta_1}}(\pmb{s},\pmb{a}) \nabla\log\pi_{\pmb\theta_1}(\pmb{a}|\pmb{s}) - A^{\pi_{\pmb\theta_2}}(\pmb{s},\pmb{a}) \nabla\log\pi_{\pmb\theta_1}(\pmb{a}|\pmb{s}) \nonumber\\
 & \quad + A^{\pi_{\pmb\theta_2}}(\pmb{s},\pmb{a}) \nabla\log\pi_{\pmb\theta_1}(\pmb{a}|\pmb{s}) -A^{\pi_{\pmb\theta_2}}(\pmb{s},\pmb{a}) \nabla\log\pi_{\pmb\theta_2}(\pmb{a}|\pmb{s}) \Vert \nonumber\\
    &\leq |A^{\pi_{\pmb\theta_1}}(\pmb{s},\pmb{a}) -A^{\pi_{\pmb\theta_2}}(\pmb{s},\pmb{a})|\Vert \nabla\log\pi_{\pmb\theta_1}(\pmb{a}|\pmb{s})\Vert \nonumber\\
    &\quad + |A^{\pi_{\pmb\theta_2}}(\pmb{s},\pmb{a})| \Vert \nabla\log\pi_{\pmb\theta_1}(\pmb{a}|\pmb{s})-\nabla\log\pi_{\pmb\theta_2}(\pmb{a}|\pmb{s}) \Vert. \nonumber
\end{align}
Notice that $|A^{\pi_{\pmb\theta_2}}(\pmb{s},\pmb{a})|\leq 2U_J$ holds because both state- and action-value functions are bounded by $U_J$, as shown in Eq.~\eqref{eq: bound of action value function} and \eqref{eq: bound of state value function}. By applying Lemma~\ref{lemma: lipchitz continuity of advantage} and Assumption~\ref{assumption 2}, we have
\begin{equation*}
\Vert g(\pmb{s},\pmb{a}|\pmb\theta_1)-g(\pmb{s},\pmb{a}|\pmb\theta_2)\Vert \leq U_\Theta L_A\Vert\pmb\theta_1-\pmb\theta_2\Vert + 2U_ J L_\Theta \Vert\pmb\theta_1-\pmb\theta_2\Vert = L_g \Vert\pmb\theta_1-\pmb\theta_2\Vert
\end{equation*}
where $L_g=U_\Theta L_A+2U_ J L_\Theta$.
\end{proof}

\begin{lemma} \label{lemma: variance bound} \textit{At the $k$-th iteration with the target distribution $\pi_{\pmb\theta_k}$, given the replay buffer of policies $\mathcal{U}_k$ storing the behavioral policies $\pi_{\pmb\theta_i}$, 
we have, for $R\in$\{LR,CLR\},
\begin{align*}
&\E\left[\left\Vert\widehat{\nabla} J^{R}_{k}\right\Vert^2\right]\leq \Tr\left(\Var
 \left[\widehat{\nabla} J^{R}_{k} \right]\right) +2\left \Vert \E\left[\nabla  J(\pmb\theta_k)\right]\right\Vert^2+2\left\Vert\E\left[\widehat{\nabla} J^{R}_{k}\right]-\E[\nabla J(\pmb\theta_k)]\right\Vert^2.
\end{align*}
}
\end{lemma}

\begin{proof} The conclusion follows the following steps
\begin{eqnarray}
\lefteqn{
\Tr\left(\Var\left[\widehat{\nabla} J^{R}_{k}\right]\right) }
\nonumber \\
&=& \Tr\left(\E\left[\left(\widehat{\nabla} J^{R}_{k}\right)\left(\widehat{\nabla} J^{R}_{k}\right)^\top\right] - \E\left[\widehat{\nabla} J^{R}_{k}\right]\E\left[\widehat{\nabla} J^{R}_{k}\right]^\top\right)\nonumber\\
&=&\E\left[\left\Vert\widehat{\nabla} J^{R}_{k}\right\Vert^2\right] - \left\Vert\E\left[\widehat{\nabla} J^{R}_{k}\right]\right\Vert^2\nonumber\\
&\geq &\E\left[\left\Vert\widehat{\nabla} J^{R}_{k}\right\Vert^2\right] - 2 \Vert\E[\nabla  J(\pmb\theta_k)]\Vert^2-2\left\Vert\E\left[\widehat{\nabla} J^{R}_{k}\right]-\E\left[\nabla J(\pmb\theta_k)\right]\right\Vert^2
\label{eq.midstep100}
\end{eqnarray}
\end{proof}

\section{Proofs of Main Theorems}\label{appendix sec: proof of main theorems}

\subsection{Proof of Theorem~\ref{theorem: bounded bias for LR graident}} \label{appendix sec: bias of policy gradient}


\textbf{Theorem~\ref{theorem: bounded bias for LR graident}} (Gradient Biasedness)\textbf{.} \textit{Suppose that Assumptions \ref{assumption 2} and \ref{assumption 3} hold. For both LR and CLR policy gradient estimators in Eq.~\eqref{eq.LR-gradient} and \eqref{eq.CLR-gradient}, we consider the policy update rule (\ref{eq: policy gradient update}) with the learning rate $\eta_k$. For any integer $t$ such that $0\leq t < i$, the bias of policy gradient estimators (R$\in$\{LR,CLR\}) can be bounded by 
\begin{equation*}
     \left\Vert\E\left[\widehat{\nabla} J^{R}_k\right] - \E[\nabla  J(\pmb\theta_k)]\right\Vert \leq \frac{Z_1^{R} \cdot t}{|\mathcal{U}_k|}
 \sum_{{\pmb\theta_i}\in \mathcal{U}_k}\sum^i_{\ell = i-t}\eta_\ell + 2M\varphi(nt) + \frac{Z_2^{R}}{|\mathcal{U}_k|}
 \sum_{{\pmb\theta_i}\in \mathcal{U}_k} \sum^k_{\ell=i-t}\eta_\ell, 
\end{equation*}
where $Z_1^{LR}=2nM^2 U_\pi$, $Z_1^{CLR}= 2nM^2 U_\pi U_f$, $Z_2^{LR}=M(2MC_d+2L_g+MU_\pi)$ and $Z_2^{CLR}=M(2MC_d+ (U_f+1)L_g+MU_\pi)$. Here, the Lipchitz constants $C_d$ and $L_g$ are defined in Lemma~\ref{lemma: lipchitz continuity of stationary distribution} and Lemma~\ref{lemma: lipchitz continuity of sample gradient estimate} respectively in Appendix~\ref{appendix subsec: lipschitz continuity}.
}
\begin{proof}
\noindent\textbf{(i)} We start with the proof for the LR policy gradient estimator. Notice that
\begin{align}
    \left\Vert\E\left[\widehat{\nabla} J^{LR}_k\right] -\E[\nabla  J(\pmb\theta_k)]\right\Vert &= \left\Vert \frac{1}{|\mathcal{U}_k|n}
 \sum_{{\pmb\theta_i}\in \mathcal{U}_k}
 \sum^{n}_{j=1}\E\left[
\frac{\pi_{\pmb\theta_k}\left(\pmb{a}^{(i,j)}|\pmb{s}^{(i,j)}\right)}
 {\pi_{\pmb\theta_i}\left(\pmb{a}^{(i,j)}|\pmb{s}^{(i,j)}\right)} g\left(\pmb{x}^{(i,j)}|\pmb\theta_k\right)-\nabla J(\pmb\theta_k)\right]\right\Vert \nonumber\\
 &\leq \frac{1}{|\mathcal{U}_k|n}
 \sum_{{\pmb\theta_i}\in \mathcal{U}_k}
 \sum^{n}_{j=1}\left\Vert \E\left[
\frac{\pi_{\pmb\theta_k}\left(\pmb{a}^{(i,j)}|\pmb{s}^{(i,j)}\right)}
{\pi_{\pmb\theta_i}\left(\pmb{a}^{(i,j)}|\pmb{s}^{(i,j)}\right)} g\left(\pmb{x}^{(i,j)}|\pmb\theta_k\right)-\nabla J(\pmb\theta_k)\right]\right\Vert \nonumber\\
 & = \frac{1}{|\mathcal{U}_k|n}
 \sum_{{\pmb\theta_i}\in \mathcal{U}_k}
 \sum^{n}_{j=1}\left\Vert \Delta_{\textbf{(i)}}(\pmb{x}^{(i,j)},\pmb\theta_i,\pmb\theta_k)\right\Vert,
 \label{eq1: theorem bounded bias for LR graident}
\end{align}
where $\Delta_{\textbf{(i)}}(\pmb{x}^{(i,j)},\pmb\theta_i,\pmb\theta_k)=\E\left[
\frac{\pi_{\pmb\theta_k}\left(\pmb{a}^{(i,j)}|\pmb{s}^{(i,j)}\right)}
 {\pi_{\pmb\theta_i}\left(\pmb{a}^{(i,j)}|\pmb{s}^{(i,j)}\right)} g\left(\pmb{x}^{(i,j)}|\pmb\theta_k\right)-\nabla J(\pmb\theta_k)\right]$ can be decomposed by
\begin{align}
\left\Vert\Delta_{\textbf{(i)}}(\pmb{x}^{(i,j)},\pmb\theta_i,\pmb\theta_k)\right\Vert&\leq\left\Vert\E\left[\widehat{\nabla} J^{LR}\left({\pmb{x}}^{(i,j)}, \pmb\theta_i,\pmb\theta_{k}\right)-\widehat{\nabla} J^{LR}\left({\pmb{x}}^{(i,j)}, \pmb\theta_i,\pmb\theta_{i-t}\right) \right] \right\Vert \nonumber\\
&\quad +\left\Vert\E\left[\widehat{\nabla} J^{LR}\left({\pmb{x}}^{(i,j)}, \pmb\theta_i,\pmb\theta_{i-t}\right)-\widehat{\nabla} J\left(\tilde{\pmb{x}}^{(i,j)}, \pmb\theta_i,\pmb\theta_{i-t}\right)\right] \right\Vert \nonumber\\
&\quad + \left\Vert \E\left[\widehat{\nabla} J\left(\tilde{\pmb{x}}^{(i,j)}, \pmb\theta_i,\pmb\theta_{i-t}\right)-\widehat{\nabla} J\left(\check{\pmb{x}}^{(i,j)}, \pmb\theta_i,\pmb\theta_{i-t}\right)\right] \right\Vert \nonumber\\
&\quad + \left\Vert \E\left[\widehat{\nabla} J\left(\check{\pmb{x}}^{(i,j)}, \pmb\theta_i,\pmb\theta_{i-t}\right)-\widehat{\nabla} J\left(\check{\pmb{x}}^{(i,j)}, \pmb\theta_i,\pmb\theta_{k}\right)\right]\right\Vert \nonumber\\
&\quad + \left\Vert\E\left[\widehat{\nabla} J\left(\check{\pmb{x}}^{(i,j)}, \pmb\theta_i,\pmb\theta_{k}\right)-\nabla  J\left(\pmb\theta_k\right)\right] \right\Vert. \label{eq: bounded bias 5}
\end{align}
By applying Lemmas~\ref{appendix lemma: supportive lemma 1}, \ref{appendix lemma: supportive lemma 2}, \ref{appendix lemma: supportive lemma 3}, \ref{appendix lemma: supportive lemma 4} and \ref{appendix lemma: supportive lemma 5} to the terms in \eqref{eq: bounded bias 5}, we have 
\begin{equation}\label{eq2: theorem bounded bias for LR graident}
\left\Vert\Delta_{\textbf{(i)}}(\pmb{x}^{(i,j)},\pmb\theta_i,\pmb\theta_k)\right\Vert \leq 
2M\varphi(nt)+2nM^2 U_\pi t\sum^i_{\ell = i-t}\eta_\ell  + M(2MC_d+2L_g+M U_\pi)\sum^k_{\ell=i-t}\eta_\ell. 
\end{equation}
The conclusion follows by plugging the upper bound \eqref{eq2: theorem bounded bias for LR graident} to \eqref{eq1: theorem bounded bias for LR graident}.

\vspace{0.1in}
\noindent\textbf{(ii)} The proof for the CLR policy gradient estimator is similar to that of LR one. First,
 \begin{align}
    &\left\Vert\E\left[\widehat{\nabla} J^{CLR}_k\right] -\E[\nabla  J(\pmb\theta_k)]\right\Vert \nonumber\\
    &= \left\Vert \frac{1}{|\mathcal{U}_k|n}
 \sum_{{\pmb\theta_i}\in \mathcal{U}_k}
 \sum^{n}_{j=1}\E\left[
\min\left(\frac{\pi_{\pmb\theta_k}\left(\pmb{a}^{(i,j)}|\pmb{s}^{(i,j)}\right)}
{\pi_{\pmb\theta_i}\left(\pmb{a}^{(i,j)}|\pmb{s}^{(i,j)}\right)},U_f\right) g\left(\pmb{x}^{(i,j)}|\pmb\theta_k\right)-\nabla J(\pmb\theta_k)\right]\right\Vert \nonumber\\
 &\leq \frac{1}{|\mathcal{U}_k|n}
 \sum_{{\pmb\theta_i}\in \mathcal{U}_k}
 \sum^{n}_{j=1}\left\Vert \Delta_{\textbf{(ii)}}(\pmb{x}^{(i,j)},\pmb\theta_i,\pmb\theta_k)\right\Vert \label{eq3: theorem bounded bias for LR graident}
\end{align}
where $\Delta_{\textbf{(ii)}}(\pmb{x}^{(i,j)},\pmb\theta_i,\pmb\theta_k)=\E\left[
\min\left(\frac{\pi_{\pmb\theta_k}\left(\pmb{a}^{(i,j)}|\pmb{s}^{(i,j)}\right)}
{\pi_{\pmb\theta_i}\left(\pmb{a}^{(i,j)}|\pmb{s}^{(i,j)}\right)},U_f\right) g\left(\pmb{x}^{(i,j)}|\pmb\theta_k\right)-\nabla J(\pmb\theta_k)\right]$. Then we decompose $\Delta_{\textbf{(ii)}}(\pmb{x}^{(i,j)},\pmb\theta_i,\pmb\theta_k)$ by
\begin{align*}
\left\Vert\Delta_{\textbf{(ii)}}(\pmb{x}^{(i,j)},\pmb\theta_i,\pmb\theta_k)\right\Vert&\leq\left\Vert\E\left[\widehat{\nabla} J^{CLR}\left({\pmb{x}}^{(i,j)}, \pmb\theta_i,\pmb\theta_{k}\right)-\widehat{\nabla} J^{CLR}\left({\pmb{x}}^{(i,j)}, \pmb\theta_i,\pmb\theta_{i-t}\right) \right] \right\Vert\\
&\quad +\left\Vert\E\left[\widehat{\nabla} J^{CLR}\left({\pmb{x}}^{(i,j)}, \pmb\theta_i,\pmb\theta_{i-t}\right)-\widehat{\nabla} J\left(\tilde{\pmb{x}}^{(i,j)}, \pmb\theta_i,\pmb\theta_{i-t}\right)\right] \right\Vert \\
&\quad + \left\Vert \E\left[\widehat{\nabla} J\left(\tilde{\pmb{x}}^{(i,j)}, \pmb\theta_i,\pmb\theta_{i-t}\right)-\widehat{\nabla} J\left(\check{\pmb{x}}^{(i,j)}, \pmb\theta_i,\pmb\theta_{i-t}\right)\right] \right\Vert\\
&\quad + \left\Vert \E\left[\widehat{\nabla} J\left(\check{\pmb{x}}^{(i,j)}, \pmb\theta_i,\pmb\theta_{i-t}\right)-\widehat{\nabla} J\left(\check{\pmb{x}}^{(i,j)}, \pmb\theta_i,\pmb\theta_{k}\right)\right]\right\Vert \\
&\quad + \left\Vert\E\left[\widehat{\nabla} J\left(\check{\pmb{x}}^{(i,j)}, \pmb\theta_i,\pmb\theta_{k}\right)-\nabla  J\left(\pmb\theta_k\right)\right] \right\Vert
\end{align*}

By applying Lemmas~\ref{appendix lemma: supportive lemma 1}, \ref{appendix lemma: supportive lemma 2}, \ref{appendix lemma: supportive lemma 3}, \ref{appendix lemma: supportive lemma 4} and \ref{appendix lemma: supportive lemma 5}, we have 
\begin{align}\label{eq4: theorem bounded bias for LR graident}
&\left\Vert\Delta_{\textbf{(ii)}}(\pmb{x}^{(i,j)},\pmb\theta_i,\pmb\theta_k)\right\Vert 
\nonumber\\
    &
\leq 2nM^2U_\pi U_f t\sum^i_{\ell=i-t}\eta_\ell + 2M\varphi(nt) + M(2MC_d+(U_f+1) L_g +M U_\pi)\sum^k_{\ell=i-t}\eta_\ell. 
\end{align}
The conclusion follows by obtained by plugging the upper bound \eqref{eq4: theorem bounded bias for LR graident} to \eqref{eq3: theorem bounded bias for LR graident}.
\end{proof}

\subsection{Proof of Proposition~\ref{prop: gradient variance decomposition}}\label{appendix sec: prop 1}
\textbf{Proposition~\ref{prop: gradient variance decomposition}} (Gradient Variance)\textbf{.}
\textit{Suppose that Assumptions \ref{assumption 2} and \ref{assumption 3} hold. At the $k$-th iteration with the target policy $\pi_{\pmb\theta_k}$, given the replay buffer of policies $\mathcal{U}_k$, the total variance of LR/CLR policy gradient estimators ($R\in$\{LR,CLR\}) in Eq.~\eqref{eq.LR-gradient} and \eqref{eq.CLR-gradient} can be decomposed to
\begin{equation*}
     \Tr\left(\Var
 \left[\widehat{\nabla} J^{R}_{k} \right]\right)=\frac{1}{|\mathcal{U}_k|^2}\sum_{\pmb\theta_i\in\mathcal{U}_k}\sum_{\pmb\theta_{i^\prime}\in\mathcal{U}_k} \left(\pmb\sigma^{R}_{i,k}\right)^\top \mathbf{P}_{i,i^\prime,k} \left(\pmb\sigma^{R}_{i^\prime,k}\right)
\end{equation*}
where $\pmb\sigma^{R}_{i,k}=\left(\sqrt{\Var\left[\widehat{\nabla} J^{R,(1)}_{i,k} \right]}, \sqrt{\Var\left[\widehat{\nabla} J^{R,(2)}_{i,k} \right]},\ldots,\sqrt{\Var\left[\widehat{\nabla} J^{R,(d)}_{i,k} \right]}\right)^\top$ and \\ $\mathbf{P}_{i,i^\prime,k}=\Diag\left(\Corr^{(1)}_{i,i^\prime,k},\ldots,\Corr^{(d)}_{i,i^\prime,k}\right)$with correlation coefficients of each gradient estimate pair $\Corr_{i,i^\prime,k}^{(\ell)}=\Corr\left(\widehat{\nabla} J^{R,(\ell)}_{i,k},\widehat{\nabla} J^{R,(\ell)}_{i^\prime,k}\right)$ 
for each $\ell$-th dimension of policy parameters with $\ell=1,2,\ldots,d$. 
By using the greatest element of $\mathbf{P}$, the total variance~\eqref{eq: variance of LR estimaor} is bounded by
\begin{equation*}
     \Tr\left(\Var
 \left[\widehat{\nabla} J^{R}_{k} \right]\right)\leq \frac{1}{|\mathcal{U}_k|^2}\sum_{\pmb\theta_i\in\mathcal{U}_k}\sum_{\pmb\theta_{i^\prime}\in\mathcal{U}_k} \max_{\ell=1,2,\ldots,d}\left(\Corr^{(\ell)}_{i,i^\prime,k}\right)\left(\pmb\sigma^{R}_{i,k}\right)^\top\left(\pmb\sigma^{R}_{i^\prime,k}\right).
\end{equation*}}
\begin{proof}
Due to the Markovian noise and policy update, the observations $\pmb{x}^{(i,j)}$ for any $\pmb\theta_i\in\mathcal{U}_k$ and $j=1,2,\ldots,n$ are dependent. Thus, we have
 \begin{eqnarray}
 \Tr\left(\Var\left[
 \widehat{\nabla} J^{R}_{k} \right]\right)&=&\frac{1}{|\mathcal{U}_k|^2}\Tr\left(\Var\left[ \sum_{\pmb\theta_i\in\mathcal{U}_k}\widehat{\nabla} J^{R}_{i,k} \right]\right)\nonumber\\
 &=&\frac{1}{|\mathcal{U}_k|^2}\sum_{\pmb\theta_i\in\mathcal{U}_k}\sum_{\pmb\theta_{i^\prime}\in\mathcal{U}_k} \Tr\left(\Cov\left[ \widehat{\nabla} J^{R}_{i,k},\widehat{\nabla} J^{R}_{i^\prime,k} \right]\right). 
 \label{eq: LR equality 1}
 \end{eqnarray}

Due to the sample dependence, policy gradient estimates are dependent. Let $\Corr_{i,i^\prime,k}^{(\ell)}=\Corr\left(\widehat{\nabla} J^{R,(\ell)}_{i,k},\widehat{\nabla} J^{R,(\ell)}_{i^\prime,k}\right)$ denote the correlation coefficient of the $\ell$-th element between two individual LR/CLR policy gradient estimates. Then the total variance \eqref{eq: LR equality 1} can be rewritten as 
\begin{small}
\begin{align}
   \Tr\left(\Var\left[
 \widehat{\nabla} J^{R}_{k} \right]\right)&= \frac{1}{|\mathcal{U}_k|^2}\sum_{\pmb\theta_i\in\mathcal{U}_k}\sum_{\pmb\theta_{i^\prime}\in\mathcal{U}_k} \sum^d_{\ell=1} \Corr^{(\ell)}_{i,i^\prime,k}\sqrt{\Var\left[ \widehat{\nabla} J^{R,(\ell)}_{i,k}\right]} \sqrt{\Var\left[\widehat{\nabla} J^{R,(\ell)}_{i^\prime,k} \right]} \nonumber\\
 &= \frac{1}{|\mathcal{U}_k|^2}\sum_{\pmb\theta_i\in\mathcal{U}_k}\sum_{\pmb\theta_{i^\prime}\in\mathcal{U}_k} \left(\pmb\sigma^{R}_{i,k}\right)^\top \mathbf{P}_{i,i^\prime,k} \left(\pmb\sigma^{R}_{i^\prime,k}\right)
\end{align}
\end{small}
where $\pmb\sigma^{R}_{i,k}=\left(\sqrt{\Var\left[\widehat{\nabla} J^{R,(1)}_{i,k} \right]}, \sqrt{\Var\left[\widehat{\nabla} J^{R,(2)}_{i,k} \right]},\ldots,\sqrt{\Var\left[\widehat{\nabla} J^{R,(d)}_{i,k} \right]}\right)^\top$ and \\ $\mathbf{P}_{i,i^\prime,k}=\Diag\left(\Corr^{(1)}_{i,i^\prime,k},\ldots,\Corr^{(d)}_{i,i^\prime,k}\right)=\begin{pmatrix}
    \Corr^{(1)}_{i,i^\prime,k} & 0&\ldots& 0\\
    0 &  \Corr^{(2)}_{i,i^\prime,k} & \ldots & 0 \\
    \vdots & \vdots & \ddots &  \vdots \\
    0 & 0 & \ldots &   \Corr^{(d)}_{i,i^\prime,k} 
\end{pmatrix}$.
\end{proof}




\subsection{Proof of Theorem~\ref{convergence theorem}} \label{appendix: sec convergence}

\begin{lemma}[Proof of $I_1$ in Lemma~\ref{lemma: main lemma for convergence analysis}]\label{I1 in lemma}
Suppose Assumptions \ref{assumption 2} and \ref{assumption 3} hold. Then, 
\begin{equation*}
    I_1 \leq \frac{1}{|\mathcal{U}_k|}\sum_{{\pmb\theta_i}\in \mathcal{U}_k}
 \left(C_1(k-i+t)\eta_{i-t}+C_2 (t+1)t\eta_{i-t} +2M^2\varphi(nt)\right). \nonumber
\end{equation*}
\end{lemma}
\begin{proof}
We have
$$I_1 =\left|\frac{1}{|\mathcal{U}_k|n}
 \sum_{{\pmb\theta_i}\in \mathcal{U}_k}
 \sum^{n}_{j=1}\E\left[\left\langle \nabla J(\pmb\theta_k), 
\frac{\pi_{\pmb\theta_k}\left(\pmb{a}^{(i,j)}|\pmb{a}^{(i,j)}\right)}
 {\pi_{\pmb\theta_i}\left(\pmb{a}^{(i,j)}|\pmb{a}^{(i,j)}\right)} g\left(\pmb{x}^{(i,j)}|\pmb\theta_k\right)-\nabla J(\pmb\theta_k)\right\rangle\right]\right|.$$
 By applying Jensen's inequality, we have
\begin{align}
I_1&\leq \frac{1}{|\mathcal{U}_k|}
 \sum_{{\pmb\theta_i}\in \mathcal{U}_k}
 \left |\frac{1}{n}\sum^{n}_{j=1}\E\left[\left\langle \nabla J(\pmb\theta_k), \widehat{\nabla} J^{R}(\pmb{x}^{(i,j)}, \pmb\theta_i,\pmb\theta_k)-\nabla J(\pmb\theta_k)\right\rangle \right] \right|  \nonumber\\
&\leq \frac{1}{|\mathcal{U}_k|}
 \sum_{{\pmb\theta_i}\in \mathcal{U}_k}
\left| \frac{1}{n}\sum^{n}_{j=1}\E\left[\Gamma\left(\pmb{x}^{(i,j)},\pmb\theta_i, \pmb\theta_k\right)-\Gamma\left(\pmb{x}^{(i,j)},\pmb\theta_i,\pmb\theta_{i-t}\right) \right]\right| \nonumber\\
&\quad +\frac{1}{|\mathcal{U}_k|n}
 \sum_{{\pmb\theta_i}\in \mathcal{U}_k}
\sum^{n}_{j=1}\left|\E\left[\Gamma\left(\pmb{x}^{(i,j)},\pmb\theta_i, \pmb\theta_{i-t}\right)-\Gamma\left(\tilde{\pmb{x}}^{(i,j)},\pmb\theta_i,\pmb\theta_{i-t}\right)\right]\right|\nonumber\\
& \quad +\frac{1}{|\mathcal{U}_k|n}
 \sum_{{\pmb\theta_i}\in \mathcal{U}_k}
 \sum^{n}_{j=1}\left| \E\left[\Gamma\left(\tilde{\pmb{x}}^{(i,j)},\pmb\theta_i, \pmb\theta_{i-t}\right)-\Gamma\left(\check{\pmb{x}}^{(i,j)},\pmb\theta_i,\pmb\theta_{i-t}\right)\right]\right|\nonumber\\
 & \quad +\frac{1}{|\mathcal{U}_k|n}
 \sum_{{\pmb\theta_i}\in \mathcal{U}_k}
\sum^{n}_{j=1}\left|\E\left[\Gamma\left(\check{\pmb{x}}^{(i,j)},\pmb\theta_i, \pmb\theta_{i-t}\right)\right]\right|. \nonumber
\end{align}
By applying Lemmas~\ref{appendix lemma: supportive lemma 6}, \ref{appendix lemma: supportive lemma 7}, \ref{appendix lemma: supportive lemma 8}, and \ref{appendix lemma: supportive lemma 9}, it follows that
\begin{align}
&\left|\frac{1}{n}\sum^{n}_{j=1}\E\left[\left\langle \nabla J(\pmb\theta_k), \widehat{\nabla} J^{R}(\pmb{x}^{(i,j)}, \pmb\theta_i,\pmb\theta_k)-\nabla J(\pmb\theta_k)\right\rangle \right] \right| \nonumber \\
&\leq C^\Gamma_1 (k-i+t)^{1/2}\eta_{i-t} + C_2^\Gamma (k-i+t)\eta_{i-t}+2nM^3U_\pi U_f t^2 \eta_{i-t} +2M^2\varphi(nt)
\nonumber\\
&\leq   C_1 (k-i+t)\eta_{i-t}+{2nM^3U_\pi U_f} t^2 \eta_{i-t} +2M^2\varphi(nt) \label{eq1: conv theorem equation 1} 
\end{align}
where $C_1=\max\{C^\Gamma_1, C^\Gamma_2\}$. The first term of Step~\eqref{eq1: conv theorem equation 1}  follows due to the fact that $(k-i+t)^{1/2} \leq k-i+t$. The second term of Step~\eqref{eq1: conv theorem equation 1} follows from noticing that $\sum^i_{\ell=i-t}\eta_\ell \leq t \eta_{i-t}$ . Let $C_2={2nM^3U_\pi U_f}$. Then, the term $I_1$ becomes
\begin{align}
I_1 &= \frac{1}{|\mathcal{U}_k|}
 \sum_{{\pmb\theta_i}\in \mathcal{U}_k} \left|\frac{1}{n}\sum^{n}_{j=1}\E\left[\left\langle \nabla J(\pmb\theta_k), \widehat{\nabla} J^{R}(\pmb{x}^{(i,j)}, \pmb\theta_i,\pmb\theta_k)-\nabla J(\pmb\theta_k)\right\rangle \right] \right| \nonumber\\
& \leq \frac{1}{|\mathcal{U}_k|}\sum_{{\pmb\theta_i}\in \mathcal{U}_k}
 \left(C_1(k-i+t)\eta_{i-t}+C_2 t^2\eta_{i-t} +2M^2\varphi(nt)\right). \nonumber
\end{align}
\end{proof}

\begin{lemma}[Proof of $I_2$ in Lemma~\ref{lemma: main lemma for convergence analysis}]\label{lemma: total variance of LR policy gradient}
Suppose Assumptions \ref{assumption 2} and \ref{assumption 3} hold. Then,
\begin{align*}
 I_2 &\leq L\eta_k \Tr\left(\Var
 \left[\widehat{\nabla} J^{R}_{k} \right]\right) \\
      \Tr\left(\Var
 \left[\widehat{\nabla} J^{R}_{k} \right]\right) & \leq \frac{M^2}{|\mathcal{U}_k|^2}\sum_{\pmb\theta_i\in\mathcal{U}_k}\sum_{\pmb\theta_{i^\prime}\in\mathcal{U}_k}\max_{\ell=1,2,\ldots,d}\left(\Corr^{(\ell)}_{i,i^\prime,k}\right)w_{i,k} \cdot w_{i^\prime,k} 
\end{align*}
\end{lemma}
\begin{proof}
By applying Proposition~\ref{prop: gradient variance decomposition}, we have
\begin{equation*}
   \Tr\left(\Var
 \left[\widehat{\nabla} J^{R}_{k} \right]\right) \leq \frac{1}{|\mathcal{U}_k|^2}\sum_{\pmb\theta_i\in\mathcal{U}_k}\sum_{\pmb\theta_{i^\prime}\in\mathcal{U}_k} \max_{\ell=1,2,\ldots,d}\left(\Corr^{(\ell)}_{i,i^\prime,k}\right)\left(\pmb\sigma^{R}_{i,k}\right)^\top\left(\pmb\sigma^{R}_{i^\prime,k}\right).
\end{equation*}
By the definition of variance and L2 vector norm, it holds
\begin{align}
\left(\pmb\sigma^{R}_{i,k}\right)^\top\left(\pmb\sigma^{R}_{i^\prime,k}\right) &=\sum^d_{\ell=1} \sqrt{\Var\left[ \widehat{\nabla} J^{R,(\ell)}_{i,k}\right]} \sqrt{\Var\left[\widehat{\nabla} J^{R,(\ell)}_{i^\prime,k} \right]} \nonumber\\
&\leq\sum^d_{\ell=1} \sqrt{\left(\E\left[\left(\widehat{\nabla} J^{R,(\ell)}_{i,k} \right )^2\right]\right)}\sqrt{\left(\E\left[\left(\widehat{\nabla} J^{R,(\ell)}_{i^\prime,k} \right )^2\right]\right)}. \nonumber\\
&\leq  \sqrt{\E\left[\sum^d_{\ell=1}\left(\widehat{\nabla} J^{R,(\ell)}_{i,k} \right )^2\right]} \sqrt{\E\left[\sum^d_{\ell=1}\left(\widehat{\nabla} J^{R,(\ell)}_{i^\prime,k} \right )^2\right]} \label{eq: expected cross variance bound 1} \\
&= \sqrt{\E\left[\left\Vert\widehat{\nabla} J^{R}_{i,k} \right \Vert^2 \right]} \sqrt{\E\left[\left\Vert\widehat{\nabla} J^{R}_{i^\prime,k} \right \Vert^2\right]} \label{eq: expected cross variance bound 2}
\end{align}
where Step~\eqref{eq: expected cross variance bound 1} holds due to Cauchy-Schwarz Inequality. 
By applying \eqref{eq: expected cross variance bound 2}, the total variance becomes
\begin{equation*}
\Tr\left(\Var
 \left[\widehat{\nabla} J^{R}_{k} \right]\right) \leq \frac{1}{|\mathcal{U}_k|^2}\sum_{\pmb\theta_i\in\mathcal{U}_k}\sum_{\pmb\theta_{i^\prime}\in\mathcal{U}_k} \max_{\ell=1,2,\ldots,d}\left(\Corr^{(\ell)}_{i,i^\prime,k}\right)\sqrt{\E\left[\left\Vert\widehat{\nabla} J^{R}_{i,k} \right \Vert^2 \right]} \sqrt{\E\left[\left\Vert\widehat{\nabla} J^{R}_{i^\prime,k} \right \Vert^2\right]}.
\end{equation*}
Let $w_{i,k}=\sqrt{\frac{1}{n}\sum^n_{j=1}\E\left[f_{i,k}\left(\pmb{s}^{(i,j)},\pmb{a}^{(i,j)}\right)^2 \right]}$. Note that for term $\E\left[\left\Vert\widehat{\nabla} J^{R}_{i,k}\right\Vert^2\right]$, it holds 
\begin{align}
    \E\left[\left\Vert\widehat{\nabla} J^{R}_{i,k}\right\Vert^2\right]&\leq\frac{1}{n}\sum^n_{j=1}\E\left[f_{i,k}\left(\pmb{s}^{(i,j)},\pmb{a}^{(i,j)}\right)^2\left\Vert g(\pmb{s}^{(i,j)},\pmb{a}^{(i,j)}|\pmb\theta_k)\right\Vert^2\right] \nonumber\\
    &\leq\frac{M^2}{n}\sum^n_{j=1}\E\left[f_{i,k}\left(\pmb{s}^{(i,j)},\pmb{a}^{(i,j)}\right)^2 \right]\nonumber\\
    &=M^2 w_{i,k}^2, \label{eq: expected norm of LR policy gradients}
\end{align}
where the last step holds due to the conditional independence of $(\pmb{s}^{(i,j)},\pmb{a}^{(i,j)})$, i.e., $$\frac{1}{n}\sum_{j=1}^n\E\left[f_{i,k}\left(\pmb{s}^{(i,j)},\pmb{a}^{(i,j)}\right)^2\right] = \E\left[ \E\left[\left. \frac{1}{n}\sum_{j=1}^n f_{i,k}\left(\pmb{s}^{(i,j)},\pmb{a}^{(i,j)}\right)^2 \right| \pmb\theta_i\right]\right]=w_{i,k}^2.$$
Then we have
\begin{align}
  \Tr\left(\Var
 \left[\widehat{\nabla} J^{R}_{k} \right]\right) &\leq \frac{M^2}{|\mathcal{U}_k|^2}\sum_{\pmb\theta_i\in\mathcal{U}_k}\sum_{\pmb\theta_{i^\prime}\in\mathcal{U}_k}\max_{\ell=1,2,\ldots,d}\left(\Corr^{(\ell)}_{i,i^\prime,k}\right)w_{i,k} \cdot w_{i^\prime,k}
\end{align}
\end{proof}


\begin{lemma}[Proof of $I_3$ in Lemma~\ref{lemma: main lemma for convergence analysis}]\label{I3 in lemma}
Suppose Assumptions \ref{assumption 2} and \ref{assumption 3} hold. Then,
\begin{equation*}
    I_3=2L\eta_k \Vert\Bias_k \Vert^2.
\end{equation*}
\end{lemma}
\begin{proof}
By applying Theorem~\ref{theorem: bounded bias for LR graident}, we have

\begin{equation}\label{eq-1: I3 in lemma 9}
I_3= 2L\eta_k\left\Vert\E\left[\widehat{\nabla} J^{R}_{k}\right]-\E\left[\nabla J(\pmb\theta_k)\right]\right\Vert^2 =2L\eta_k \Vert\Bias_k \Vert^2,
\end{equation}
where the bias is a convergent sequence given by Theorem~\ref{theorem: bounded bias for LR graident}, 
i.e.,
\begin{align*}
\Vert\Bias_k \Vert &\leq\frac{Z_1^{R} (t+1)}{|\mathcal{U}_k|}
 \sum_{{\pmb\theta_i}\in \mathcal{U}_k}\sum^i_{\ell = i-t}\eta_\ell + 2M\varphi(nt) + \frac{Z_2^{R}}{|\mathcal{U}_k|}
 \sum_{{\pmb\theta_i}\in \mathcal{U}_k} \sum^k_{\ell=i-t}\eta_\ell.
\end{align*}
\end{proof}

\begin{lemma}\label{lemma: main lemma for convergence analysis}
Suppose Assumptions \ref{assumption 2} and \ref{assumption 3} hold. Let $\eta_k=\eta_1 k^{-r}$ denote the learning rate used in the $k$-th iteration with two constants $\eta_1 \in(0, \frac{1}{4L}]$ and $r\in (0,1)$, where $L$ is defined in Lemma~\ref{lemma: Lipschitz continuity}.Let $w_{i,k}=\sqrt{\E\left[f_{i,k}\left(\pmb{s}^{(i,j)},\pmb{a}^{(i,j)}\right)^2 \right]}$. By running Algorithm~\ref{algo: online}, for both LR and CLR policy gradient estimators, we have 
\begin{align*}
\E\left[\Vert\nabla  J(\pmb\theta_k)\Vert^2\right] &\leq \frac{2}{\eta_k}\left(\E\left[ J(\pmb\theta_{k+1})\right]-\E\left[ J(\pmb\theta_k)\right]\right) + 2(I_1 + I_2 + I_3)\nonumber\\
I_1 & \leq\frac{1}{|\mathcal{U}_k|}\sum_{{\pmb\theta_i}\in \mathcal{U}_k}
 \left(C_1(k-i+t)\eta_{i-t}+C_2 t^2\eta_{i-t} +2M^2\varphi(nt)\right)\nonumber\\
I_2 & \leq\frac{LM^2\eta_k}{|\mathcal{U}_k|^2}\sum_{\pmb\theta_i\in \mathcal{U}_k}\sum_{\pmb\theta_{i^\prime}\in\mathcal{U}_k}\max_{\ell=1,2,\ldots,d}\left(\Corr^{(\ell)}_{i,i^\prime,k}\right) w_{i,k} w_{i^\prime,k} \nonumber\\
  I_3 &=2L\eta_k\Vert \Bias_k\Vert^2
\end{align*}
where $C_1=\max\{C^\Gamma_1, C^\Gamma_2\}$, $C_2={2nM^3U_\pi U_f}$, and $\Bias_k=\E\left[\widehat{\nabla} J^{R}_{k}\right]-\E\left[\nabla J(\pmb\theta_k)\right]$ with $R\in\{LR,CLR\}$. Here $C_1^\Gamma$ and $C_2^\Gamma$ are defined in Lemma~\ref{appendix lemma: supportive lemma 6} and $Z^{LR}_2$ is defined in Theorem~\ref{theorem: bounded bias for LR graident}.
\end{lemma}
\begin{proof}
Lemma~\ref{lemma: Lipschitz continuity} implies the L-Lipschitz property of policy gradient, which by the definition of smoothness (\cite{nesterov2003introductory}, Lemma 1.2.3) is also equivalent to
 $$ J(\pmb\theta_k) -  J(\pmb\theta_{k+1})\leq \langle \nabla  J(\pmb\theta_k),\pmb\theta_k - \pmb\theta_{k+1}\rangle + L\Vert \pmb\theta_{k+1}-\pmb\theta_k\Vert^2.$$
 
Let $\widehat{\nabla} J^{R}_{k}$ represents either LR policy gradient estimator $\widehat{\nabla} J^{LR}_{k}$ or CLR policy gradient estimator $\widehat{\nabla} J^{CLR}_{k}$, (i.e. $R\in\{LR,CLR\}$). By applying the policy parameter update implemented in the proposed algorithm $\pmb{\theta}_{k+1} = \pmb{\theta}_k+ \eta_k\widehat{\nabla} J^{R}_{k}$,  we have
\begin{align} 
& J(\pmb \theta_{k}) -  J(\pmb{\theta}_{k+1}) \leq -\left\langle \nabla  J(\pmb\theta_k),\eta_k \widehat{\nabla} J^{R}_k(\pmb\theta_k)\right\rangle + L\Vert \pmb\theta_{k+1}-\pmb\theta_k\Vert^2 \nonumber\\
&= \eta_k \left\langle \nabla  J(\pmb\theta_k),\nabla  J(\pmb\theta_k)-\widehat{\nabla} J^{R}_k(\pmb\theta_k)\right\rangle -\eta_k\Vert\nabla  J(\pmb\theta_k)\Vert^2+ L\Vert \pmb\theta_{k+1}-\pmb\theta_k\Vert^2. \label{eq: theorem equation 0}
\end{align}
Then by taking the expectation of both sides of \eqref{eq: theorem equation 0}, we have
\begin{align}
 &\E\left[ J(\pmb\theta_k)\right]-\E\left[ J(\pmb\theta_{k+1})\right] \nonumber\\ &\leq \eta_k\left|\E\left[\left\langle \nabla J(\pmb\theta_k), \widehat{\nabla} J^{R}_k(\pmb\theta_k)-\nabla J(\pmb\theta_k)\right\rangle \right] \right| -{\eta_k}\E\left[\Vert\nabla  J(\pmb\theta_k)\Vert^2\right]  + L\E[\Vert\pmb\theta_{k+1}-\pmb\theta_k\Vert^2] \nonumber\\
 &\leq \eta_k\left|\E\left[\left\langle \nabla J(\pmb\theta_k), \widehat{\nabla} J^{R}_k(\pmb\theta_k)-\nabla J(\pmb\theta_k)\right\rangle \right] \right| -{\eta_k}\E\left[\Vert\nabla  J(\pmb\theta_k)\Vert^2\right] + L\eta_k^2\E\left[\left\Vert \widehat{\nabla} J^{R}_{k} \right\Vert^2\right].
\end{align}
By applying Lemma~\ref{lemma: variance bound} and rearranging both sides, we have
\begin{align}
(1-2L\eta_k)\E\left[\left\Vert\nabla  J(\pmb\theta_k)\right\Vert^2\right]  &\leq  \frac{1}{\eta_k}\left(\E\left[ J(\pmb\theta_{k+1})\right]-\E\left[ J(\pmb\theta_k)\right]\right) \nonumber \\
&\quad +\underbrace{\left|\E\left[\left\langle \nabla J(\pmb\theta_k), \widehat{\nabla} J^{R}_{k}-\nabla J(\pmb\theta_k)\right\rangle \right] \right|}_{I_1} \nonumber \\
 &\quad+\underbrace{L\eta_k \Tr\left(\Var
 \left[\widehat{\nabla} J^{R}_{k} \right]\right)}_{I_2} \nonumber\\
&\quad +\underbrace{2L\eta_k\left\Vert\E\left[\widehat{\nabla} J^{R}_{k}\right]-\nabla J(\pmb\theta_k)\right\Vert^2}_{I_3}. \label{eq: expected squared norm bound}
\end{align}


Consider $\eta_k$ small enough that $1-2L\eta_k \geq \frac{1}{2}$ or equivalently $\eta_k \leq \frac{1}{4L}$. As the learning rate is non-increasing, this condition can be simplified as the initial learning rate $\eta_1 \leq\frac{1}{4L}$. Then it proceeds with
\begin{equation}
\left(1-2L\eta_k\right)\E\left[\Vert\nabla  J(\pmb\theta_k)\Vert^2\right] \geq\frac{1}{2}\E\left[\Vert\nabla  J(\pmb\theta_k)\Vert^2\right].
 \nonumber 
\end{equation}
Thus, the bound \eqref{eq: expected squared norm bound} becomes
\begin{align}
\frac{1}{2}\E\left[\Vert\nabla  J(\pmb\theta_k)\Vert^2\right] &\leq \frac{1}{\eta_k}\left(\E\left[ J(\pmb\theta_{k+1})\right]-\E\left[ J(\pmb\theta_k)\right]\right) + I_1 + I_2 + I_3
\label{eq.mid7}
\end{align}
which completes the proof.
\end{proof}

\clearpage
\noindent\textbf{Theorem \ref{convergence theorem}}
\textit{
Suppose Assumptions \ref{assumption 2} and \ref{assumption 3} hold. Let $\eta_k=\eta_1 k^{-r}$ denote the learning rate used in the $k$-th iteration with two constants $\eta_1\in(0, \frac{1}{4L}]$ and $r\in (0,1)$. Define the \textbf{L2 importance-weight norm} as \begin{equation} w_{i,k}=\sqrt{\E\left[f_{i,k}\left(\pmb{s}^{(i,j)},\pmb{a}^{(i,j)}\right)^2 \right]} \text{ with } f_{i,k}(\pmb{s},\pmb{a})=\frac{\pi_{\pmb\theta_k}(\pmb{a}|\pmb{s})}{\pi_{\pmb\theta_i}(\pmb{a}|\pmb{s})}. \end{equation} By running Algorithm ~\ref{algo: online} (with the replay buffer of size $B_K$), for both LR/CLR policy gradient estimators in Eq.~\eqref{eq.LR-gradient} and \eqref{eq.CLR-gradient} and any $t\leq K-B_K$, we have the rate of convergence ($R\in\{LR,CLR\}$) 
\begin{align*} 
\frac{1}{K}\sum^K_{k=1}\E\left[\Vert\nabla  J(\pmb\theta_k)\Vert^2\right]  & \leq   \frac{8U_ J /\eta_1}{K^{1-r}} + \frac{4cLM^2}{K}\sum_{k=1}^K{\eta_k} \bar{\rho}_k+     {4M^2}\varphi(nt) + \frac{2^{r+1}C_3}{(1-r)K^r}  \nonumber\\     &\quad  + \frac{2^{r+1}C_2\eta_1 t^2}{(1-r)K^r} + \frac{2^{r+1}C_1\eta_1}{1-r}\frac{B_K+t}{K^r} + M^2\frac{B_K+t}{K}  \end{align*} 
where $C_1=\max\{C^\Gamma_1, C^\Gamma_2\}$, $C_2={2nM^3U_\pi  U_f}$, $C_3=\sup_{k\geq 1}\Vert \Bias_k\Vert$ with $\Bias_k=\E\left[\widehat{\nabla} J^{R}_{k}\right]-\E\left[\nabla J(\pmb\theta_k)\right]$ and $\bar{\rho}_{k}=\frac{1}{|\mathcal{U}_k|^2}\sum_{\pmb\theta_i\in\mathcal{U}_k}\sum_{\pmb\theta_{i^\prime}\in\mathcal{U}_k}\left|\max_{\ell=1,2,\ldots,d}\left(\Corr^{(\ell)}_{i,i^\prime,k}\right)\right|w_{i,k}w_{i^\prime,k}$. Here $C_1^\Gamma$ and $C_2^\Gamma$ are defined in Lemma~\ref{appendix lemma: supportive lemma 6}. Using $\mathcal{O}(\cdot)$ notation gives  \begin{equation*} \small
\frac{1}{K}\sum^K_{k=1}\E\left[\Vert\nabla  J(\pmb\theta_k)\Vert^2\right]      \leq \mathcal{O}\left(\frac{1}{K^{1-r}}\right) + \mathcal{O}\left(\frac{\sum^K_{k=1}\eta_k \bar{\rho}_k}{K}\right) +     \mathcal{O}\left(\varphi(nt)\right) + \mathcal{O}\left(\frac{t^2}{K^r}\right) + \mathcal{O}\left(\frac{B_K+t}{K^r}\right)    \end{equation*} \noindent where $n$ is the number of steps in each iteration. The notation $\mathcal{O}(\cdot)$ hides constants $c$, $L$, $\eta_1$, $M$, $U_ J$, $n$, $U_f$, $U_\pi$, $\kappa_0$, $\kappa$ and $L_g$. The Lipchitz constant $L_g$ of the sample gradient is defined in Lemma~\ref{lemma: lipchitz continuity of sample gradient estimate}.
}
\begin{proof} For the dynamic buffer with size $B_k>0$, it holds that $k-i\leq B_k$ for any $\pmb\theta_i\in\mathcal{U}_k$. Here the buffer size is an increasing function of $k$.
For the first term on the right-hand side of Lemma~\ref{lemma: main lemma for convergence analysis}, we have
\begin{align}
&\sum^K_{k=B_K+t}\frac{1}{\eta_k}\left(\E\left[ J(\pmb\theta_{k+1})\right]-\E\left[ J(\pmb\theta_k)\right]\right) \nonumber\\
&= -\frac{1}{\eta_{B_K+t}}\E\left[ J(\pmb\theta_{B_K+t})\right] +\frac{1}{\eta_{K}}\E\left[ J(\pmb\theta_{K+1})\right]+ \sum_{k=B_K+t+1}^K\left(\frac{1}{\eta_{k-1}}-\frac{1}{\eta_k}\right)\E\left[ J(\pmb\theta_k)\right]
\nonumber \\
 &\leq \frac{1}{\eta_{B_K+t}} U_ J+\frac{1}{\eta_{K}}U_ J + \sum_{k=B_K+t+1}^K\left(\frac{1}{\eta_{k}}-\frac{1}{\eta_{k-1}}\right)U_ J \nonumber\\
 &\leq \frac{2}{\eta_{K}}U_ J. 
\end{align}
Then, by applying Lemma~\ref{lemma: main lemma for convergence analysis} and noticing $| J(\pmb\theta_k)|\leq U_ J$ and summing over $k = B_K+t,B_K+t+1,\ldots, K$, we have
\begin{align}
&\frac{1}{2}\sum_{k=B_K+t}^K\E\left[\Vert\nabla  J(\pmb\theta_k)\Vert^2\right] \leq \frac{{2}}{\eta_{K}}U_ J \nonumber\\
&\quad +\underbrace{\sum^K_{k=B_K+t}\frac{1}{|\mathcal{U}_k|}\sum_{{\pmb\theta_i}\in \mathcal{U}_k}
 \left(C_1(k-i+t)\eta_{i-t}+C_2 t^2\eta_{i-t} +2M^2\varphi(nt)\right)}_{I_1}\nonumber\\
&\quad +\underbrace{LM^2\sum_{k=B_K+t}^K\frac{\eta_k}{|\mathcal{U}_k|^2}\sum_{\pmb\theta_i\in\mathcal{U}_k}\sum_{\pmb\theta_{i^\prime}\in\mathcal{U}_k}\max_{\ell=1,2,\ldots,d}\left(\Corr^{(\ell)}_{i,i^\prime,k}\right)w_{i,k}w_{i^\prime,k}}_{I_2} +\underbrace{2L\sum^K_{k=B_K+t}\eta_k\Vert\Bias_k\Vert}_{I_3}.
\end{align} 

\begin{lemma}\label{I1 in theorem}
Suppose Assumptions \ref{assumption 2} and \ref{assumption 3} hold. Then, 
\begin{align}
    I_1&\leq\frac{C_1\eta_1(B_K+t)}{1-r}(K-B_K-t)^{1-r}+\frac{C_2\eta_1 t^2}{1-r}(K-B_K-t)^{1-r} \nonumber\\
 & \quad +{2M^2}\varphi(nt)(K-B_K-t). \nonumber
\end{align}
\end{lemma}
 \begin{proof}
 Noticing that (1) $k-i\leq B_k\leq B_K$ for any $K\geq k\geq i$, and (2) $\eta_{i-t}=\frac{\eta_1}{(i-t)^r}\leq \frac{\eta_1}{(k-B_K-t)^r}$. Thus, it holds that
\begin{align}
I_1&=\sum^K_{k=B_K+t}\frac{1}{|\mathcal{U}_k|}\sum_{{\pmb\theta_i}\in \mathcal{U}_k}
 \left(\frac{C_1\eta_1(B_K+t)}{(k-B_K-t)^r}+\frac{C_2\eta_1 t^2}{(k-B_K-t)^r} +2M^2\varphi(nt)\right) \nonumber\\
 &\leq \sum^K_{k=B_K+t}
 \left(\frac{C_1\eta_1(B_K+t)}{(k-B_K-t)^r}+\frac{C_2\eta_1 t^2}{(k-B_K-t)^r} +2M^2\varphi(nt)\right) \nonumber\\
 &\leq\frac{C_1\eta_1(B_K+t)}{1-r}(K-B_K-t)^{1-r}+\frac{C_2\eta_1 t^2}{1-r}(K-B_K-t)^{1-r} \nonumber\\
 & \quad +{2M^2}\varphi(nt)(K-B_K-t)\nonumber
\end{align}
where the last step holds due to the fact that $$\sum_{k=B_K+t}^K\frac{1}{(k-B_K-t)^r}\leq 
\int^{K-B_K-t}_0 \frac{\eta_1}{k^r}\dd k =\frac{\eta_1}{1-r}(K-B_K-t)^{1-r}.$$
\end{proof}

\begin{lemma}\label{I2 in theorem}
 Suppose Assumptions \ref{assumption 2} and \ref{assumption 3} hold. Then,
\begin{equation*}
    I_2\leq LM^2\sum_{k=B_K+t}^K{\eta_k} \bar{\rho}_k,
\end{equation*}
where $\bar{\rho}_{k}=\frac{1}{|\mathcal{U}_k|^2}\sum_{\pmb\theta_i\in\mathcal{U}_k}\sum_{\pmb\theta_{i^\prime}\in\mathcal{U}_k}\left|\max_{\ell=1,2,\ldots,d}\left(\Corr^{(\ell)}_{i,i^\prime,k}\right)\right|w_{i,k}w_{i^\prime,k}$.
\end{lemma}

\begin{proof}
Notice that $\max_{\ell=1,2,\ldots,d}\left(\Corr^{(\ell)}_{i,i^\prime,k}\right) \leq 1$. Then we have
\begin{equation}\label{eq0: theorem 4}
    I_2=LM^2\sum_{k=B_K+t}^K\frac{\eta_k}{|\mathcal{U}_k|^2}\sum_{\pmb\theta_i\in\mathcal{U}_k}\sum_{\pmb\theta_{i^\prime}\in\mathcal{U}_k}\max_{\ell=1,2,\ldots,d}\left(\Corr^{(\ell)}_{i,i^\prime,k}\right)w_{i,k}w_{i^\prime,k}\leq LM^2\sum_{k=B_K+t}^K{\eta_k} \bar{\rho}_k.
\end{equation}
\end{proof}

\begin{lemma}\label{I3 in theorem}
   Suppose Assumptions \ref{assumption 2} and \ref{assumption 3} hold. Then,  
   \begin{equation*}
           I_3 \leq {C_3}(K-B_K-t)^{1-r},
   \end{equation*}
   where $C_3=\sup_{k\geq 1}\Vert \Bias_k\Vert$.
\end{lemma}
\begin{proof}
Notice the bias term $\Vert \Bias_k\Vert$ is upper bounded by a convergent sequence based on Theorem~\ref{theorem: bounded bias for LR graident} and Eq.~\eqref{eq: big O of bias with buffer size}. For $R\in\{LR, CLR\}$
\begin{align}
\Vert\Bias_k \Vert &\leq\frac{Z_1^{R} (t+1)}{|\mathcal{U}_k|}
 \sum_{{\pmb\theta_i}\in \mathcal{U}_k}\sum^i_{\ell = i-t}\eta_\ell + 2M\varphi(nt) + \frac{Z_2^{R}}{|\mathcal{U}_k|}
 \sum_{{\pmb\theta_i}\in \mathcal{U}_k} \sum^k_{\ell=i-t}\eta_\ell \nonumber\\
 & \leq Z_1^{R}\eta_1\frac{(t+1)t}{(k-B_k-t)^r}+Z_2^{R}\eta_1\frac{B_k+t}{(k-B_k-t)^r}+ 2M\varphi(nt) \xrightarrow{k\rightarrow\infty} 0\label{eq: bias bound theorem 4}
\end{align}
where the step \eqref{eq: bias bound theorem 4} holds because $\sum^i_{\ell=i-t}\eta_\ell \leq \eta_1 t (i-t)^{-r}\leq \eta_1 t (k-B_k-t)^{-r}$ and $\sum^k_{\ell=i-t}\eta_\ell \leq \eta_1 (k-i+t) (i-t)^{-r}\leq \eta_1 (B_k+t)(k-B_k-t)^{-r}$. The convergence of a sequence implies that it is bounded. Thus,  we can bound the norm of bias as $\Vert \Bias_k\Vert\leq \sup_{k\geq 1} \Vert\Bias_k\Vert$.
By applying $\sum^K_{k=B_K+t}\eta_k = \sum^K_{k=B_K+t}\frac{\eta_1}{k^r} \leq \int^{K-B_K-t}_{0}\frac{\eta_1}{k^r}\dd k \leq\frac{\eta_1}{1-r}(K-B_K-t)^{1-r}$ and the bias bound, we can show
\begin{align*}
    I_3 \leq {C_3}(K-B_K-t)^{1-r}.
\end{align*}
\end{proof}
Putting Lemmas~\ref{I1 in theorem},~\ref{I2 in theorem},~\ref{I3 in theorem} together, dividing both sides by $\frac{1}{2}(K-B_K-t+1)$ gives the result
\begin{align}
  & \frac{\sum_{k=B_K+t}^K\E\left[\Vert\nabla  J(\pmb\theta_k)\Vert^2\right]}{K-B_K-t+1} 
   \leq \frac{4\eta_1^{-1} U_ J K^{r}}{K-B_K-t+1} + \frac{2cLM^2}{K-B_k-t+1}\sum_{k=B_K+t}^K{\eta_k} \bar{\rho}_k \nonumber\\
&\quad + \frac{2C_3\eta_1}{1-r}(K-B_K-t)^{-r}+ {4M^2}\varphi(nt) +\frac{2C_2\eta_1 t^2}{1-r}(K-B_K-t)^{-r} \nonumber\\
& \quad + \frac{2C_1\eta_1}{1-r}(B_K+t)(K-B_K-t)^{-r}. \label{eq1: theorem 4}
\end{align}
For large enough $K$ such that $K \geq 2t+2B_K$, we have 
\begin{align}
    \frac{4 U_ J K^r }{ K-B_K-t+1 }& \leq \frac{4U_ J K^r }{ K-B_K-t } \leq \frac{8U_ J}{K^{1-r}} \nonumber\\
    (K-B_K-t)^{-r}&\leq \frac{2^r}{K^r}. \nonumber
\end{align}
Therefore, Eq.~\eqref{eq1: theorem 4} becomes
\begin{align}
   &\frac{\sum_{k=B_K+t}^K\E\left[\Vert\nabla  J(\pmb\theta_k)\Vert^2\right]}{K-B_K-t+1} \leq \frac{8U_J/\eta_1}{K^{1-r}} + \frac{2cLM^2}{K-B_k-t+1}\sum_{k=B_K+t}^K{\eta_k} \bar{\rho}_k\nonumber\\
&\quad + \frac{2^{r+1}C_3\eta_1}{1-r}K^{-r}+ {4 M^2}\varphi(nt)+\frac{2^{r+1}C_2\eta_1 t^2}{1-r}K^{-r}    + \frac{2^{r+1}C_1\eta_1}{1-r}(B_K+t)K^{-r}. \nonumber
\end{align}
Notice that the learning rate is positive $\eta_k> 0$. Since $\bar{\rho}_k>0$, when $K\geq 2t+2B_K$, we have $$\frac{2cLM^2}{K-B_k-t+1}\sum_{k=B_K+t}^K{\eta_k} \bar{\rho}_k\leq \frac{4cLM^2}{K}\sum_{k=1}^K{\eta_k} \bar{\rho}_k.$$ 
By applying Lemma~\ref{supportive lemma: mean sequence Big O} with $f(K)=B_K+t$ and $U_a=M^2$,  we have the first conclusion
\begin{align}
   \frac{1}{K}\sum_{k=1}^K\E\left[\Vert\nabla  J(\pmb\theta_k)\Vert^2\right] &\leq \frac{\sum_{k=B_K+t}^K\E\left[\Vert\nabla  J(\pmb\theta_k)\Vert^2\right]}{K-B_K-t+1} + M^2 \frac{B_K+t}{K} \nonumber\\
   &\leq \frac{8U_ J /\eta_1}{K^{1-r}} + \frac{4cLM^2}{K}\sum_{k=1}^K{\eta_k} \bar{\rho}_k+
    {4M^2}\varphi(nt) + \frac{2^{r+1}C_3\eta_1}{(1-r)K^r}  \nonumber\\
    &\quad  + \frac{2^{r+1}C_2\eta_1 t^2}{(1-r)K^r} + \frac{2^{r+1}C_1\eta_1}{1-r}\frac{B_K+t}{K^r} + M^2\frac{B_K+t}{K}. \label{eq2: theorem 4}
\end{align}
Then writing Eq.~\eqref{eq2: theorem 4} using $\mathcal{O}(\cdot)$ notation gives the second conclusion
\begin{equation*}
\frac{1}{K}\sum_{k=1}^K\E\left[\Vert\nabla  J(\pmb\theta_k)\Vert^2\right] \leq \mathcal{O}\left(\frac{1}{K^{1-r}}\right) +\mathcal{O}\left(\frac{\sum^K_{k=1}\eta_k \bar{\rho}_k}{K}\right)+ \mathcal{O}\left(\varphi(nt)\right)+ \mathcal{O}\left(\frac{t^2}{K^r}\right) + \mathcal{O}\left(\frac{B_K+t}{K^r}\right) \\
\end{equation*}
which completes the proof.
\end{proof}

\subsection{Proof of Corollary~\ref{cor: convergence rate}}
\label{appendix subsec: corollary convergence}

\noindent\textbf{Corollary \ref{cor: convergence rate}}
\textit{
Suppose Assumptions \ref{assumption 2} and \ref{assumption 3} hold. Under the same configurations as Theorem~\ref{convergence theorem}, by setting $t=\log_{\kappa}K^{-r/n}$, we have the rate of convergence
 \begin{equation*}
    \frac{1}{K}\sum^K_{k=1}\E\left[\Vert\nabla  J(\pmb\theta_k)\Vert^2\right] 
    \leq \mathcal{O}\left(\frac{1}{K^{1-r}}\right) + \mathcal{O}\left(\frac{1}{K^r}\right) + \mathcal{O}\left(\frac{t^2}{K^r}\right) + \mathcal{O}\left(\frac{B_K+t}{K^r}\right)  
 \end{equation*}
\noindent where the notation
$\mathcal{O}(\cdot)$ hides constants $c$, $L$, $\eta_1$, $M$, $U_ J$, $n$, $U_f$, $U_\pi$, $\kappa_0$, $\kappa$ and $L_g$.
}
\begin{proof} For the dynamic buffer with size $B_K>0$, it holds that $K-i\leq B_K$ for any $\pmb\theta_i\in\mathcal{U}_K$. Here the buffer size is an increasing function of $K$. By setting $t=\log_{\kappa}K^{-r/n}$, it holds that 
$$\varphi(nt)=\kappa_0 \kappa^{n\log_{\kappa} \left(\frac{1}{K^r}\right)^{1/n}}=\kappa_0 \kappa^{\log_{\kappa}\left(\frac{1}{K^r}\right)}=\kappa_0K^{-r}.$$
The conclusion is obtained by applying Theorem~\ref{convergence theorem}.
\end{proof}

\section{Proof of Technical Lemmas}\label{appendix sec: proof of technical lemmas}

\begin{lemma}\label{appendix lemma: supportive lemma 1}
   Under Assumption~\ref{assumption 2}-\ref{assumption 3}, for any $t$ such that $t <i\leq k$ we have
   \begin{itemize}
     \item[\textbf{(i)}]  $\left\Vert\E\left[\widehat{\nabla} J^{LR}\left({\pmb{x}}^{(i,j)}, \pmb\theta_i,\pmb\theta_{k}\right)-\widehat{\nabla} J^{LR}\left({\pmb{x}}^{(i,j)}, \pmb\theta_i,\pmb\theta_{i-t}\right) \right] \right\Vert \leq 
     M(L_g + MU_\pi)  \sum^k_{\ell=i-t}\eta_\ell $
    \item[\textbf{(ii)}]    $\E\left[\left\Vert\widehat{\nabla} J^{LR}\left({\pmb{x}}^{(i,j)}, \pmb\theta_i,\pmb\theta_{k}\right)-\widehat{\nabla} J^{LR}\left({\pmb{x}}^{(i,j)}, \pmb\theta_i,\pmb\theta_{i-t}\right) \right\Vert\right] \leq 
    M(L_g + MU_\pi)  \sum^k_{\ell=i-t}\eta_\ell  $
\item[\textbf{(iii)}] $\E\left[\left\Vert\widehat{\nabla} J^{CLR}\left({\pmb{x}}^{(i,j)}, \pmb\theta_i,\pmb\theta_{k}\right)-\widehat{\nabla} J^{CLR}\left({\pmb{x}}^{(i,j)}, \pmb\theta_i,\pmb\theta_{i-t}\right)\right\Vert\right] \leq M(U_f L_g+M U_\pi) \sum^k_{\ell=i-t}\eta_\ell$
   \end{itemize}
\end{lemma}
\begin{proof}
\noindent\textbf{(i)} and \textbf{(ii)}: Let 
\begin{align*}
    \Delta_{\textbf{(i)}}&\defeq\left\Vert\E\left[\widehat{\nabla} J^{LR}\left({\pmb{x}}^{(i,j)}, \pmb\theta_i,\pmb\theta_{k}\right)-\widehat{\nabla} J^{LR}\left({\pmb{x}}^{(i,j)}, \pmb\theta_i,\pmb\theta_{i-t}\right)  \right] \right\Vert\\
\Delta_{\textbf{(ii)}}&\defeq\E\left[\left\Vert\widehat{\nabla} J^{LR}\left({\pmb{x}}^{(i,j)}, \pmb\theta_i,\pmb\theta_{k}\right)-\widehat{\nabla} J^{LR}\left({\pmb{x}}^{(i,j)}, \pmb\theta_i,\pmb\theta_{i-t}\right)\right\Vert\right]
\end{align*}
It follows that
\begin{align}
\Delta_{\textbf{(ii)}}&= \E\left[ \left\Vert\frac{\pi_{\pmb\theta_k}\left(\pmb{a}^{(i,j)}|\pmb{s}^{(i,j)}\right)}{\pi_{\pmb\theta_i}\left(\pmb{a}^{(i,j)}|\pmb{s}^{(i,j)}\right)} g\left(\pmb{x}^{(i,j)}|\pmb\theta_k\right)-\frac{\pi_{\pmb\theta_{i-t}}\left(\pmb{a}^{(i,j)}|\pmb{s}^{(i,j)}\right)}{\pi_{\pmb\theta_i}\left(\pmb{a}^{(i,j)}|\pmb{s}^{(i,j)}\right)} g\left(\pmb{x}^{(i,j)}|\pmb\theta_{i-t}\right)\right\Vert \right] \nonumber\\
&= 
\E\left[\left\Vert\frac{\pi_{\pmb\theta_{i-t}}\left(\pmb{a}^{(i,j)}|\pmb{s}^{(i,j)}\right)}{\pi_{\pmb\theta_i}\left(\pmb{a}^{(i,j)}|\pmb{s}^{(i,j)}\right)} \left[g\left(\pmb{x}^{(i,j)}|\pmb\theta_{i-t}\right)-g\left(\pmb{x}^{(i,j)}|\pmb\theta_{k}\right)\right] \right.\right.\nonumber\\
&\qquad \qquad +\left.\left.\frac{\pi_{\pmb\theta_{i-t}}\left(\pmb{a}^{(i,j)}|\pmb{s}^{(i,j)}\right)-\pi_{\pmb\theta_{k}}\left(\pmb{a}^{(i,j)}|\pmb{s}^{(i,j)}\right)}{\pi_{\pmb\theta_i}\left(\pmb{a}^{(i,j)}|\pmb{s}^{(i,j)}\right)} g\left(\pmb{x}^{(i,j)}|\pmb\theta_{k}\right) \right\Vert \right] \nonumber\\
&= \E\left[\frac{\pi_{\pmb\theta_{i-t}}\left(\pmb{a}^{(i,j)}|\pmb{s}^{(i,j)}\right)}{\pi_{\pmb\theta_i}\left(\pmb{a}^{(i,j)}|\pmb{s}^{(i,j)}\right)} \left\Vert g\left(\pmb{x}^{(i,j)}|\pmb\theta_k\right)-g\left(\pmb{x}^{(i,j)}|\pmb\theta_{i-t}\right)\right\Vert \right] \nonumber\\
&\quad +\E\left[\frac{\left|\pi_{\pmb\theta_{k}}\left(\pmb{a}^{(i,j)}|\pmb{s}^{(i,j)}\right)-\pi_{\pmb\theta_{i-t}}\left(\pmb{a}^{(i,j)}|\pmb{s}^{(i,j)}\right)\right|}{\pi_{\pmb\theta_i}\left(\pmb{a}^{(i,j)}|\pmb{s}^{(i,j)}\right)} \left\Vert g\left(\pmb{x}^{(i,j)}|\pmb\theta_{k}\right) \right\Vert\right] \nonumber\\
&\leq  \E\left[\E\left[\left. \left\Vert g\left(\pmb{x}|\pmb\theta_k\right)-g\left(\pmb{x}|\pmb\theta_{i-t}\right)\right\Vert \right| \pmb\theta_{i-t} \right]\right] \nonumber\\
& \quad +\E\left[\frac{\left|\pi_{\pmb\theta_{k}}\left(\pmb{a}^{(i,j)}|\pmb{s}^{(i,j)}\right)-\pi_{\pmb\theta_{i-t}}\left(\pmb{a}^{(i,j)}|\pmb{s}^{(i,j)}\right)\right|}{\pi_{\pmb\theta_i}\left(\pmb{a}^{(i,j)}|\pmb{s}^{(i,j)}\right)} \left\Vert g\left(\pmb{x}^{(i,j)}|\pmb\theta_{k}\right) \right\Vert\right] \label{eq0: appendix lemma: supportive lemma 1}.
\end{align}
By applying Lemma~\ref{lemma: lipchitz continuity of sample gradient estimate}, we have $\Vert g\left(\pmb{x}^{(i,j)}|\pmb\theta_k\right)-g\left(\pmb{x}^{(i,j)}|\pmb\theta_{i-t}\right)\Vert\leq L_g \Vert\pmb\theta_k-\pmb\theta_{i-t}\Vert$. By Lemma~\ref{lemma: bounded of policy gradient}, we have $\Vert g\left(\pmb{x}^{(i,j)}|\pmb\theta_{k}\right)\Vert \leq M$. Then the conclusion follows that
\begin{align}
\Delta_{\textbf{(i)}} \leq \Delta_{\textbf{(ii)}} &\leq 
\E\left[\E\left[\left. \left\Vert g\left(\pmb{x}|\pmb\theta_k\right)-g\left(\pmb{x}|\pmb\theta_{i-t}\right)\right\Vert \right| \pmb\theta_{i-t} \right]\right] \nonumber\\
&\quad +\E\left[\frac{|\pi_{\pmb\theta_{k}}\left(\pmb{a}^{(i,j)}|\pmb{s}^{(i,j)}\right)-\pi_{\pmb\theta_{i-t}}\left(\pmb{a}^{(i,j)}|\pmb{s}^{(i,j)}\right)|}{\pi_{\pmb\theta_i}\left(\pmb{a}^{(i,j)}|\pmb{s}^{(i,j)}\right)} \left\Vert g\left(\pmb{x}^{(i,j)}|\pmb\theta_{i-t}\right)\right\Vert  \right] \label{eq-0: appendix lemma: supportive lemma 1} \\
&\leq  L_g \E\left[\E\left[\Vert\pmb\theta_k-\pmb\theta_{i-t}\Vert| \pmb\theta_{i-t}\right] \right] + M  U_\pi \E[\Vert\pmb\theta_k-\pmb\theta_{i-t}\Vert]\label{eq-1: appendix lemma: supportive lemma 1}\\
&\leq  (L_g + M  U_\pi )\E[\Vert\pmb\theta_k-\pmb\theta_{i-t}\Vert]
\end{align}
where Step~\eqref{eq-0: appendix lemma: supportive lemma 1} follows \eqref{eq0: appendix lemma: supportive lemma 1}. The first term of Step~\eqref{eq-1: appendix lemma: supportive lemma 1} holds due to Lemma~\ref{lemma: lipchitz continuity of sample gradient estimate}. The second term of Step~\eqref{eq-1: appendix lemma: supportive lemma 1} holds because
\begin{align*}
&\E\left[\frac{|\pi_{\pmb\theta_{k}}\left(\pmb{a}^{(i,j)}|\pmb{s}^{(i,j)}\right)-\pi_{\pmb\theta_{i-t}}\left(\pmb{a}^{(i,j)}|\pmb{s}^{(i,j)}\right)|}{\pi_{\pmb\theta_i}\left(\pmb{a}^{(i,j)}|\pmb{s}^{(i,j)}\right)} \left\Vert g\left(\pmb{x}^{(i,j)}|\pmb\theta_{i-t}\right)\right\Vert \right]\\
&=M\E\left[\int_{\mathcal{A}}\left|\pi_{\pmb\theta_{k}}\left(\pmb{a}^{(i,j)}|\pmb{s}^{(i,j)}\right)-\pi_{\pmb\theta_{i-t}}\left(\pmb{a}^{(i,j)}|\pmb{s}^{(i,j)}\right)\right|\dd \pmb{a}^{(i,j)}\right]\\
&\leq M U_\pi \E\left[\Vert\pmb\theta_k - \pmb\theta_{i-t}\Vert\right]. \nonumber
\end{align*}
The last step holds by applying (\ref{eq: assumption item 3}) in Assumption~\ref{assumption 2}. Applying 
Lemma~\ref{auxillary lemma: parametric distance} to \eqref{eq-1: appendix lemma: supportive lemma 1} gives
\begin{equation*}
\Delta_{\textbf{(i)}} \leq \Delta_{\textbf{(ii)}} \leq  M(L_g + MU_\pi)  \sum^k_{\ell=i-t}\eta_\ell 
\end{equation*}

\noindent\textbf{(iii)} The proof of \textbf{(iii)} is similar to that of \textbf{(i)} and \textbf{(ii)}.

Bound the likelihood ratio by $\min\left\{\frac{\pi_{\pmb\theta_k}\left(\pmb{a}^{(i,j)}|\pmb{s}^{(i,j)}\right)}{\pi_{\pmb\theta_i}\left(\pmb{a}^{(i,j)}|\pmb{s}^{(i,j)}\right)}, U_f\right\}\leq U_f$, and notice that  for any $(\pmb{s},\pmb{a})\in\mathcal{S}\times\mathcal{A}$, it holds that

\begin{equation*}
    \min\left\{\frac{\pi_{\pmb\theta_k}\left(\pmb{a}|\pmb{s}\right)}{\pi_{\pmb\theta_i}\left(\pmb{a}|\pmb{s}\right)}, U_f\right\} - \min\left\{\frac{\pi_{\pmb\theta_{i-t}}\left(\pmb{a}|\pmb{s}\right)}{\pi_{\pmb\theta_i}\left(\pmb{a}|\pmb{s}\right)}, U_f\right\} \leq \frac{|\pi_{\pmb\theta_{k}}\left(\pmb{a}|\pmb{s}\right)-\pi_{\pmb\theta_{i-t}}\left(\pmb{a}|\pmb{s}\right)|}{\pi_{\pmb\theta_i}\left(\pmb{a}|\pmb{s}\right)}.
\end{equation*}
Then by applying Lemma~\ref{lemma: lipchitz continuity of sample gradient estimate} and Assumption~\ref{assumption 2}, the conclusion follows that
\begin{align}
& \E\left[\left\Vert\widehat{\nabla} J^{CLR}\left({\pmb{x}}^{(i,j)}, \pmb\theta_i,\pmb\theta_{k}\right)-\widehat{\nabla} J^{CLR}\left({\pmb{x}}^{(i,j)}, \pmb\theta_i,\pmb\theta_{i-t}\right)\right\Vert\right] \nonumber\\
&\leq \E\left[\min\left\{\frac{\pi_{\pmb\theta_k}\left(\pmb{a}^{(i,j)}|\pmb{s}^{(i,j)}\right)}{\pi_{\pmb\theta_i}\left(\pmb{a}^{(i,j)}|\pmb{s}^{(i,j)}\right)}, U_f\right\} \left\Vert g\left(\pmb{x}^{(i,j)}|\pmb\theta_k\right)-g\left(\pmb{x}^{(i,j)}|\pmb\theta_{i-t}\right)\right\Vert\right] \nonumber\\
&\quad +\E\left[\frac{|\pi_{\pmb\theta_{k}}\left(\pmb{a}^{(i,j)}|\pmb{s}^{(i,j)}\right)-\pi_{\pmb\theta_{i-t}}\left(\pmb{a}^{(i,j)}|\pmb{s}^{(i,j)}\right)|}{\pi_{\pmb\theta_i}\left(\pmb{a}^{(i,j)}|\pmb{s}^{(i,j)}\right)} \left\Vert g\left(\pmb{x}^{(i,j)}|\pmb\theta_{i-t}\right)\right\Vert  \right] \nonumber\\
&\leq U_f L_g \E\left[\Vert\pmb\theta_k-\pmb\theta_{i-t}\Vert\right] + M  U_\pi \E[\Vert\pmb\theta_k-\pmb\theta_{i-t}\Vert]\nonumber\\
&\leq (U_f L_g+M U_\pi) M \sum^k_{\ell=i-t}\eta_\ell.\nonumber
\end{align}
\end{proof}

\begin{lemma}\label{appendix lemma: supportive lemma 2}
   Under Assumptions~\ref{assumption 2}-\ref{assumption 3}, for any $t$ such that $t <i\leq k$, we have
   \begin{itemize}
      \item[\textbf{(i)}] $\left\Vert\E\left[\widehat{\nabla} J^{LR}\left({\pmb{x}}^{(i,j)}, \pmb\theta_i,\pmb\theta_{i-t}\right)-\widehat{\nabla} J\left(\tilde{\pmb{x}}^{(i,j)}, \pmb\theta_i,\pmb\theta_{i-t}\right)\right] \right\Vert \leq  2 nM^2U_\pi t  \sum^{i}_{\ell=i-t}\eta_{\ell}$ \\
    \item[\textbf{(ii)}] $\E\left[\left\Vert\widehat{\nabla} J^{LR}\left({\pmb{x}}^{(i,j)}, \pmb\theta_i,\pmb\theta_{i-t}\right)-\widehat{\nabla} J\left(\tilde{\pmb{x}}^{(i,j)}, \pmb\theta_i,\pmb\theta_{i-t}\right)\right\Vert\right]  \leq 2M + 2M^2U_\pi\sum^i_{\ell=i-t}\eta_\ell$ \\
 \item[\textbf{(iii)}] $\left\Vert\E\left[\widehat{\nabla} J^{CLR}\left({\pmb{x}}^{(i,j)}, \pmb\theta_i,\pmb\theta_{i-t}\right)-\widehat{\nabla} J\left(\tilde{\pmb{x}}^{(i,j)}, \pmb\theta_i,\pmb\theta_{i-t}\right)\right]\right\Vert\leq 
      2nM^2U_\pi U_f t\sum^i_{\ell=i-t}\eta_\ell$
    \item[\textbf{(iv)}] $\E\left[\left\Vert\widehat{\nabla} J^{CLR}\left({\pmb{x}}^{(i,j)}, \pmb\theta_i,\pmb\theta_{i-t}\right)-\widehat{\nabla} J\left(\tilde{\pmb{x}}^{(i,j)}, \pmb\theta_i,\pmb\theta_{i-t}\right)\right\Vert\right]\leq 
       (3U_f+1)M$
   \end{itemize}
\end{lemma}
\begin{proof}

\noindent \textbf{(i)}: Let $\Delta_{\textbf{(i)}} = \left\Vert \E\left[\left.\widehat{\nabla} J^{LR}\left({\pmb{x}}^{(i,j)}, \pmb\theta_i,\pmb\theta_{i-t}\right)-\widehat{\nabla} J\left(\tilde{\pmb{x}}^{(i,j)}, \pmb\theta_i,\pmb\theta_{i-t}\right)\right|\pmb{s}^{(i-t,j)}, \pmb\theta_{i-t}\right]\right\Vert$.

Conditioning on $\pmb\theta_{i-t}$ and $\pmb{s}^{(i-t,j)}$, we have
\begin{align}
    \Delta_{\textbf{(i)}}&=\left\Vert \E\left[\left.\frac{\pi_{\pmb\theta_{i-t}}\left(\pmb{a}^{(i,j)}|\pmb{s}^{(i,j)}\right)}{\pi_{\pmb\theta_i}\left(\pmb{a}^{(i,j)}|\pmb{s}^{(i,j)}\right)} g\left(\pmb{x}^{(i,j)}|\pmb\theta_{i-t}\right)- g\left(\tilde{\pmb{x}}^{(i,j)}|\pmb\theta_{i-t}\right)\right| \pmb{s}^{(i-t,j)}, \pmb\theta_{i-t}\right]\right\Vert\nonumber\\
&=\left\Vert \int \P\left(\pmb{s}^{(i,j)}=\dd \pmb{s}|\pmb{s}^{(i-t,j)},\pmb\theta_{i-t}\right)\pi_{\pmb\theta_{i-t}}\left(\pmb{a}|\pmb{s}\right)g\left(\pmb{s}, \pmb{a}|\pmb\theta_{i-t}\right)\dd \pmb{a} \right. \nonumber\\ 
& \qquad \left. - \int \P\left(\tilde{\pmb{s}}^{(i,j)}=\dd \pmb{s}|\pmb{s}^{(i-t,j)},\pmb\theta_{i-t}\right)\pi_{\pmb\theta_{i-t}}\left(\pmb{a}|\pmb{s}\right)g\left(\pmb{s}, \pmb{a}|\pmb\theta_{i-t}\right)\dd \pmb{a} \right\Vert \nonumber\\ 
&=\left\Vert \int \left(\P(\pmb{s}^{(i,j)}=\dd \pmb{s}|\pmb{s}^{(i-t,j)},\pmb\theta_{i-t})-\P(\tilde{\pmb{s}}^{(i,j)}=\dd \pmb{s}|\pmb{s}^{(i-t,j)},\pmb\theta_{i-t}) \right)\pi_{\pmb\theta_{i-t}}\left(\pmb{a}|\pmb{s}\right)g\left(\pmb{s}, \pmb{a}|\pmb\theta_{i-t}\right)\dd \pmb{a}\right\Vert \nonumber\\
&\leq \int \left|\P(\pmb{s}^{(i,j)}=\dd \pmb{s}|\pmb{s}^{(i-t,j)},\pmb\theta_{i-t})-\P(\tilde{\pmb{s}}^{(i,j)}=\dd \pmb{s}|\pmb{s}^{(i-t,j)},\pmb\theta_{i-t}) \right|\pi_{\pmb\theta_{i-t}}\left(\pmb{a}|\pmb{s}\right)\left\Vert g\left(\pmb{s}, \pmb{a}|\pmb\theta_{i-t}\right)\right\Vert\dd \pmb{a} \nonumber\\
&\leq M \int\left|\P(\pmb{s}^{(i,j)}=\dd \pmb{s}|\pmb{s}^{(i-t,j)},\pmb\theta_{i-t})-\P(\tilde{\pmb{s}}^{(i,j)}=\dd \pmb{s}|\pmb{s}^{(i-t,j)},\pmb\theta_{i-t}) \right| \int\pi_{\pmb\theta_{i-t}}\left(\pmb{a}|\pmb{s}\right)\dd \pmb{a}\label{eq1: appendix lemma: supportive lemma 2}\\
&\leq 2M\left\Vert \P\left(\pmb{s}^{(i,j)}\in \cdot|\pmb{s}^{(i-t,j)},\pmb\theta_{i-t}\right)-\P\left(\tilde{\pmb{s}}^{(i,j)}\in \cdot|\pmb{s}^{(i-t,j)},\pmb\theta_{i-t}\right)\right\Vert_{TV}\label{eq: appendix lemma: supportive lemma 2}
\end{align}
where Step~\eqref{eq1: appendix lemma: supportive lemma 2} holds due to Lemma~\ref{lemma: bounded of policy gradient} and the last inequality is by the definition of total variation. Let $$\Delta^{(i,j)}=2M\left\Vert \P\left(\pmb{s}^{(i,j)}\in \cdot|\pmb{s}^{(i-t,j)},\pmb\theta_{i-t}\right)-\P\left(\tilde{\pmb{s}}^{(i,j)}\in \cdot|\pmb{s}^{(i-t,j)},\pmb\theta_{i-t}\right)\right\Vert_{TV}.$$ By Lemma~\ref{lemma: relations of total variation measures}, we have 
\begin{align}
\Delta^{(i,j)}&\leq 2M\left\Vert \P\left(
\left(\pmb{s}^{(i,j-1)},\pmb{a}^{(i,j-1)}  \right)\in \cdot|\pmb{s}^{(i-t,j)},\pmb\theta_{i-t}\right)-\P\left(\left(\tilde{\pmb{s}}^{(i,j-1)},\tilde{\pmb{a}}^{(i,j-1)} \right)\in \cdot|\pmb{s}^{(i-t,j)},\pmb\theta_{i-t}\right)\right\Vert_{TV} \nonumber\\
&\leq 2M\left\Vert \P\left(\left.\pmb{s}^{(i,j-1)}\in \cdot\right|\pmb{s}^{(i-t,j)},\pmb\theta_{i-t}\right)-\P\left(\left.\tilde{\pmb{s}}^{(i,j-1)}\in \cdot\right|\pmb{s}^{(i-t,j)},\pmb\theta_{i-t}\right)\right\Vert_{TV} \nonumber\\
&\quad +2M U_\pi \E\left[\Vert \pmb\theta_{i}-\pmb\theta_{i-t}\Vert|\pmb\theta_{i-t}\right] \nonumber\\
&={\Delta^{(i,j-1)}} + 2M U_\pi \E\left[\Vert \pmb\theta_{i}-\pmb\theta_{i-t}\Vert|\pmb\theta_{i-t}\right]\label{eq2: appendix lemma: supportive lemma 2}\\
& \cdots\nonumber\\
&={\Delta^{(i-1,n)}} + 2M U_\pi \sum^j_{j^\prime=1}\E\left[\Vert \pmb\theta_{i}-\pmb\theta_{i-t}\Vert|\pmb\theta_{i-t}\right]. \nonumber
\end{align}
In sum, repeating applying the inequality~\eqref{eq2: appendix lemma: supportive lemma 2} from $(i,j)$ to $(i-t,1)$ gives
\begin{align}
    \Delta^{(i,j)}\leq \ldots &\leq \Delta^{(i-t,1)} + 2M U_\pi \sum^j_{j^\prime=1}\E\left[\Vert \pmb\theta_{i}-\pmb\theta_{i-t}\Vert|\pmb\theta_{i-t}\right]+2M U_\pi\sum^{i-1}_{i^\prime=i-t}\sum^{n}_{j=1}\E[\Vert \pmb\theta_{i^\prime}-\pmb\theta_{i-t}\Vert|\pmb\theta_{i-t}] \nonumber\\
    &\leq \Delta^{(i-t,1)} + 2n M U_\pi \sum^{i}_{i^\prime=i-t}\E\left[\Vert \pmb\theta_{i^\prime}-\pmb\theta_{i-t}\Vert|\pmb\theta_{i-t}\right] \nonumber
\end{align}
where the last step holds as $\sum^j_{j^\prime=1}\E\left[\Vert \pmb\theta_{i}-\pmb\theta_{i-t}\Vert|\pmb\theta_{i-t}\right] \leq n \E\left[\Vert \pmb\theta_{i}-\pmb\theta_{i-t}\Vert|\pmb\theta_{i-t}\right]$. By noting $$\Delta^{(i-t,j)}=2M\left\Vert \P\left(\pmb{s}^{(i-t,j)}\in \cdot|\pmb{s}^{(i-t,j)}\right)-\P\left(\tilde{\pmb{s}}^{(i-t,j)}\in \cdot|\pmb{s}^{(i-t,j)}\right)\right\Vert_{TV}=0 \text{  for any $j\in [n]$}.$$ 
Thus we have  $\Delta^{(i,j)}\leq 2n M U_\pi \sum^{i}_{i^\prime=i-t}\E[\Vert \pmb\theta_{i^\prime}-\pmb\theta_{i-t}\Vert|\pmb\theta_{i-t}]$. 

By Lemma~\ref{auxillary lemma: parametric distance}, we have $\E[\Vert \pmb\theta_{i^\prime}-\pmb\theta_{i-t}\Vert|\pmb\theta_{i-t}] \leq M\sum^{i^\prime}_{\ell=i-t}{\eta_{\ell}}$ and thus it holds that
\begin{equation}\label{eq4: appendix lemma: supportive lemma 2}
\Delta_{\textbf{(i)}}\leq\Delta^{(i,j)} \leq 2nU_\pi M^2\sum^{i}_{i^\prime=i-t}  \sum^{i^\prime}_{\ell=i-t}\eta_{\ell} \leq 2nU_\pi M^2t  \sum^{i}_{\ell=i-t}\eta_{\ell}.
\end{equation}
Then the conclusion follows from the fact that $\E[\Delta_{\textbf{(i)}}]\leq 2nU_\pi M^2t  \sum^{i}_{\ell=i-t}\eta_{\ell}$.

\vspace{0.1in}
\noindent \textbf{(ii)}: Different from \textbf{(i)}, we define the conditional expectation $$\Delta_{\textbf{(ii)}} = \E\left[\left.\left\Vert \widehat{\nabla} J^{LR}\left({\pmb{x}}^{(i,j)}, \pmb\theta_i,\pmb\theta_{i-t}\right)-\widehat{\nabla} J\left(\tilde{\pmb{x}}^{(i,j)}, \pmb\theta_i,\pmb\theta_{i-t}\right)\right\Vert\right|\pmb{s}^{(i-t,j)},\pmb\theta_{i-t}\right].$$
It follows that
\begin{align}
\Delta_{\textbf{(ii)}} &=\E\left[\left.\left\Vert \frac{\pi_{\pmb\theta_{i-t}}\left(\pmb{a}^{(i,j)}|\pmb{s}^{(i,j)}\right)}
{\pi_{\pmb\theta_i}\left(\pmb{a}^{(i,j)}|\pmb{s}^{(i,j)}\right)}g\left(\pmb{s}^{(i,j)},\pmb{a}^{(i,j)}|\pmb\theta_{i-t}\right) - g\left(\tilde{\pmb{s}}^{(i,j)},\tilde{\pmb{a}}^{(i,j)}|\pmb\theta_{i-t}\right)\right\Vert\right|\pmb{s}^{(i-t,j)},\pmb\theta_{i-t}\right] \nonumber\\
&\leq \E\left[\left.\left\Vert \frac{\pi_{\pmb\theta_{i-t}}\left(\pmb{a}^{(i,j)}|\pmb{s}^{(i,j)}\right)}
{\pi_{\pmb\theta_i}\left(\pmb{a}^{(i,j)}|\pmb{s}^{(i,j)}\right)}g\left(\pmb{s}^{(i,j)},\pmb{a}^{(i,j)}|\pmb\theta_{i-t}\right) - g\left({\pmb{s}}^{(i,j)},{\pmb{a}}^{(i,j)}|\pmb\theta_{i-t}\right) \right\Vert\right|\pmb{s}^{(i-t,j)},\pmb\theta_{i-t} \right]\nonumber\\
&\quad +\E\left[\left.\left\Vert g\left({\pmb{s}}^{(i,j)},{\pmb{a}}^{(i,j)}|\pmb\theta_{i-t}\right) - g\left(\tilde{\pmb{s}}^{(i,j)},\tilde{\pmb{a}}^{(i,j)}|\pmb\theta_{i-t}\right)\right\Vert\right|\pmb{s}^{(i-t,j)},\pmb\theta_{i-t}\right]. \nonumber
\end{align}
Let $p\left(\pmb{\tau}\right)=p\left(\pmb{s}^{(i-t,j)},\pmb{a}^{(i-t,j)}, \pmb{s}^{(i-t,j+1)},\ldots,\pmb{s}^{(i,j)}\right)$ denote the probability density of sample path from $\pmb{s}^{(i-t,j)}$ to $\pmb{s}^{(i,j)}$. Conditioning on $\pmb{s}^{(i-t,j)}$ and $\pmb\theta_{i-t}$, we have
\begin{align}
\Delta_{\textbf{(ii)}} &\leq \E\left[\left.\frac{|\pi_{\pmb\theta_{i-t}}\left(\pmb{a}^{(i,j)}|\pmb{s}^{(i,j)}\right)-\pi_{\pmb\theta_{i}}\left(\pmb{a}^{(i,j)}|\pmb{s}^{(i,j)}\right)|}
{\pi_{\pmb\theta_i}\left(\pmb{a}^{(i,j)}|\pmb{s}^{(i,j)}\right)} \left\Vert g\left(\pmb{s}^{(i,j)},\pmb{a}^{(i,j)}|\pmb\theta_{i-t}\right)\right\Vert \right|\pmb{s}^{(i-t,j)},\pmb\theta_{i-t}\right] +2M\nonumber\\
&\leq M \int p\left(\pmb{\tau}\right) \left(\int_{\mathcal{A}}\pi_{\pmb\theta_i}\left(\pmb{a}^{(i,j)}|\pmb{s}^{(i,j)}\right)\frac{|\pi_{\pmb\theta_{i-t}}\left(\pmb{a}^{(i,j)}|\pmb{s}^{(i,j)}\right)-\pi_{\pmb\theta_{i}}\left(\pmb{a}^{(i,j)}|\pmb{s}^{(i,j)}\right)|}
{\pi_{\pmb\theta_i}\left(\pmb{a}^{(i,j)}|\pmb{s}^{(i,j)}\right)} \dd \pmb{a}^{(i,j)}\right)\dd \pmb{\tau} +2M\nonumber\\
&\leq M \int p\left(\pmb{\tau}\right) \left(\int_{\mathcal{A}}\left|\pi_{\pmb\theta_{i-t}}\left(\pmb{a}|\pmb{s}^{(i,j)}\right)-\pi_{\pmb\theta_{i}}\left(\pmb{a}|\pmb{s}^{(i,j)}\right)\right| \dd \pmb{a}\right)\dd \pmb{\tau} +2M\nonumber\\
&\leq 2M \E\left[\left.\left\Vert\pi_{\pmb\theta_{i-t}}\left(\cdot|\pmb{s}^{(i,j)}\right)-\pi_{\pmb\theta_{i}}\left(\cdot|\pmb{s}^{(i,j)}\right)\right\Vert_{TV}\right|\pmb{s}^{(i-t,j)},\pmb\theta_{i-t}\right] +2M\label{eq5: appendix lemma: supportive lemma 2} \\
&\leq 2M U_\pi\E\left[\left.\left\Vert \pmb\theta_{i}-\pmb\theta_{i-t}\right\Vert\right|\pmb{s}^{(i-t,j)},\pmb\theta_{i-t}\right] +2M\label{eq6: appendix lemma: supportive lemma 2}\\
&\leq 2M^2U_\pi\sum^i_{\ell=i-t}\eta_\ell +2M \label{eq7: appendix lemma: supportive lemma 2}
\end{align}
where the expectation in Step~\eqref{eq5: appendix lemma: supportive lemma 2} is taken over the sample path from $\pmb{s}^{(i-t,j)}$ to $\pmb{s}^{(i,j)}$. Step~\eqref{eq6: appendix lemma: supportive lemma 2} holds by Assumption~\ref{assumption 2} and Step~\eqref{eq7: appendix lemma: supportive lemma 2} follows Lemma~\ref{auxillary lemma: parametric distance}. The conclusion follows by the fact that $$\E\left[\left\Vert\widehat{\nabla} J^{LR}\left({\pmb{x}}^{(i,j)}, \pmb\theta_i,\pmb\theta_{i-t}\right)-\widehat{\nabla} J\left(\tilde{\pmb{x}}^{(i,j)}, \pmb\theta_i,\pmb\theta_{i-t}\right)\right\Vert\right] 
\leq 2M^2U_\pi\sum^i_{\ell=i-t}\eta_\ell +2M.$$

\noindent \textbf{(iii)}: We define conditional expectation and notice that $\min(1,U_f) \equiv 1$ as $U_f > 1$ \begin{align*}
    \Delta_{\textbf{(iii)}}=\left\Vert \E\left[\left.\min\left(\frac{\pi_{\pmb\theta_{i-t}}\left(\pmb{a}^{(i,j)}|\pmb{s}^{(i,j)}\right)}{\pi_{\pmb\theta_i}\left(\pmb{a}^{(i,j)}|\pmb{s}^{(i,j)}\right)}, U_f \right)g\left(\pmb{x}^{(i,j)}|\pmb\theta_{i-t}\right)- g\left(\tilde{\pmb{x}}^{(i,j)}|\pmb\theta_{i-t}\right)\right|\pmb{s}^{(i-t,j)},\pmb\theta_{i-t}\right]\right\Vert\\
=\left\Vert \E\left[\left.\min\left(\frac{\pi_{\pmb\theta_{i-t}}\left(\pmb{a}^{(i,j)}|\pmb{s}^{(i,j)}\right)}{\pi_{\pmb\theta_i}\left(\pmb{a}^{(i,j)}|\pmb{s}^{(i,j)}\right)}, U_f \right)g\left(\pmb{x}^{(i,j)}|\pmb\theta_{i-t}\right)- \min(1,U_f) g\left(\tilde{\pmb{x}}^{(i,j)}|\pmb\theta_{i-t}\right)\right|\pmb{s}^{(i-t,j)},\pmb\theta_{i-t}\right]\right\Vert.
\end{align*}
By the Lipschitz property of the function $\min(\cdot, U_f)$, \begin{equation}
    \min\left(\frac{\pi_{\pmb\theta_{i-t}}\left(\pmb{a}^{(i,j)}|\pmb{s}^{(i,j)}\right)}{\pi_{\pmb\theta_i}\left(\pmb{a}^{(i,j)}|\pmb{s}^{(i,j)}\right)}, U_f \right) - \min(1, U_f) \leq \left|\frac{\pi_{\pmb\theta_{i-t}}\left(\pmb{a}^{(i,j)}|\pmb{s}^{(i,j)}\right)}{\pi_{\pmb\theta_i}\left(\pmb{a}^{(i,j)}|\pmb{s}^{(i,j)}\right)}- 1 \right|. \label{eq9: appendix lemma: supportive lemma 2}
\end{equation}
Conditioning on $\pmb{s}^{(i-t,j)}$ and $\pmb\theta_{i-t}$, it holds
\begin{align}
\Delta_{\textbf{(iii)}}&\leq\left\Vert \E\left[\left.\min\left(\frac{\pi_{\pmb\theta_{i-t}}\left(\pmb{a}^{(i,j)}|\pmb{s}^{(i,j)}\right)}{\pi_{\pmb\theta_i}\left(\pmb{a}^{(i,j)}|\pmb{s}^{(i,j)}\right)}, U_f \right)\left[g\left(\pmb{x}^{(i,j)}|\pmb\theta_{i-t}\right)- g\left(\tilde{\pmb{x}}^{(i,j)}|\pmb\theta_{i-t}\right)\right]\right|\pmb{s}^{(i-t,j)},\pmb\theta_{i-t}\right]\right\Vert \nonumber\\
&\quad +\left\Vert \E\left[\left.\left(\min\left(\frac{\pi_{\pmb\theta_{i-t}}\left(\pmb{a}^{(i,j)}|\pmb{s}^{(i,j)}\right)}{\pi_{\pmb\theta_i}\left(\pmb{a}^{(i,j)}|\pmb{s}^{(i,j)}\right)}, U_f \right) - \min(1, U_f)\right)g\left(\tilde{\pmb{x}}^{(i,j)}|\pmb\theta_{i-t}\right)\right|\pmb{s}^{(i-t,j)},\pmb\theta_{i-t}\right]\right\Vert\nonumber\\
&\leq\left\Vert \E\left[\left.U_f\left[g\left(\pmb{x}^{(i,j)}|\pmb\theta_{i-t}\right)- g\left(\tilde{\pmb{x}}^{(i,j)}|\pmb\theta_{i-t}\right)\right]\right|\pmb{s}^{(i-t,j)},\pmb\theta_{i-t}\right] \right\Vert \label{eq8: appendix lemma: supportive lemma 2}\\
&\quad +\E\left[\left.\frac{\left|\pi_{\pmb\theta_{i-t}}\left(\pmb{a}^{(i,j)}|\pmb{s}^{(i,j)}\right)-\pi_{\pmb\theta_i}\left(\pmb{a}^{(i,j)}|\pmb{s}^{(i,j)}\right)\right|}{\pi_{\pmb\theta_i}\left(\pmb{a}^{(i,j)}|\pmb{s}^{(i,j)}\right)} \left\Vert g\left(\tilde{\pmb{x}}^{(i,j)}|\pmb\theta_{i-t}\right)\right\Vert\right|\pmb{s}^{(i-t,j)},\pmb\theta_{i-t}\right] \label{eq10: appendix lemma: supportive lemma 2}
\end{align}
where the term~\eqref{eq10: appendix lemma: supportive lemma 2} holds due to the inequality \eqref{eq9: appendix lemma: supportive lemma 2}.

\noindent \textbf{First}, for term \eqref{eq8: appendix lemma: supportive lemma 2}, we have
\begin{align}
&\left\Vert \E\left[\left.U_f\left[g\left(\pmb{x}^{(i,j)}|\pmb\theta_{i-t}\right)- g\left(\tilde{\pmb{x}}^{(i,j)}|\pmb\theta_{i-t}\right)\right]\right|\pmb{s}^{(i-t,j)},\pmb\theta_{i-t}\right] \right\Vert \nonumber\\
&\quad \leq U_f \left\Vert\int_{\mathcal{S}}\int_{\mathcal{A}} \P\left(\pmb{s}^{(i,j)}=\dd \pmb{s}|\pmb{s}^{(i-t,j)},\pmb\theta_{i-t}\right) \pi_{\pmb\theta_i}(\pmb{a}|\pmb{s})g\left(\pmb{s},\pmb{a}|\pmb\theta_{i-t}\right) \right.\nonumber \\
&\qquad\qquad\quad  - \left.\P\left(\tilde{\pmb{s}}^{(i,j)}=\dd \pmb{s}|\pmb{s}^{(i-t,j)},\pmb\theta_{i-t}\right)\pi_{\pmb\theta_{i-t}}(\pmb{a}|\pmb{s})g\left(\pmb{s},\pmb{a}|\pmb\theta_{i-t}\right) \dd\pmb{a} \right\Vert \nonumber\\
&\quad \leq U_f  \int_{\mathcal{S}}\int_{\mathcal{A}} \left|\P\left(\pmb{s}^{(i,j)}=\dd \pmb{s}|\pmb{s}^{(i-t,j)},\pmb\theta_{i-t}\right) \pi_{\pmb\theta_i}(\pmb{a}|\pmb{s}) -\P\left(\tilde{\pmb{s}}^{(i,j)}=\dd \pmb{s}|\pmb{s}^{(i-t,j)},\pmb\theta_{i-t}\right)\pi_{\pmb\theta_{i-t}}(\pmb{a}|\pmb{s})\right|\nonumber \\
&\qquad\qquad\quad  \Vert g\left(\pmb{s},\pmb{a}|\pmb\theta_{i-t}\right) \Vert \dd\pmb{a} \nonumber\\
&\quad \leq 2MU_f\left\Vert \P\left((\pmb{s}^{(i,j)},\pmb{a}^{(i,j)})\in \cdot|\pmb{s}^{(i-t,j)},\pmb\theta_{i-t}\right)-\P\left((\tilde{\pmb{s}}^{(i,j)},\tilde{\pmb{a}}^{(i,j)})\in \cdot|\pmb{s}^{(i-t,j)},\pmb\theta_{i-t}\right)\right\Vert_{TV}\label{eq11: appendix lemma: supportive lemma 2}
\end{align}
where Step~\eqref{eq11: appendix lemma: supportive lemma 2} holds due to Lemma~\ref{lemma: bounded of policy gradient}. Let $$\Delta^{(i,j)}=2M\left\Vert \P\left((\pmb{s}^{(i,j)},\pmb{a}^{(i,j)})\in \cdot|\pmb{s}^{(i-t,j)},\pmb\theta_{i-t}\right)-\P\left((\tilde{\pmb{s}}^{(i,j)},\tilde{\pmb{a}}^{(i,j)})\in \cdot|\pmb{s}^{(i-t,j)},\pmb\theta_{i-t}\right)\right\Vert_{TV}.$$ By Lemma~\ref{lemma: relations of total variation measures}, we have 
\begin{align}
\Delta^{(i,j)}
&\leq 2M\left\Vert \P\left(\left.\pmb{s}^{(i,j)}\in \cdot\right|\pmb{s}^{(i-t,j)},\pmb\theta_{i-t}\right)-\P\left(\left.\tilde{\pmb{s}}^{(i,j)}\in \cdot\right|\pmb{s}^{(i-t,j)},\pmb\theta_{i-t}\right)\right\Vert_{TV} \nonumber\\
&\quad +2M U_\pi \E\left[\Vert \pmb\theta_{i}-\pmb\theta_{i-t}\Vert|\pmb\theta_{i-t}\right] \nonumber\\
&\leq\Delta^{(i,j-1)} + 2M U_\pi \E\left[\Vert \pmb\theta_{i}-\pmb\theta_{i-t}\Vert|\pmb\theta_{i-t}\right].\label{eq12: appendix lemma: supportive lemma 2}
\end{align}
Repeat the inequality~\eqref{eq12: appendix lemma: supportive lemma 2} from $(i,j)$ to $(i-t,j)$, we have
\begin{align}
    \Delta^{(i,j)}\leq \ldots &\leq \Delta^{(i-t,j)} + 2M U_\pi\sum^{i}_{i^\prime=i-t}\sum^{n}_{j=1}\E[\Vert \pmb\theta_{i^\prime}-\pmb\theta_{i-t}\Vert|\pmb\theta_{i-t}] \nonumber\\
    &= \Delta^{(i-t,j)} + 2n M U_\pi \sum^{i}_{i^\prime=i-t}\E\left[\Vert \pmb\theta_{i^\prime}-\pmb\theta_{i-t}\Vert|\pmb\theta_{i-t}\right]. \nonumber
\end{align}
Notice that $$\Delta^{(i-t,j)}=2M\left\Vert \P\left(\pmb{s}^{(i-t,j)}\in \cdot|\pmb{s}^{(i-t,j)}\right)\pi_{\pmb\theta_{i-t}}(\cdot|\pmb{s})-\P\left(\tilde{\pmb{s}}^{(i-t,j)}\in \cdot|\pmb{s}^{(i-t,j)}\right)\pi_{\pmb\theta_{i-t}}(\cdot|\pmb{s})\right\Vert_{TV}=0.$$ Thus we have  $\Delta^{(i,j)}\leq 2n M U_\pi \sum^{i}_{i^\prime=i-t}\E[\Vert \pmb\theta_{i^\prime}-\pmb\theta_{i-t}\Vert|\pmb\theta_{i-t}]$. 
By Lemma~\ref{auxillary lemma: parametric distance}, we have $\E[\Vert \pmb\theta_{i^\prime}-\pmb\theta_{i-t}\Vert|\pmb\theta_{i-t}] \leq M\sum^{i^\prime}_{\ell=i-t}{\eta_{\ell}}$ and thus it holds that
\begin{equation*}
\Delta^{(i,j)} \leq 2nU_\pi M^2\sum^{i}_{i^\prime=i-t}  \sum^{i^\prime}_{\ell=i-t}\eta_{\ell} \leq 2nU_\pi M^2t  \sum^{i}_{\ell=i-t}\eta_{\ell}.
\end{equation*}
Thus, the term \eqref{eq8: appendix lemma: supportive lemma 2} is bounded by
\begin{equation} \label{eq8: bound appendix lemma: supportive lemma 2}
\left\Vert \E\left[\left.U_f\left[g\left(\pmb{x}^{(i,j)}|\pmb\theta_{i-t}\right)- g\left(\tilde{\pmb{x}}^{(i,j)}|\pmb\theta_{i-t}\right)\right]\right|\pmb{s}^{(i-t,j)},\pmb\theta_{i-t}\right]\right\Vert \leq 2 U_f nU_\pi M^2t  \sum^{i}_{\ell=i-t}\eta_{\ell}.
\end{equation}

\noindent \textbf{Second}, for term \eqref{eq10: appendix lemma: supportive lemma 2}, based on the similar steps as Eq.~\eqref{eq7: appendix lemma: supportive lemma 2}, we have
\begin{align}
&\E\left[\left.\frac{\left|\pi_{\pmb\theta_{i-t}}\left(\pmb{a}^{(i,j)}|\pmb{s}^{(i,j)}\right)-\pi_{\pmb\theta_i}\left(\pmb{a}^{(i,j)}|\pmb{s}^{(i,j)}\right)\right|}{\pi_{\pmb\theta_i}\left(\pmb{a}^{(i,j)}|\pmb{s}^{(i,j)}\right)} \left\Vert g\left(\tilde{\pmb{x}}^{(i,j)}|\pmb\theta_{i-t}\right)\right\Vert\right|\pmb{s}^{(i-t,j)},\pmb\theta_{i-t}\right]\nonumber\\
&\leq 2M^2U_\pi\sum^i_{\ell=i-t}\eta_\ell. 
\label{eq10: bound appendix lemma: supportive lemma 2}
\end{align}
By putting \eqref{eq8: bound appendix lemma: supportive lemma 2} and \eqref{eq10: bound appendix lemma: supportive lemma 2} all together and taking full expectation, we have
\begin{equation}\label{eq14: appendix lemma: supportive lemma 2}
   \E[\Delta_{\textbf{(iii)}}] \leq 2nM^2U_\pi U_f t\sum^i_{\ell=i-t}\eta_\ell.
\end{equation}

\noindent \textbf{(iv):} Notice that for any $\pmb{s},\tilde{\pmb{s}}\in\mathcal{S}$ and $\pmb{a},\tilde{\pmb{a}}\in\mathcal{A}$, it holds
\begin{align}
    &\left\Vert\min \left(\frac{\pi_{\pmb\theta_{i-t}}\left(\pmb{a}|\pmb{s}\right)}{\pi_{\pmb\theta_{i}}\left(\pmb{a}^|\pmb{s}\right)},U_f\right)g\left(\pmb{x}|\pmb\theta_{i-t}\right)- g\left(\tilde{\pmb{x}}|\pmb\theta_{i-t}\right) \right\Vert  \nonumber\\
    &\leq  \left\Vert\min \left(\frac{\pi_{\pmb\theta_{i-t}}\left(\pmb{a}|\pmb{s}\right)}{\pi_{\pmb\theta_{i}}\left(\pmb{a}|\pmb{s}\right)},U_f\right)g\left(\pmb{x}|\pmb\theta_{i-t}\right)- \min \left(\frac{\pi_{\pmb\theta_{i-t}}\left(\pmb{a}|\pmb{s}\right)}{\pi_{\pmb\theta_{i}}\left(\pmb{a}|\pmb{s}\right)},U_f\right)g\left(\tilde{\pmb{x}}|\pmb\theta_{i-t}\right)\right\Vert \nonumber\\
    &+ \left \Vert \min \left(\frac{\pi_{\pmb\theta_{i-t}}\left(\pmb{a}|\pmb{s}\right)}{\pi_{\pmb\theta_{i}}\left(\pmb{a}|\pmb{s}\right)},U_f\right)g\left(\tilde{\pmb{x}}|\pmb\theta_{i-t}\right)-g\left(\tilde{\pmb{x}}|\pmb\theta_{i-t}\right) \right\Vert \nonumber\\
&\leq  \min \left(\frac{\pi_{\pmb\theta_{i-t}}\left(\pmb{a}|\pmb{s}\right)}{\pi_{\pmb\theta_{i}}\left(\pmb{a}|\pmb{s}\right)},U_f\right)\left\Vert g\left(\pmb{x}|\pmb\theta_{i-t}\right)- g\left(\tilde{\pmb{x}}|\pmb\theta_{i-t}\right)\right\Vert + \left | \min \left(\frac{\pi_{\pmb\theta_{i-t}}\left(\pmb{a}|\pmb{s}\right)}{\pi_{\pmb\theta_{i}}\left(\pmb{a}|\pmb{s}\right)},U_f\right)-1 \right| \Vert g\left(\tilde{\pmb{x}}|\pmb\theta_{i-t}\right)\Vert \nonumber\\
&\leq 2U_f M + (U_f+1)M \nonumber
\end{align}
from which the conclusion immediately follows.
\end{proof}

\begin{lemma}\label{appendix lemma: supportive lemma 3}
   Under Assumption~\ref{assumption 2} and \ref{assumption 3}, for any $t$ such that $t <i\leq k$, we have
   \begin{itemize}
       \item[\textbf{(i)}] $\left\Vert \E\left[\widehat{\nabla} J\left(\tilde{\pmb{x}}^{(i,j)}, \pmb\theta_i,\pmb\theta_{i-t}\right)-\widehat{\nabla} J\left(\check{\pmb{x}}^{(i,j)}, \pmb\theta_i,\pmb\theta_{i-t}\right)\right] \right\Vert\leq 2M \varphi(nt)$
       \item[\textbf{(ii)}] $ \E\left[\left\Vert\widehat{\nabla} J\left(\tilde{\pmb{x}}^{(i,j)}, \pmb\theta_i,\pmb\theta_{i-t}\right)-\widehat{\nabla} J\left(\check{\pmb{x}}^{(i,j)}, \pmb\theta_i,\pmb\theta_{i-t}\right)\right\Vert\right] \leq 2M.$
   \end{itemize}
\end{lemma}
\begin{proof}
\noindent \textbf{(i)} Let $\Delta = \left\Vert\E\left[\left.\widehat{\nabla} J\left(\tilde{\pmb{x}}^{(i,j)}, \pmb\theta_i,\pmb\theta_{i-t}\right)-\widehat{\nabla} J\left(\check{\pmb{x}}^{(i,j)}, \pmb\theta_i,\pmb\theta_{i-t}\right)\right|\pmb{s}^{(i-t,j)},\pmb\theta_{i-t}\right]\right\Vert$. It holds
\begin{align}
\Delta&=
\left\Vert \E\left[\left. g\left(\tilde{\pmb{s}}^{(i,j)}, \tilde{\pmb{a}}^{(i,j)}|\pmb\theta_{i-t}\right) - g\left(\check{\pmb{s}}^{(i,j)}, \check{\pmb{a}}^{(i,j)}|\pmb\theta_{i-t}\right) \right|\pmb{s}^{(i-t,j)},\pmb\theta_{i-t}\right]\right\Vert \nonumber\\
&=\left\Vert \int \P\left(\tilde{\pmb{s}}^{(i,j)}=\dd \pmb{s}|\pmb{s}^{(i-t,j)},\pmb\theta_{i-t}\right)\pi_{\pmb\theta_{i-t}}\left(\pmb{a}|\pmb{s}\right)g\left(\pmb{s}, \pmb{a}|\pmb\theta_{i-t}\right) \dd \pmb{a} \right. \nonumber\\
& \qquad - \left. \int d^{\pi_{\pmb\theta_{i-t}}}(\pmb{s})\pi_{\pmb\theta_{i-t}}\left(\pmb{a}|\pmb{s}\right)g\left(\pmb{s}, \pmb{a}|\pmb\theta_{i-t}\right) \dd \pmb{s} \dd\pmb{a} \right\Vert \nonumber\\
&=\left\Vert \int \left(p\left(\tilde{\pmb{s}}^{(i,j)}= \pmb{s}|\pmb{s}^{(i-t,j)},\pmb\theta_{i-t}\right)-d^{\pi_{\pmb\theta_{i-t}}}(\pmb{s})\right)\pi_{\pmb\theta_{i-t}}\left(\pmb{a}|\pmb{s}\right)g\left(\pmb{s}, \pmb{a}|\pmb\theta_{i-t}\right) \dd \pmb{s} \dd\pmb{a} \right\Vert \nonumber\\
&\leq \int \left|p\left(\tilde{\pmb{s}}^{(i,j)}= \pmb{s}|\pmb{s}^{(i-t,j)},\pmb\theta_{i-t}\right)-d^{\pi_{\pmb\theta_{i-t}}}(\pmb{s})\right|\pi_{\pmb\theta_{i-t}}\left(\pmb{a}|\pmb{s}\right) \left\Vert 
 g\left(\pmb{s}, \pmb{a}|\pmb\theta_{i-t}\right)\right\Vert \dd \pmb{s} \dd\pmb{a} \nonumber\\
 &\leq M\int \left|p\left(\tilde{\pmb{s}}^{(i,j)}= \pmb{s}|\pmb{s}^{(i-t,j)},\pmb\theta_{i-t}\right)-d^{\pi_{\pmb\theta_{i-t}}}(\pmb{s})\right|\pi_{\pmb\theta_{i-t}}\left(\pmb{a}|\pmb{s}\right)  \dd \pmb{s} \dd\pmb{a} \label{eq1: appendix lemma: supportive lemma 3}\\
 &= M\int \left|p\left(\tilde{\pmb{s}}^{(i,j)}= \pmb{s}|\pmb{s}^{(i-t,j)},\pmb\theta_{i-t}\right)-d^{\pi_{\pmb\theta_{i-t}}}(\pmb{s})\right| \left(\int \pi_{\pmb\theta_{i-t}}\left(\pmb{a}|\pmb{s}\right) \dd\pmb{a} \right) \dd \pmb{s}\nonumber\\
 &= 2M \left\Vert \P\left(\tilde{\pmb{s}}^{(i,j)}\in \cdot|\pmb{s}^{(i-t,j)},\pmb\theta_{i-t}\right)-d^{\pi_{\pmb\theta_{i-t}}}(\cdot)\right\Vert_{TV}\nonumber\\
 &\leq 2M \varphi(nt). \label{eq2: appendix lemma: supportive lemma 3}
\end{align}
Then the conclusion follows from the fact that 
$$\left\Vert \E\left[\widehat{\nabla} J\left(\tilde{\pmb{x}}^{(i,j)}, \pmb\theta_i,\pmb\theta_{i-t}\right)-\widehat{\nabla} J\left(\check{\pmb{x}}^{(i,j)}, \pmb\theta_i,\pmb\theta_{i-t}\right)\right] \right\Vert
\leq 2M\varphi(nt).$$

\noindent \textbf{(ii)} The conclusion can be obtained by observing that $$\E\left[\left\Vert g\left(\tilde{\pmb{s}}^{(i,j)}, \tilde{\pmb{a}}^{(i,j)}|\pmb\theta_{i-t}\right) - g\left(\check{\pmb{s}}^{(i,j)}, \check{\pmb{a}}^{(i,j)}|\pmb\theta_{i-t}\right)\right\Vert\right]\leq \E[2M]=2M,$$
which concludes the proof.
\end{proof}

\begin{lemma}\label{appendix lemma: supportive lemma 4}
   Under Assumptions~\ref{assumption 2} and \ref{assumption 3}, for any $t$ such that $t <i\leq k$, we have
   \begin{equation*}
       \Vert\E[\Delta]\Vert\leq \E\left[\Vert\Delta\Vert\right]\leq L_g M \sum^k_{\ell=i-t}\eta_\ell
   \end{equation*}
   where $\Delta=\widehat{\nabla} J\left(\check{\pmb{x}}^{(i,j)}, \pmb\theta_i,\pmb\theta_{i-t}\right)-\widehat{\nabla} J\left(\check{\pmb{x}}^{(i,j)}, \pmb\theta_i,\pmb\theta_{k}\right)$.
\end{lemma}
\begin{proof}
By Lemma~\ref{lemma: lipchitz continuity of sample gradient estimate},
\begin{equation*}
\E[\Vert\Delta \Vert] = \E\left[\Vert g(\check{\pmb{x}}^{(i,j)}|\pmb\theta_{i-t})-g(\check{\pmb{x}}^{(i,j)}|\pmb\theta_{k})\Vert \right] \leq L_g\E[\Vert\pmb\theta_k-\pmb\theta_{i-t}\Vert].
\end{equation*}
The conclusion follows by applying Lemma~\ref{auxillary lemma: parametric distance}.
\end{proof}

\begin{lemma}\label{appendix lemma: supportive lemma 5}
   Under Assumption~\ref{assumption 2} and \ref{assumption 3}, for any $t$ such that $t <i\leq k$, we have
   \begin{itemize}
     \item[\textbf{(i)}] $\left\Vert \E\left[\widehat{\nabla} J\left(\check{\pmb{x}}^{(i,j)}, \pmb\theta_i,\pmb\theta_{k}\right)-\nabla  J\left(\pmb\theta_k\right)\right] \right\Vert \leq 2M^2 C_d\sum^{k}_{\ell=i-t}\eta_{\ell}$
     \item[\textbf{(ii)}] $\E\left[\left\Vert \widehat{\nabla} J\left(\check{\pmb{x}}^{(i,j)}, \pmb\theta_i,\pmb\theta_{k}\right)-\nabla  J\left(\pmb\theta_k\right)\right\Vert\right]  \leq 2M$
   \end{itemize}
\end{lemma}
\begin{proof}
\noindent \textbf{(i)}:
Let $\Delta = \E\left[\left.\widehat{\nabla} J\left(\check{\pmb{x}}^{(i,j)}, \pmb\theta_i,\pmb\theta_{k}\right)-\nabla  J\left(\pmb\theta_k\right)\right|\pmb{s}^{i-t,j},\pmb\theta_{i-t}\right]$. For notational simplicity, we omit the conditions on $(\pmb{s}^{i-t,j},\pmb\theta_{i-t})$ below. Then we have
\begin{align}
\Vert\Delta \Vert&= \left\Vert\E\left[\E_{(\pmb{s},\pmb{a})\sim d^{\pi_{\pmb\theta_{i-t}}}(\cdot,\cdot)}[g(\pmb{s},\pmb{a}|\pmb\theta_k)]\right] - \E\left[\E_{(\pmb{s},\pmb{a})\sim d^{\pi_{\pmb\theta_{k}}}(\cdot,\cdot)}[g(\pmb{s},\pmb{a}|\pmb\theta_k)]\right] \right\Vert\nonumber\\
&\leq \E\left[\left\Vert \int d^{\pi_{\pmb\theta_{i-t}}}(\pmb{s},\pmb{a}) g(\pmb{s}, \pmb{a}|\pmb\theta_k)\dd \pmb{s}\dd \pmb{a} - \int d^{\pi_{\pmb\theta_{k}}}(\pmb{s},\pmb{a}) g(\pmb{s}, \pmb{a}|\pmb\theta_k)\dd \pmb{s}\dd \pmb{a}\right\Vert\right] \nonumber\\
&\leq \E\left[ \int | d^{\pi_{\pmb\theta_{i-t}}}(\pmb{s},\pmb{a}) - d^{\pi_{\pmb\theta_{k}}}(\pmb{s},\pmb{a}) |\left\Vert g(\pmb{s}, \pmb{a}|\pmb\theta_k) \right\Vert\dd \pmb{s}\dd \pmb{a}\right] \nonumber \\
&\leq 2M \E\left[\left\Vert d^{\pi_{\pmb\theta_{i-t}}}(\pmb{s},\pmb{a}) - d^{\pi_{\pmb\theta_{k}}}(\pmb{s},\pmb{a}) \right\Vert_{TV}\right] \label{eq1: appendix lemma: supportive lemma 5} \\
&\leq 2M C_d \E\left[\left\Vert \pmb\theta_{k}-\pmb\theta_{i-t}\right\Vert \right] \label{eq2: appendix lemma: supportive lemma 5}
\end{align}
where $C_d=U_\pi(1+\lceil\log_{\kappa}\kappa_0^{-1}\rceil+\frac{1}{1-\kappa})$ and $\kappa\geq 1$. Step~\eqref{eq1: appendix lemma: supportive lemma 5} holds due to Lemma~\ref{lemma: bounded of policy gradient} and Step~\eqref{eq2: appendix lemma: supportive lemma 5} is by Lemma~\ref{lemma: lipchitz continuity of stationary distribution}. The conclusion is obtained by applying Lemma~\ref{auxillary lemma: parametric distance}, i.e.
\begin{equation*}
  \left\Vert \E\left[\widehat{\nabla} J\left(\check{\pmb{x}}^{(i,j)}, \pmb\theta_i,\pmb\theta_{k}\right)-\nabla  J\left(\pmb\theta_k\right)\right] \right\Vert=\Vert\E[\Delta]\Vert\leq \E[\Vert\Delta\Vert]\leq 2M^2 C_d\sum^{k}_{\ell=i-t}\eta_{\ell}.
\end{equation*}

\noindent \textbf{(ii)}: The conclusion is obtained by Minkowski's inequality and Lemma~\ref{lemma: bounded of policy gradient}:
\begin{align*}
    &\E\left[\left\Vert \widehat{\nabla} J\left(\check{\pmb{x}}^{(i,j)}, \pmb\theta_i,\pmb\theta_{k}\right)-\nabla  J\left(\pmb\theta_k\right)\right\Vert\right] \\
    &=\E\left[\left\Vert g(\check{\pmb{x}}^{(i,j)}|\pmb\theta_k)-g(\pmb{x}|\pmb\theta_k)\right\Vert\right] \leq \E\left[\left\Vert g(\check{\pmb{x}}^{(i,j)}|\pmb\theta_k)\right\Vert +\left \Vert g(\pmb{x}|\pmb\theta_k)\right\Vert\right] \leq 2M
\end{align*}
which completes the proof.
\end{proof}

\begin{lemma}\label{appendix lemma: supportive lemma 6}
  Let $\bar{\rho}_k = \frac{1}{|\mathcal{U}_k|^2}\sum_{\pmb\theta_i\in\mathcal{U}_k}\sum_{\pmb\theta_{i^\prime}\in\mathcal{U}_k}|\max_{\ell=1,2,\ldots,d}\left(\Corr^{(\ell)}_{i,i^\prime,k}\right)|w_{i,k} w_{i^\prime,k}$. For a learning rate $\eta_k=\eta_1 k^{-r}$ with $\eta_1 \leq \frac{1}{4L}$ and $r\in(0,1)$, under Assumptions~\ref{assumption 2} and \ref{assumption 3}, we have
    \begin{equation*}
        \left|\frac{1}{n}\sum^n_{j=1}\E\left[\Gamma\left(\pmb{x}^{(i,j)},\pmb\theta_i, \pmb\theta_k\right)-\Gamma\left(\pmb{x}^{(i,j)},\pmb\theta_i,\pmb\theta_{i-t}\right) \right]\right|\leq  C^\Gamma_1 (k-i+t)^{1/2}\eta_{i-t} + C_2^\Gamma (k-i+t)\eta_{i-t}
    \end{equation*}
where $C_1^\Gamma = {LM^2 (w_{i,k}+2)} \sqrt{\bar{\rho}_k + 1}$ and $C_2^\Gamma = M^2(U_f L_g+M U_\pi)$.
\end{lemma}
\begin{proof}
Let $\Delta=\frac{1}{n}\sum^n_{j=1}\left(\Gamma\left(\pmb{x}^{(i,j)},\pmb\theta_i, \pmb\theta_k\right)-\Gamma\left(\pmb{x}^{(i,j)},\pmb\theta_i,\pmb\theta_{i-t}\right)\right)$. 
Recall that $$\widehat{\nabla} J^{R}_{i,k}=\frac{1}{n}\sum^n_{j=1}\widehat{\nabla} J^{R}(\pmb{x}^{(i,j)},\pmb\theta_i,\pmb\theta_k)\text{  and  } \widehat{\nabla} J^{R}_{i,i-t}=\frac{1}{n}\sum^n_{j=1}\widehat{\nabla} J^{R}(\pmb{x}^{(i,j)},\pmb\theta_i,\pmb\theta_{i-t}),$$
where $R\in\{LR,CLR\}$. Then we have
\begin{align}
    \E[\Delta] &= \E\left[\left\langle \nabla J(\pmb\theta_k),\frac{1}{n}\sum^n_{j=1}\widehat{\nabla} J^{R}(\pmb{x}^{(i,j)},\pmb\theta_i,\pmb\theta_k)- \nabla J(\pmb\theta_k)\right\rangle\right.\nonumber\\
    &\qquad -\left. \left\langle \nabla J(\pmb\theta_{i-t}),\frac{1}{n}\sum^n_{j=1}\widehat{\nabla} J^{R}(\pmb{x}^{(i,j)},\pmb\theta_i,\pmb\theta_{i-t})-\nabla J(\pmb\theta_{i-t})\right\rangle\right] \nonumber\\
&=\Delta_1+\Delta_2 +\Delta_3
\end{align}
where 
\begin{align}
\Delta_1 &= \E\left[\left\langle \nabla J(\pmb\theta_k),\widehat{\nabla} J^{R}_{i,k}-\nabla J(\pmb\theta_k)\right\rangle \right. \nonumber\\
    &\qquad -\left.\left\langle \nabla J(\pmb\theta_{i-t}),\widehat{\nabla} J^{R}_{i,k}-\nabla J(\pmb\theta_{k})\right\rangle\right] \nonumber\\
\Delta_2 &=\E\left[\left\langle \nabla J(\pmb\theta_{i-t}),\widehat{\nabla} J^{R}_{i,k}-\nabla J(\pmb\theta_k)\right\rangle \right. \nonumber\\
 &\qquad -\left.\left\langle \nabla J(\pmb\theta_{i-t}),\widehat{\nabla} J^{R}_{i,k}-\nabla J(\pmb\theta_{i-t})\right\rangle\right] \nonumber\\
\Delta_3 &=\E\left[\left\langle \nabla J(\pmb\theta_{i-t}),\widehat{\nabla} J^{R}_{i,k}-\nabla J(\pmb\theta_{i-t})\right\rangle \right. \nonumber\\
    &\qquad -\left.\left\langle \nabla J(\pmb\theta_{i-t}),\widehat{\nabla} J^{R}_{i,i-t}-\nabla J(\pmb\theta_{i-t})\right\rangle\right]. \nonumber
\end{align}
By Minkowski's inequality, the norm $|\E[\Delta]|\leq |\Delta_1|+|\Delta_2|+|\Delta_3|$.

\noindent\textbf{(i)} For the term $|\Delta_1|$: 
\begin{align}
| \Delta_1 | &= \left| \E\left[\left\langle \nabla J(\pmb\theta_k),\widehat{\nabla} J^{R}_{i,k}-\nabla J(\pmb\theta_k)\right\rangle \right.\right. \nonumber\\
    &\qquad -\left.\left.\left\langle \nabla J(\pmb\theta_{i-t}),\widehat{\nabla} J^{R}_{i,k}-\nabla J(\pmb\theta_{k})\right\rangle\right] \right| \nonumber\\
&= \left| \E\left[\left\langle \nabla J(\pmb\theta_k)-\nabla J(\pmb\theta_{i-t}),\widehat{\nabla} J^{R}_{i,k}-\nabla J(\pmb\theta_k)\right\rangle\right] \right| \nonumber \\
&\leq  \E\left[\left\Vert\nabla J(\pmb\theta_k)-\nabla J(\pmb\theta_{i-t})\right\Vert^2 \right]^{1/2} \E\left[\left\Vert\widehat{\nabla} J^{R}_{i,k}-\nabla J(\pmb\theta_k)\right\Vert^2 \right]^{1/2} \label{eq1: appendix lemma: supportive lemma 6} \\
&\leq  L \E[\Vert \pmb\theta_{k}-\pmb\theta_{i-t}\Vert^2]^{1/2}\E\left[\left\Vert\widehat{\nabla} J^{R}_{i,k}-\nabla J(\pmb\theta_k)\right\Vert^2 \right]^{1/2} \label{eq2: appendix lemma: supportive lemma 6}
\end{align}
where Step~\eqref{eq1: appendix lemma: supportive lemma 6} follows by Cauchy–Schwarz inequality and Step~\eqref{eq2: appendix lemma: supportive lemma 6} holds due to Lemma~\ref{lemma: Lipschitz continuity}. For the first term of \eqref{eq2: appendix lemma: supportive lemma 6}, applying  Lemma~\ref{auxillary lemma: L2 parametric distance} gives
\begin{equation*}
   \E\left[\Vert \pmb\theta_{k}-\pmb\theta_{i-t}\Vert^2\right]^{1/2} \leq \sqrt{\bar{\rho}_k+1} M(k-i+t)^{1/2} \eta_{i-t}. 
\end{equation*}

\noindent For the second term of Step~\eqref{eq2: appendix lemma: supportive lemma 6}, applying Minkowski's inequality and Lemma~\ref{lemma: bounded of policy gradient} gives
\begin{align*}
   \E\left[\left\Vert\widehat{\nabla} J^{R}_{i,k}-\nabla J(\pmb\theta_k)\right\Vert^2 \right]^{1/2} &\leq \E\left[\left\Vert\widehat{\nabla} J^{R}_{i,k}\right\Vert^2\right]^{1/2}+\E\left[\left\Vert\nabla J(\pmb\theta_k)\right\Vert^2 \right]^{1/2} \nonumber\\
    &\leq \E\left[\left\Vert\widehat{\nabla} J^{R}_{i,k}\right\Vert^2\right]^{1/2}+M.
\end{align*}

Recall the inequality~\eqref{eq: expected norm of LR policy gradients}, which shows $\E\left[\left\Vert\widehat{\nabla} J^{R}_{i,k}\right\Vert^2\right] \leq M^2w_{i,k}^2$ and consequently, the second term of Step~\eqref{eq2: appendix lemma: supportive lemma 6} becomes
\begin{equation*}
    \E\left[\left\Vert\widehat{\nabla} J^{R}_{i,k}-\nabla J(\pmb\theta_k)\right\Vert^2 \right]^{1/2} \leq M(w_{i,k} +1).
\end{equation*}
Then we have
\begin{equation*}
    | \Delta_1 | \leq L(w_{i,k} +1)(\sqrt{\bar{\rho}_k+1}) M^2 (k-i+t)^{1/2} \eta_{i-t}.
\end{equation*}

\noindent\textbf{(ii)} For the term $|\Delta_2|$: 
\begin{align}
| \Delta_2| &=\left|\E\left[\left\langle \nabla J(\pmb\theta_{i-t}),\widehat{\nabla} J^{R}(\pmb{x}^{(i,j)},\pmb\theta_i,\pmb\theta_k)-\nabla J(\pmb\theta_k)\right\rangle \right. \right.\nonumber\\
 &\qquad -\left.\left.\left\langle \nabla J(\pmb\theta_{i-t}),\widehat{\nabla} J^{R}(\pmb{x}^{(i,j)},\pmb\theta_i,\pmb\theta_{k})-\nabla J(\pmb\theta_{i-t})\right\rangle\right]\right| \nonumber\\
&= \left| \E\left[\left\langle \nabla J(\pmb\theta_{i-t}),\nabla J(\pmb\theta_k)-\nabla J(\pmb\theta_{i-t})\right\rangle\right] \right| \nonumber \\
&\leq \E\left[\left\Vert \nabla J(\pmb\theta_{i-t})\right\Vert^2\right]^{1/2}\E\left[\left\Vert\nabla J(\pmb\theta_k)-\nabla J(\pmb\theta_{i-t})\right\Vert^2\right]^{1/2} \nonumber\\
&\leq LM\E\left[\Vert\pmb\theta_k-\pmb\theta_{i-t}\Vert^2\right]^{1/2} \label{eq.1: appendix lemma: supportive lemma 6}\\
&\leq L\sqrt{\bar{\rho}_k+1} M^2 (k-i+t)^{1/2} \eta_{i-t} \nonumber
\end{align}
where the first inequality follows by applying Cauchy–Schwarz inequality and Lemma~\ref{lemma: Lipschitz continuity} and Step~\eqref{eq.1: appendix lemma: supportive lemma 6} is due to Lemma~\ref{lemma: bounded of policy gradient}. The last step is held by using the bound from Lemma~\ref{auxillary lemma: L2 parametric distance}.

\noindent\textbf{(iii)} For the term $|\Delta_3|$:
\begin{align}
|\Delta_3 | &\leq \left|\E\left[\E\left[\left\langle \nabla J(\pmb\theta_{i-t}),\widehat{\nabla} J^{R}(\pmb{x}^{(i,j)},\pmb\theta_i,\pmb\theta_k)-\nabla J(\pmb\theta_{i-t})\right\rangle \right.\right.\right. \nonumber\\
    &\qquad -\left.\left.\left.\left.\left\langle \nabla J(\pmb\theta_{i-t}),\widehat{\nabla} J^{R}(\pmb{x}^{(i,j)},\pmb\theta_i,\pmb\theta_{i-t})-\nabla J(\pmb\theta_{i-t})\right\rangle \right| \pmb{s}^{(i-t,j)},\pmb\theta_{i-t}\right] \right]\right| \nonumber\\
& \leq \left.\E\left[\left| \left\langle \nabla J(\pmb\theta_{i-t}),\E\left[\widehat{\nabla} J^{R}(\pmb{x}^{(i,j)},\pmb\theta_i,\pmb\theta_{k})-\widehat{\nabla} J^{R}(\pmb{x}^{(i,j)},\pmb\theta_i,\pmb\theta_{i-t}) \right| \pmb{s}^{(i-t,j)},\pmb\theta_{i-t} \right]\right\rangle\right| \right]\label{eq3-1: appendix lemma: supportive lemma 6}\\
& \leq \E\left[\left\Vert \nabla J(\pmb\theta_{i-t})\right\Vert \left\Vert\E\left[\left.\widehat{\nabla} J^{R}(\pmb{x}^{(i,j)},\pmb\theta_i,\pmb\theta_{k})-\widehat{\nabla} J^{R}(\pmb{x}^{(i,j)},\pmb\theta_i,\pmb\theta_{i-t}) \right| \pmb{s}^{(i-t,j)},\pmb\theta_{i-t} \right] \right\Vert \right]\label{eq3: appendix lemma: supportive lemma 6}\\
&\leq M \E\left[\left.\left\Vert \E\left[ \widehat{\nabla} J^{R}(\pmb{x}^{(i,j)},\pmb\theta_i,\pmb\theta_{k})-\widehat{\nabla} J^{R}(\pmb{x}^{(i,j)},\pmb\theta_i,\pmb\theta_{i-t})  \right| \pmb{s}^{(i-t,j)},\pmb\theta_{i-t}\right] \right\Vert\right] \label{eq4: appendix lemma: supportive lemma 6}\\
&\leq M \E\left[\left\Vert \widehat{\nabla} J^{R}(\pmb{x}^{(i,j)},\pmb\theta_i,\pmb\theta_{k})-\widehat{\nabla} J^{R}(\pmb{x}^{(i,j)},\pmb\theta_i,\pmb\theta_{i-t}) \right\Vert\right] \label{eq5: appendix lemma: supportive lemma 6}
\end{align}
where Step~\eqref{eq3-1: appendix lemma: supportive lemma 6} and \eqref{eq5: appendix lemma: supportive lemma 6} both follows Jensen's inequality, Step~\eqref{eq3: appendix lemma: supportive lemma 6} follows Cauchy-Schwarz inequality and Step~\eqref{eq4: appendix lemma: supportive lemma 6} holds by applying Lemma~\ref{lemma: bounded of policy gradient}.

By applying Lemma~\ref{appendix lemma: supportive lemma 1}, we have
\begin{equation}
    | \Delta_3 | =\begin{cases}
M^2 (L_g+M U_\pi) \sum^k_{\ell=i-t}\eta_\ell & \text{R $=$ LR}; \\
M^2 (U_f L_g+M U_\pi) \sum^k_{\ell=i-t}\eta_\ell & \text{R $=$ CLR}.
\end{cases}
\end{equation}
Because $\{\eta_k\}$ is a non-increasing sequence, it holds that $\sum^k_{\ell=i-t}\eta_\ell \leq (k-i+t)\eta_{i-t}$. Recall $U_f > 1$. Thus we have $$| \Delta_3 |\leq M^2(U_f L_g+M U_\pi) \sum^k_{\ell=i-t}\eta_\ell\leq M^2(U_f L_g+M U_\pi) (k-i+t)\eta_{i-t}.$$
The conclusion is obtained by summing $|\Delta_1|$, $|\Delta_2|$ and $|\Delta_3|$.
\end{proof}
\begin{lemma}\label{appendix lemma: supportive lemma 7}
   Under Assumptions~\ref{assumption 2} and \ref{assumption 3}, we have
\begin{align*}
&\left|\E\left[\Gamma\left(\pmb{x}^{(i,j)},\pmb\theta_i, \pmb\theta_{i-t}\right)-\Gamma\left(\tilde{\pmb{x}}^{(i,j)},\pmb\theta_i,\pmb\theta_{i-t}\right) \right]\right|\nonumber\\
&\leq \begin{cases}
 2 nM^3 U_\pi t  \sum^{i}_{\ell=i-t}\eta_{\ell} &\text{LR estimator};\\
 2nM^3U_\pi U_f t\sum^i_{\ell=i-t}\eta_\ell &\text{CLR estimator.}
\end{cases}
\end{align*}
\end{lemma}
\begin{proof}
Let $\Delta =\Gamma\left(\pmb{x}^{(i,j)},\pmb\theta_i, \pmb\theta_{i-t}\right)-\Gamma\left(\tilde{\pmb{x}}^{(i,j)},\pmb\theta_i,\pmb\theta_{i-t}\right)$. For both LR and CLR policy gradient estimators, i.e. $R\in\{LR,CLR\}$, we have
\begin{align}
    |\E[\Delta]|&= \left|\E\left[\E\left[\left\langle \nabla J(\pmb\theta_{i-t}),\widehat{\nabla} J^{R}(\pmb{x}^{(i,j)},\pmb\theta_i,\pmb\theta_{i-t})- \nabla J(\pmb\theta_{i-t})\right\rangle\right.\right.\right.\nonumber\\
    &\qquad -\left.\left. \left.\left. \left\langle \nabla J(\pmb\theta_{i-t}),\widehat{\nabla} J(\tilde{\pmb{x}}^{(i,j)},\pmb\theta_i,\pmb\theta_{i-t})-\nabla J(\pmb\theta_{i-t})\right\rangle\right| \pmb{s}^{(i-t,j)},\pmb\theta_{i-t}\right]\right]\right| \nonumber\\
&\leq \E\left[\left|\left\langle \nabla J(\pmb\theta_{i-t}),\E\left[\left.\widehat{\nabla} J^{R}\left(\pmb{x}^{(i,j)},\pmb\theta_i,\pmb\theta_{i-t}\right)- \widehat{\nabla} J\left(\tilde{\pmb{x}}^{(i,j)},\pmb\theta_i,\pmb\theta_{i-t}\right)\right| \pmb{s}^{(i-t,j)},\pmb\theta_{i-t}\right]\right\rangle \right|\right]\nonumber\\
&\leq M\E\left[\left\Vert \E\left[\left.\widehat{\nabla} J^{R}\left(\pmb{x}^{(i,j)},\pmb\theta_i,\pmb\theta_{i-t}\right)- \widehat{\nabla} J\left(\tilde{\pmb{x}}^{(i,j)},\pmb\theta_i,\pmb\theta_{i-t}\right)\right| \pmb{s}^{(i-t,j)},\pmb\theta_{i-t}\right]\right\Vert\right] \label{eq1: appendix lemma: supportive lemma 7}
\end{align}
where the first inequality holds by applying Jensen's inequality and Step~\eqref{eq1: appendix lemma: supportive lemma 7} is by Cauchy–Schwarz inequality and Lemma~\ref{lemma: bounded of policy gradient}.
Then the conclusion follows by using Eq.~\eqref{eq4: appendix lemma: supportive lemma 2} and Eq.~\eqref{eq14: appendix lemma: supportive lemma 2} from Lemma~\ref{appendix lemma: supportive lemma 2}.
\end{proof}

\begin{lemma}\label{appendix lemma: supportive lemma 8}
   Under Assumption~\ref{assumption 2} and \ref{assumption 3}, we have
    $$\left|\E\left[\Gamma\left(\tilde{\pmb{x}}^{(i,j)},\pmb\theta_i, \pmb\theta_{i-t}\right)-\Gamma\left(\check{\pmb{x}}^{(i,j)},\pmb\theta_i,\pmb\theta_{i-t}\right)\right]\right|\leq 2M^2 \varphi(nt).$$
\end{lemma}
\begin{proof}
Let $\Delta=\Gamma\left(\tilde{\pmb{x}}^{(i,j)},\pmb\theta_i, \pmb\theta_{i-t}\right)-\Gamma\left(\check{\pmb{x}}^{(i,j)},\pmb\theta_i,\pmb\theta_{i-t}\right)$. We have
\begin{align}
    |\E[\Delta] |&= \left|\E\left[\E\left[\left\langle \nabla J(\pmb\theta_{i-t}),\widehat{\nabla} J(\tilde{\pmb{x}}^{(i,j)},\pmb\theta_i,\pmb\theta_{i-t})- \nabla J(\pmb\theta_{i-t})\right\rangle\right.\right.\right.\nonumber\\
    &\qquad -\left. \left.\left.\left.\left\langle \nabla J(\pmb\theta_{i-t}),\widehat{\nabla} J(\check{\pmb{x}}^{(i,j)},\pmb\theta_i,\pmb\theta_{i-t})-\nabla J(\pmb\theta_{i-t})\right\rangle\right| \pmb{s}^{(i-t,j)},\pmb\theta_{i-t}\right]\right]\right| \nonumber\\
&\leq\E\left[\left|\left\langle \nabla J(\pmb\theta_{i-t}),\E\left[\left.\widehat{\nabla} J\left(\tilde{\pmb{x}}^{(i,j)},\pmb\theta_i,\pmb\theta_{i-t}\right)- \widehat{\nabla} J\left(\check{\pmb{x}}^{(i,j)},\pmb\theta_i,\pmb\theta_{i-t}\right)\right| \pmb{s}^{(i-t,j)},\pmb\theta_{i-t}\right]\right\rangle\right|\right] \nonumber\\
&\leq M\E\left[\left\Vert \E\left[\left.\widehat{\nabla} J\left(\tilde{\pmb{x}}^{(i,j)},\pmb\theta_i,\pmb\theta_{i-t}\right)- \widehat{\nabla} J\left(\check{\pmb{x}}^{(i,j)},\pmb\theta_i,\pmb\theta_{i-t}\right) \right| \pmb{s}^{(i-t,j)},\pmb\theta_{i-t}\right]\right\Vert \right]\label{eq1: appendix lemma: supportive lemma 8}
\end{align}
where Step~\eqref{eq1: appendix lemma: supportive lemma 8} holds by applying Cauchy–Schwarz inequality and Lemma~\ref{lemma: bounded of policy gradient}.
Then the conclusion immediately follows by applying Lemma~\ref{appendix lemma: supportive lemma 3}.
\end{proof}
\begin{lemma}\label{appendix lemma: supportive lemma 9}
   Under Assumption~\ref{assumption 2} and \ref{assumption 3}, we have
    $$\E\left[\Gamma\left(\check{\pmb{x}}^{(i,j)},\pmb\theta_i, \pmb\theta_k\right)\right]=0.$$
\end{lemma}
\begin{proof}
Since the sample $\check{\pmb{x}}^{(i,j)}$ is drawn from the stationary distribution $d^{\pi_{\pmb\theta_{i-t}}}(\pmb{x})$, the gradient estimate $\widehat{\nabla} J\left(\check{\pmb{x}}^{(i,j)},\pmb\theta_i,\pmb\theta_{i-t}\right)$ is unbiased, that is,
\begin{align*}
&\E\left[\Gamma\left(\check{\pmb{x}}^{(i,j)},\pmb\theta_i, \pmb\theta_k\right)|\pmb{s}^{i-t,j},\pmb\theta_{i-t}\right] \\
&=\left\langle \nabla J(\pmb\theta_{i-t}),\E\left[\widehat{\nabla} J\left(\check{\pmb{x}}^{(i,j)},\pmb\theta_i,\pmb\theta_{i-t}\right)-\nabla J(\pmb\theta_{i-t})|\pmb{s}^{i-t,j},\pmb\theta_{i-t}\right] \right\rangle=0.
\end{align*}
\end{proof}

\section{Proof of Auxillary Lemmas}\label{appendix sec: proof of auxillary lemmas}

\begin{lemma}\label{auxillary lemma: parametric distance}
     Under Assumption~\ref{assumption 2} and \ref{assumption 3}, for any $k\geq i \geq t$ we have
     $$\E[\Vert \pmb\theta_{k}-\pmb\theta_{i-t}\Vert]\leq M \sum^{k}_{\ell=i-t}{\eta_{\ell}}.$$
\end{lemma}
\begin{proof} By applying Minkowski inequality, it holds
\begin{align}
    \E[\Vert \pmb\theta_{k}-\pmb\theta_{i-t}\Vert]&\leq \sum^{k}_{\ell=i-t}\eta_{\ell} \E\left[\frac{1}{|\mathcal{U}_\ell|n}
 \sum_{{\pmb\theta_h}\in \mathcal{U}_\ell}
 \sum^{n}_{j=1}
 \frac{\pi_{\pmb\theta_\ell}\left(\pmb{a}^{(h,j)}|\pmb{s}^{(h,j)}\right)}
 {\pi_{\pmb\theta_h}\left(\pmb{a}^{(h,j)}|\pmb{s}^{(h,j)}\right)}
\left\Vert g\left(\pmb{x}^{(h,j)}|\pmb\theta_{\ell}\right)\right\Vert \right] \nonumber\\
&\leq M\sum^{k}_{\ell=i-t}\eta_{\ell} \frac{1}{|\mathcal{U}_\ell|n}
 \sum_{{\pmb\theta_h}\in \mathcal{U}_\ell}
 \sum^{n}_{j=1}\E\left[
 \frac{\pi_{\pmb\theta_\ell}\left(\pmb{a}^{(h,j)}|\pmb{s}^{(h,j)}\right)}
 {\pi_{\pmb\theta_h}\left(\pmb{a}^{(h,j)}|\pmb{s}^{(h,j)}\right)}
 \right] \label{eq3: appendix lemma: supportive lemma 2}\\
& \leq M\sum^{k}_{\ell=i-t}{\eta_{\ell}} \nonumber
\end{align}
where Step~\eqref{eq3: appendix lemma: supportive lemma 2} holds by Lemma~\ref{lemma: bounded of policy gradient}.
\end{proof}

\begin{lemma}\label{auxillary lemma: L2 parametric distance}
     Under Assumption~\ref{assumption 2} and \ref{assumption 3}, for any $k\geq i \geq t$ we have
\begin{equation*}
   \E\left[\Vert \pmb\theta_{k}-\pmb\theta_{i-t}\Vert^2\right]^{1/2} \leq \sqrt{\bar{\rho}_k+1} M(k-i+t)^{1/2} \eta_{i-t}. 
\end{equation*}
\end{lemma}
\begin{proof}
Applying  Minkowski's inequality and Lemma~\ref{lemma: bounded of policy gradient} gives
\begin{equation}\label{eq-0: appendix lemma: supportive lemma 6}
\E\left[\Vert \pmb\theta_{k}-\pmb\theta_{i-t}\Vert^2\right]=\sum_{\ell=i-t+1}^k \E\left[\Vert \pmb\theta_\ell -\pmb\theta_{\ell-1}\Vert^2\right]=\sum_{\ell=i-t+1}^k \eta_\ell^2\E\left[\left\Vert\widehat{\nabla} J^{R}_\ell\right\Vert^2\right].
\end{equation}
Lemma~\ref{lemma: total variance of LR policy gradient} implies that
 \begin{align}
    \E\left[\left\Vert\widehat{\nabla} J^{R}_\ell\right\Vert^2\right]
    &=\Tr\left(\Var\left[\widehat{\nabla} J^{R}_\ell\right]\right)+ \left\Vert\E\left[\widehat{\nabla} J^{R}_\ell\right]\right\Vert^2 \nonumber\\
    &\leq \frac{M^2}{|\mathcal{U}_k|^2}\sum_{\pmb\theta_i\in\mathcal{U}_k}\sum_{\pmb\theta_{i^\prime}\in\mathcal{U}_k}\max_{\ell=1,2,\ldots,d}\left(\Corr^{(\ell)}_{i,i^\prime,k}\right)w_{i,k} w_{i^\prime,k}+ \left\Vert\E\left[\widehat{\nabla} J^{R}_\ell\right]\right\Vert^2 \nonumber\\
    & \leq {M^2} \bar{\rho}_k + \left\Vert\E\left[\widehat{\nabla} J^{R}_\ell\right]\right\Vert^2,\label{eq-1: appendix lemma: supportive lemma 6}
\end{align}
where $\bar{\rho}_k = \frac{1}{|\mathcal{U}_k|^2}\sum_{\pmb\theta_i\in\mathcal{U}_k}\sum_{\pmb\theta_{i^\prime}\in\mathcal{U}_k}\max_{\ell=1,2,\ldots,d}\left(\Corr^{(\ell)}_{i,i^\prime,k}\right)w_{i,k} w_{i^\prime,k}$.
Also by Minkowski's inequality, it holds for any $\ell$ that
\begin{align}
  \left\Vert\E\left[\widehat{\nabla} J^{R}_\ell\right]\right\Vert  &\leq\E\left[\frac{1}{|\mathcal{U}_\ell|n}
 \sum_{{\pmb\theta_h}\in \mathcal{U}_\ell}
 \sum^{n}_{j=1}
 \frac{\pi_{\pmb\theta_\ell}\left(\pmb{a}^{(h,j)}|\pmb{s}^{(h,j)}\right)}
 {\pi_{\pmb\theta_h}\left(\pmb{a}^{(h,j)}|\pmb{s}^{(h,j)}\right)}
\left\Vert g\left(\pmb{x}^{(h,j)}|\pmb\theta_{\ell}\right)\right\Vert \right] \nonumber\\
&\leq  \frac{M}{|\mathcal{U}_\ell|n}
 \sum_{{\pmb\theta_h}\in \mathcal{U}_\ell}
 \sum^{n}_{j=1}\E\left[
 \frac{\pi_{\pmb\theta_\ell}\left(\pmb{a}^{(h,j)}|\pmb{s}^{(h,j)}\right)}
 {\pi_{\pmb\theta_h}\left(\pmb{a}^{(h,j)}|\pmb{s}^{(h,j)}\right)}
 \right] \nonumber\\
& \leq M.\label{eq-2: appendix lemma: supportive lemma 6}
\end{align}
By collecting \eqref{eq-1: appendix lemma: supportive lemma 6} and \eqref{eq-2: appendix lemma: supportive lemma 6}, Eq.~\eqref{eq-0: appendix lemma: supportive lemma 6} becomes 
\begin{equation*}
    \E\left[\Vert \pmb\theta_{k}-\pmb\theta_{i-t}\Vert^2\right] \leq (\bar{\rho}_k+1)M^2 \sum_{\ell=i-t}^k \eta_\ell^2 \leq (\bar{\rho}_k+1) M^2 (k-i+t) \eta_{i-t}^2
\end{equation*}
where the last inequality holds because $\{\eta_k\}$ is a non-increasing sequence. 
Then we conclude
\begin{equation*}
   \E\left[\Vert \pmb\theta_{k}-\pmb\theta_{i-t}\Vert^2\right]^{1/2} \leq \sqrt{\bar{\rho}_k+1} M(k-i+t)^{1/2} \eta_{i-t}. 
\end{equation*}
\end{proof}

\begin{lemma}\label{auxillary lemma: sum of learning rates}
Suppose $\eta_k=\eta_1 k^{-r}$ is a non-negative sequence for some positive constants $\eta_1$ and $r$. Then for any large enough $k$ and $\tau <k$, we have:
\begin{equation*}
\sum^{k-t}_{i=1}\eta_i \leq \frac{\eta_1}{1- r}(k-t)^{1-r}.
\end{equation*}
\end{lemma}
\begin{proof} 
The conclusion is obtained by observing that
\begin{equation*}
\sum^{k-t}_{i=1}\eta_i= \sum^{k-t}_{i=1} \frac{\eta_1}{i^r}\leq \eta_1\int^{k-t}_{1} x^{-r} \dd x = \eta_1\frac{(k-t)^{1-r}}{1- r} - \frac{\eta_1}{1-r}\leq \eta_1\frac{(k-t)^{1-r}}{1- r} 
\end{equation*}
which completes the proof.
\end{proof}

\begin{lemma}\label{supportive lemma: mean sequence Big O}
Suppose $\{a_k\}$ is a non-negative, bounded sequence. 
Let $U_a$ denote the bound of the sequence. Then
for any large enough $K$, let $\tau= f(K) <K$ (i.e. $f(\cdot)$ is a positive increasing function) and we have:
\begin{equation*}
    \frac{1}{K}\sum^K_{k=1}a_k\leq\frac{\tau-1}{K}U_a+\frac{1}{1+K-\tau}\sum^K_{k=\tau}a_k.
\end{equation*}
\end{lemma}
\begin{proof}
The conclusion is obtained by observing that
\begin{equation*}
    \frac{1}{K}\sum^T_{k=1}a_k\leq \frac{1}{K}\left((\tau-1)U_a+\sum^K_{k=\tau}a_k\right)=\frac{\tau-1}{K}U_a+\frac{1}{K}\sum^K_{k=\tau}a_k \leq \frac{\tau-1}{K}U_a+\frac{1}{1+K-\tau}\sum^K_{k=\tau}a_k.
\end{equation*}
\end{proof}

\end{document}